\documentclass[10pt,journal,compsoc]{IEEEtran}

\usepackage{graphicx}	
\usepackage{paralist}
\usepackage{amsmath}
\usepackage{mathtools}
\usepackage{epstopdf}
\usepackage{stfloats}
\usepackage{tabularx,algorithm}	
\usepackage{makecell}
\usepackage{hyperref}
\usepackage{color}
\usepackage{fancyhdr}
\usepackage{algpseudocode}
\usepackage{multirow}
\usepackage{bm,bbm}
\usepackage{balance}
\usepackage{tikz}
\usepackage[labelformat=simple]{subcaption}

\usepackage{amsthm}

\usepackage[noadjust]{cite}
\usepackage{amsfonts}
\usepackage{dsfont}
\usepackage{mathrsfs}
\usepackage{amssymb}
\usepackage{enumitem}
\usepackage{algorithm}
\usepackage{algcompatible}
\usepackage{epsfig}
\usepackage[english]{babel}
\usepackage[all]{xy}
\usepackage[latin1]{inputenc}
\usepackage{tikz}
\usetikzlibrary{shapes,arrows}
\graphicspath{{figure/}}

\def\BibTeX{{\rm B\kern-.05em{\sc i\kern-.025em b}\kern-.08em
    T\kern-.1667em\lower.7ex\hbox{E}\kern-.125emX}}

\newtheorem{theorem}{Theorem}

\newtheorem{lemma}{Lemma}
\newtheorem{remark}{Remark}

\newtheorem{assumption}{Assumption}

\begin{document}
\title{Consistent and Optimal Solution to 
	Camera Motion Estimation}
\author{Guangyang Zeng*, Qingcheng Zeng*, Xinghan Li, Biqiang Mu, Jiming Chen, Ling Shi, and Junfeng Wu
\thanks{ The work was supported in part by NSFC under Grants 62336005 and 62088101; in part by the Shenzhen Science and Technology Program under Grants JCYJ20240813113609013 and JCYJ20241202124010014; in part by Guangdong Basic and Applied Basic Research Foundation 2024A151524009; in part by Key Research and Development Program of Zhejiang Province under Grant 2025C01061.}
\thanks{G. Zeng, Q. Zeng, and J. Wu are with the School of Data Science, Chinese University of Hong Kong, Shenzhen, Shenzhen 518172, P. R. China.
		{\tt\small \{zengguangyang, zengqingcheng, junfengwu\}@cuhk.edu.cn}.}
	\thanks{X. Li and J. Chen are with the College of Control Science and Engineering and the State Key Laboratory of Industrial Control Technology, Zhejiang University, Hangzhou 310027, P. R. China.
		{\tt\small xinghanli0207@gmail.com, cjm@zju.edu.cn}.}
  \thanks{B. Mu is with State Key Laboratory of Mathematical Sciences, Academy of Mathematics and Systems Science, Chinese Academy of Sciences, Beijing 100190, China.
		{\tt\small bqmu@amss.ac.cn}.}
 \thanks{L. Shi is with the Department of Electronic and Computer Engineering and with the Department of Chemical and Biological Engineering, Hong Kong University of Science and
Technology, Hong Kong.
		{\tt\small eesling@ust.hk}.}
        \thanks{*Equally contributed.}
  }


\maketitle

\begin{abstract}
Given 2D point correspondences between an image pair, inferring the camera motion is a fundamental issue in the computer vision community. The existing works generally set out from the epipolar constraint and estimate the essential matrix, which is not optimal in the maximum likelihood (ML) sense. In this paper, we dive into the original measurement model with respect to the rotation matrix and normalized translation vector and formulate the ML problem. We then propose an optimal two-step algorithm to solve it: In the first step, we estimate the variance of measurement noises and devise a consistent estimator based on bias elimination; In the second step, we execute a one-step Gauss-Newton iteration on manifold to refine the consistent estimator. We prove that the proposed estimator achieves the same asymptotic statistical properties as the ML estimator: The first is consistency, i.e., the estimator converges to the ground truth as the point number increases; The second is asymptotic efficiency, i.e., the mean squared error of the estimator converges to the theoretical lower bound --- Cramer-Rao bound. In addition, we show that our algorithm has linear time complexity. These appealing characteristics endow our estimator with a great advantage in the case of dense point correspondences. Experiments on both synthetic data and real images demonstrate that when the point number reaches the order of hundreds, our estimator outperforms the state-of-the-art ones in terms of estimation accuracy and CPU time. 
\end{abstract}


\section{Introduction} \label{section_introduction}
\IEEEPARstart{C}{amera} motion estimation (CME) involves estimating the relative camera pose from two images. It serves as a building block in many visual odometry, structure-from-motion (SfM), and simultaneous localization and mapping (SLAM) systems~\cite{hartley2003multiple,engel2017direct,sarlin2023pixel,zhan2020visual}. 
A CME pipeline generally includes two modules~\cite{hartley1997defense,nister2004efficient,li20134,chatterjee2017robust,zou2012coslam,jau2020deep}: The front-end extracts feature points from input images and then conducts feature matching to produce 2D point correspondences; The back-end recovers the relative pose based on these point correspondences. In this paper, we focus on the back-end algorithm by assuming the point correspondences are given.  


The existing literature generally estimates the essential matrix $\bf E$ first, see~\eqref{E_definition}, based on which the relative pose (including rotation matrix $\bf R$ and translation $\bf t$) is recovered. Due to the scale ambiguity caused by the unknown depths of measured points, the rotation matrix and the direction of translation can be recovered from the essential matrix, while the translation distance cannot be identified~\cite{hartley2003multiple}. To further estimate the translation distance, a priori information about the 3D points should be provided, e.g., via a calibration board or constructed 3D structures.

In an ideal noise-free case, given the normalized image coordinates of the $i$-th correspondence on the two images, say ${\bf y}_i^{h}$ and ${\bf z}_i^{h}$, the epipolar constraint yields the basic equation for the essential matrix: 
\begin{equation} \label{essential_matrix_eqn}
	{\bf z}_i^{h \top} {\bf E} {\bf y}_i^{h}=0.
\end{equation} 
Most literature estimates the essential matrix by minimizing the algebraic error originating from~\eqref{essential_matrix_eqn}~\cite{zhao2020efficient,chesi2008camera,helmke2007essential,briales2018certifiably,ding2021globally}. There are also some works minimizing geometric errors alternatively, such as the projection error~\cite{hartley2003multiple,hartley2009global,jiang2013global}. 
No matter what formulation is adopted, the resulting optimization problem is nonconvex since the set of essential matrices is a nonconvex manifold. Some works devised Gauss-Newton (GN) iterations on manifold to guarantee that the estimate at each iteration is an essential matrix~\cite{ma2001optimization,helmke2007essential,tron2017space}. These local search methods require a good initial value, otherwise, they will converge to local minima.  
A more prevalent idea is conducting relaxation, e.g., semidefinite relaxation (SDR)~\cite{zhao2020efficient,briales2018certifiably,garcia2022tighter} and direct linear transformation (DLT)~\cite{hartley1997defense,li2008linear} to obtain an optimization problem whose global minimizer can be obtained. Nevertheless, the global minimizer of the modified problem is not necessarily that of the original problem. In short, the global solution to the optimization problem over the essential matrix manifold is still an open problem. 

\begin{figure*}[!t]
	\centering
	\begin{subfigure}[b]{0.48\textwidth}
		\centering
		\includegraphics[width=\textwidth]{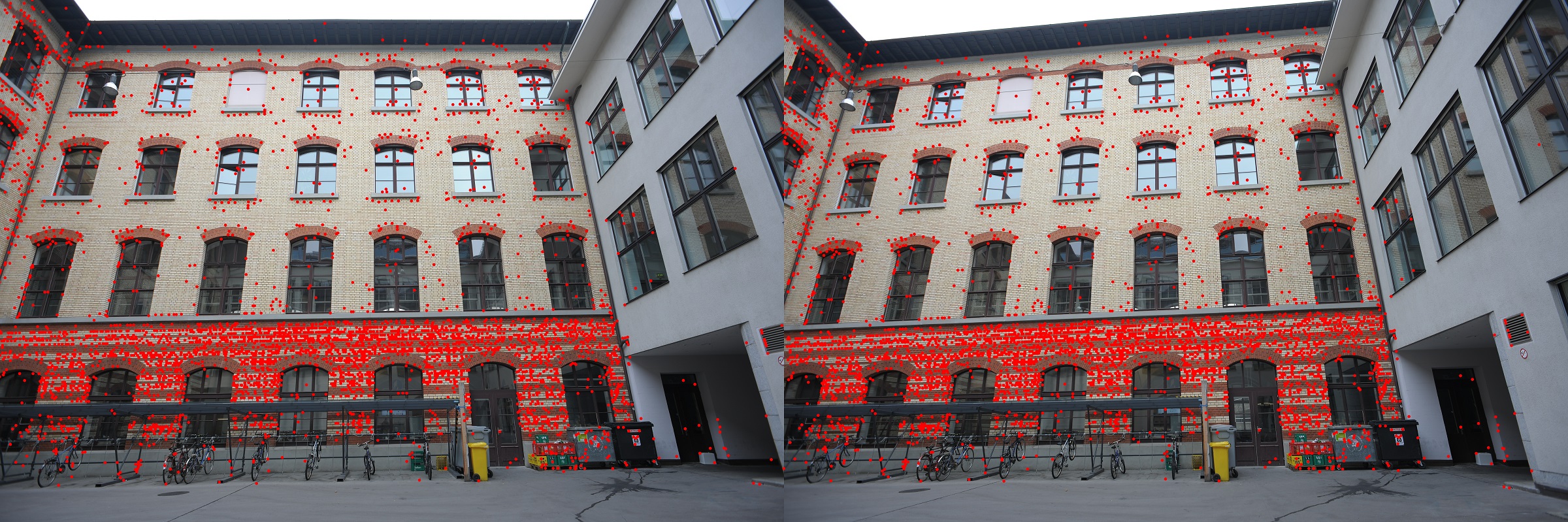}
		\caption{courtyard 36-37}
		\label{courtyard_36_37}
	\end{subfigure}
	\begin{subfigure}[b]{0.48\textwidth}
		\centering
		\includegraphics[width=\textwidth]{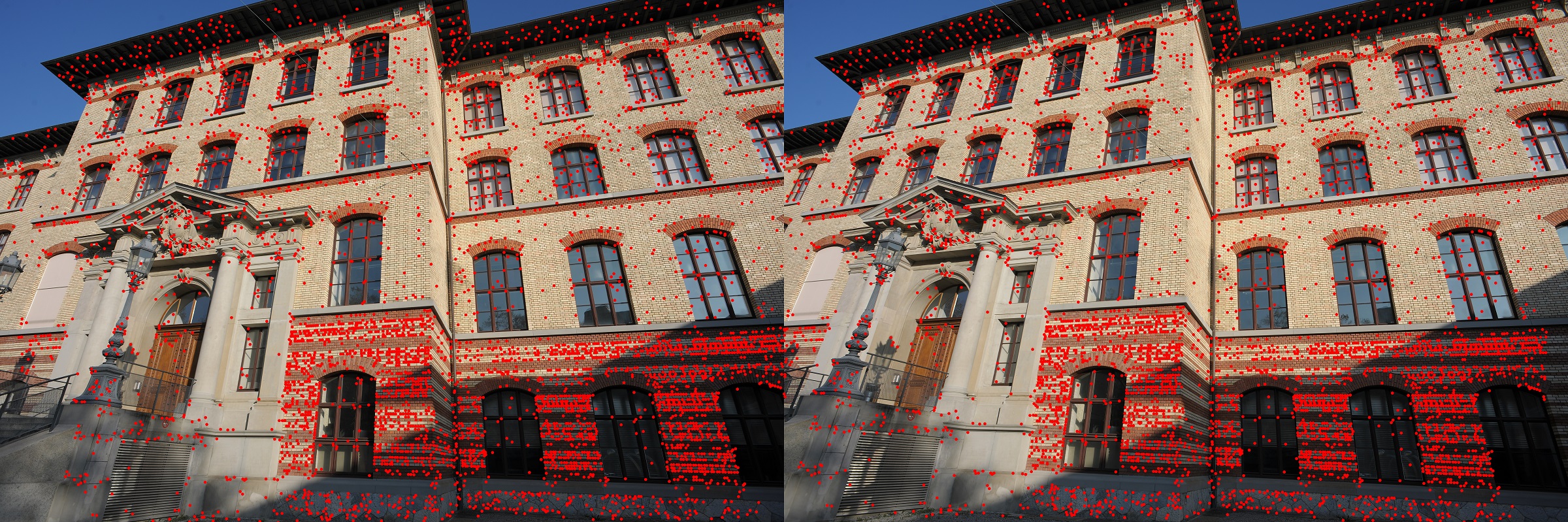}
		\caption{facade 32-33}
		\label{facade_32_33}
	\end{subfigure}
	\begin{subfigure}[b]{0.48\textwidth}
		\centering
		\includegraphics[width=\textwidth]{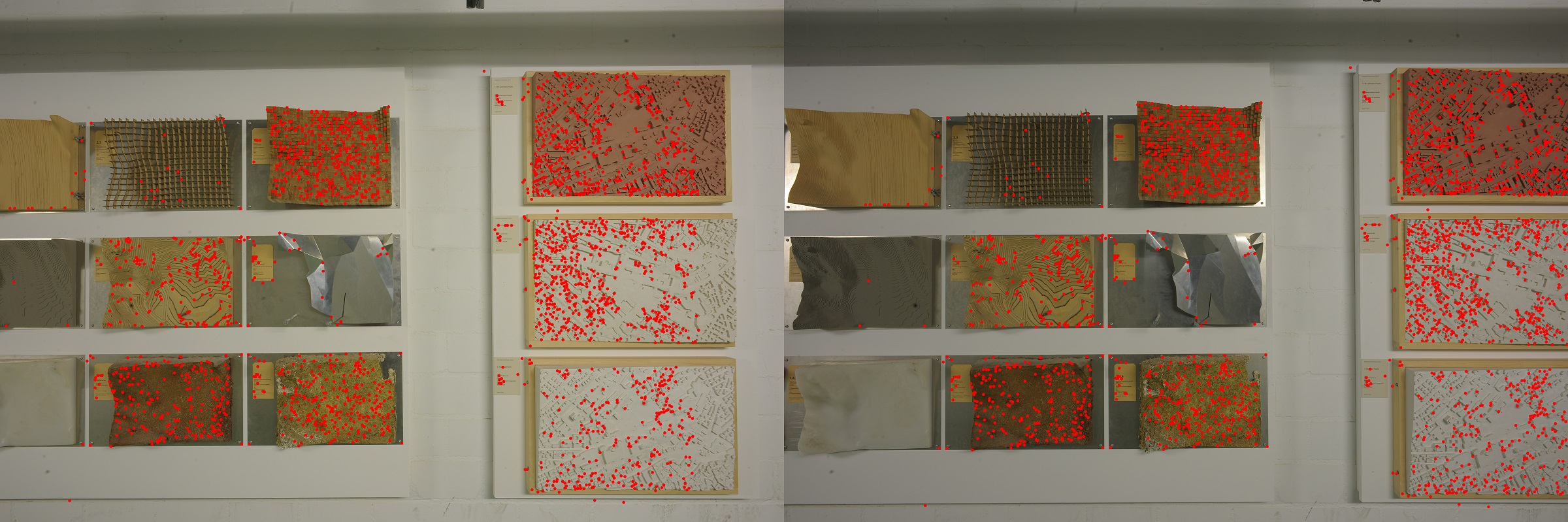}
		\caption{terrains 8-9}
		\label{terrains_8_9}
	\end{subfigure}
	\begin{subfigure}[b]{0.48\textwidth}
		\centering
		\includegraphics[width=\textwidth]{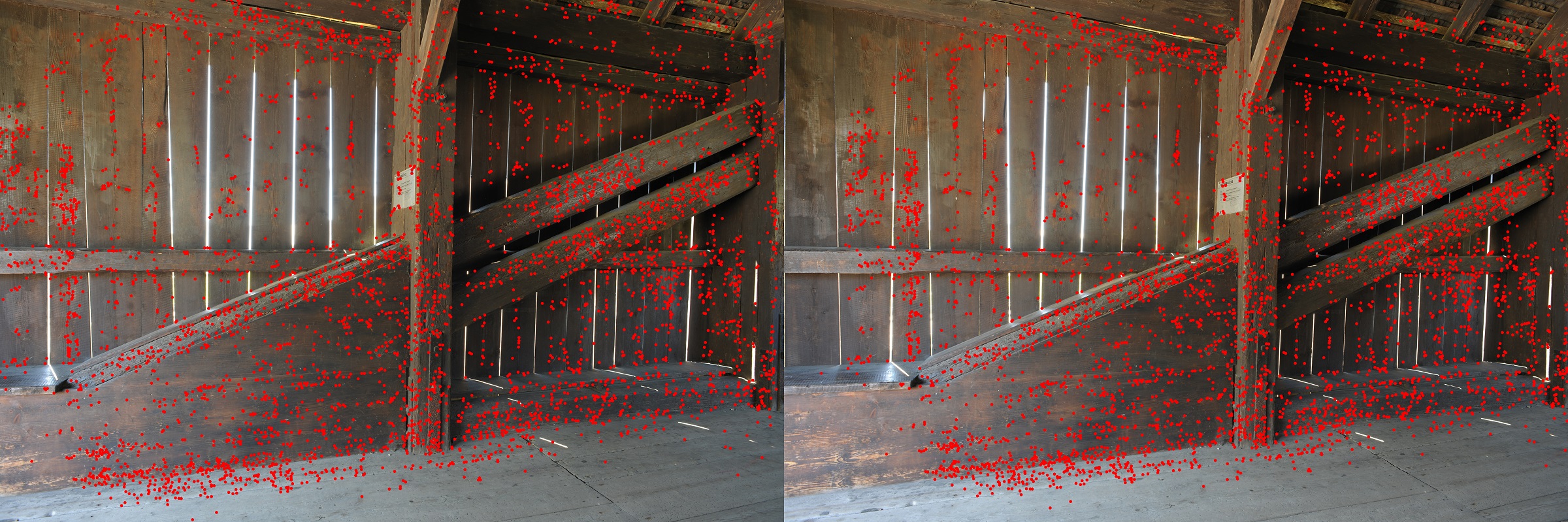}
		\caption{bridge 36-37}
		\label{bridge_36_37}
	\end{subfigure}
	\caption{Four image pairs from the ETH3D dataset~\cite{schops2017multi}. The red points are matched feature points. Each image pair has thousands of point correspondences.}
	\label{ETH3D_image_pair}
\end{figure*}

It should be noted that most of the literature sets out from the epipolar constraint~\eqref{essential_matrix_eqn}, which is not the original measurement model. As a result, the formulated optimization problem is generally not optimal in the maximum likelihood (ML) sense. This will be explicitly discussed in Section~\ref{section_problem_formulation}. From the estimation theory, we know that the ML estimator is statistically optimal in the sense that under some regularity conditions, it is \emph{consistent} and \emph{asymptotically statistically-efficient}~\cite{schervish2012theory}. Here, \emph{consistent} means that with the increase of point number, the estimator converges to the true value, and \emph{asymptotically statistically-efficient} denotes that as the point number becomes sufficiently large, the mean squared error (MSE) of the estimator reaches the theoretical lower bound --- Cramer-Rao bound (CRB)\footnote{We use statistically-efficient to distinguish from computationally-efficient which describes an algorithm that has low time complexity.}. 
Thus, a more appropriate method should take the original measurement model and measurement noises into account and formulate an ML problem. Similar to the existing formulations, the ML problem is nonconvex, and finding its global solution is a challenging task. 

In this paper, we utilize the rotation matrix and normalized translation vector, instead of the compound essential matrix, to obtain the original measurement model --- a noise-contaminated function between matched 2D points, based on which the ML problem is formulated. For the nonconvex ML problem, we propose a two-step algorithm that can optimally solve it in the asymptotic case where the number of point correspondences is large. In other words, the obtained estimator is consistent and asymptotically statistically-efficient. 
It is noteworthy to see that in some texture-rich scenarios, one can obtain a large number of feature correspondences. For instance, as shown in Figure~\ref{ETH3D_image_pair}, in the public dataset of ETH3D~\cite{schops2017multi}, there may exist thousands of point correspondences in an image pair. With abundant points, our algorithm achieves higher estimation accuracy compared with state-of-the-art (SOTA) algorithms. Actually, according to the results of simulations and real image tests, our algorithm begins to show its advantage when the point number reaches the order of hundreds.   
Besides estimation accuracy, time complexity is another important metric to appraise an algorithm, especially in the case of a large number of feature points. Since a closed-form solution is available in the first step, and only a one-step GN iteration is executed in the second step, our algorithm is \emph{computationally efficient} --- it has linear time complexity overall and has the capacity of real-time implementation even when the point number reaches the order of thousands.
To summarize, the main contributions of this paper are listed as follows:
\begin{enumerate}
	\item [(i).] We derive the analytic expression between 2D point correspondences, see~\eqref{noisy_measurement_model}. The expression is the original measurement model, which is a function of the rotation matrix and normalized translation, instead of the essential matrix. Based on the measurement model, the statistically optimal ML problem is formulated.  
	\item [(ii).] In the first step of our algorithm, we propose a novel consistent estimator of the measurement noise variance. It only involves calculating the maximum eigenvalue of a $9 \times 9$ matrix. Based on the noise variance estimate, we perform bias elimination and eigenvalue decomposition, obtaining consistent estimators of the rotation matrix and normalized translation. 
	\item [(iii).] In the second step of our algorithm, we take the consistent estimates as the initial value and devise the GN iterations on ${\rm SO}(3)$ (for the rotation matrix) and 2-sphere (for the normalized translation). Moreover, we prove that only a one-step of GN iteration is sufficient to achieve the same asymptotic property as the ML estimator, i.e., the MSE asymptotically reaches the CRB. 
	\item [(iv).] We conduct extensive experiments with both synthetic data and real images. The results show that our proposed algorithm outperforms SOTA ones in terms of estimation accuracy and CPU time when the point number reaches the order of hundreds. The open source code is available at \url{https://github.com/LIAS-CUHKSZ/epipolar_eval}.
\end{enumerate}

The remainder of the paper is organized as follows. 
In Section~\ref{section_related_work}, we review the related work on camera motion estimation. 
In Section~\ref{section_preliminaries}, we introduce some notations and necessary preliminaries. In Section~\ref{section_problem_formulation}, we derive the original measurement model and formulate the ML problem. 
In Section~\ref{section_consistent_estimator}, we estimate the variance of measurement noises and propose a consistent estimator.
In Section~\ref{section_GN}, we refine the consistent estimator via GN iterations on ${\rm SO}(3)$ and the 2-sphere. 
Experiment results are presented in Section~\ref{section_experiment}, followed by conclusions in Section~\ref{section_conclusion}.

\section{Related Work} \label{section_related_work}
In CME, the epipolar geometry constraint yields a linear equation of the essential matrix, as shown in~\eqref{essential_matrix_eqn}. Given an essential matrix, the rotation and normalized translation can be recovered via singular value decomposition~\cite{hartley2003multiple}. Thus, most of the literature estimates the essential matrix in CME. The rotation matrix and translation vector have three degrees of freedom each. The essential matrix loses one degree of freedom because of scale ambiguity. Hence, it has in total five degrees of freedom, and at least five points are needed to estimate it~\cite{nister2004efficient,kukelova2007two,kukelova2008polynomial,stewenius2006recent,kneip2012finding}, which is called the minimal case. It is noteworthy to see that the normalized image coordinates are used in estimating the essential matrix. If the pixel coordinates are utilized, one can adopt eight-point solvers~\cite{hartley1997defense,longuet1981computer} to estimate the fundamental matrix, which along with the intrinsic matrix can further recover the essential matrix. Barath~\cite{barath2018five} proposed a five-point fundamental matrix estimation for uncalibrated cameras by first estimating a homography from three correspondences. The above algorithms fall into the scope of fixed-point solvers. They are sensitive to measurement noises and usually need to be embedded in a RANSAC framework to enhance robustness~\cite{zhao2020efficient,barath2021graph}. 

More literature studies arbitrary-point solvers since they can make full use of the whole measurements. Note that the set of essential matrices is nonconvex. Hence, the optimization problems over this set are nonconvex, and how to find a global solution is still an open problem. Some works devised local iterations on the manifold of essential matrices to seek a nearby stationary solution~\cite{ma2001optimization,helmke2007essential,subbarao2008robust,tron2017space}. In these works, the essential matrix manifold is characterized by different formulations, leading to distinct performances and convergence rates. Ma \emph{et al.}~\cite{ma2001optimization} took the eight-point estimate as the initial value and proposed a Riemannian-Newton algorithm to solve the structure-from-motion problem. Helmke \emph{et al.}~\cite{helmke2007essential} improved the convergence property and reduced the computational cost by proposing Gauss-Newton-type algorithms. Tron and Daniilidis~\cite{tron2017space} characterized the space of essential matrices as a quotient manifold that takes the symmetric role played by the two views and the geometric peculiarities of
the epipolar constraint into account. We remark that the local iterative methods are sensitive to initial values. Without a good initial value, they can only converge to local minima. 

Instead of local searching, many works focus on globally optimal solvers. Branch and bound (BnB) methods, which explore the whole optimization space, were utilized in~\cite{hartley2009global,kneip2013direct}. They achieve global optima but are computationally inefficient --- exponential time in the worst case. Problem relaxation is a widely adopted idea for devising global solvers. The polynomial optimization problems established over the essential matrix manifold can be reformulated as QCQP problems, which can be further relaxed into semidefinite programming (SDP) problems via Shor's relaxation~\cite{boyd2004convex}. Karimian and Tron~\cite{karimian2023essential} used the quintessential matrix structure to derive a succession of SDRs that require fewer parameters than the existing nonminimal solvers. 
Although the SDP problems can be globally solved in polynomial time with off-the-shelf tools, their global minima generally do not coincide with that of the original problem. Therefore, recently, some works characterized the theoretical properties of the proposed SDP solvers. In~\cite{briales2018certifiably,garcia2021certifiable}, certifiable relative pose solvers were proposed whose optimality w.r.t. the original problem can be certified a posteriori. Zhao~\cite{zhao2020efficient} also presented a certifiable SDP-based solver. In addition, the tightness of the SDR relaxation was proved when the noise intensity is small. 

Apart from classical geometry-based methods, there are also some learning-based approaches. Ranftl and Koltun~\cite{ranftl2018deep} cast the estimation of fundamental matrices as a series of
weighted homogeneous least-squares problems, where robust weights are estimated using deep networks.
Moran \emph{et al.}~\cite{moran2024consensus} proposed a noise-aware Deep Sets framework to estimate relative camera poses, focusing on a simpler architectural design.

To summarize, most of the literature estimates the essential matrix that encodes rotation and translation information instead of directly optimizing the rotation matrix and normalized translation vector. In addition, due to the nonconvexity of the essential matrix manifold, the global solution to the resulting problem is still an open problem. In the rest of this paper, we set out from the original measurement model and formulate an optimization problem in the ML sense, which optimizes directly over the rotation matrix and normalized translation vector. Moreover, we propose an asymptotically optimal two-step algorithm that can optimally solve the formulated nonconvex problem when the point number is large. 

\section{Preliminaries} \label{section_preliminaries}
To facilitate the readability of the subsequent technical part, in this section, we present some notations and preliminaries in probability and statistics, rigid transformation, and the essential matrix. 

\subsection{Notations} \label{notations}
We use bold lowercase letters to denote vectors, e.g., ${\bf x}$, $\bf y$, $\bf z$, and bold uppercase letters for matrices, e.g., ${\bf X}$, $\bf Y$, $\bf Z$. The identity matrix of size $n$ is represented as ${\bf I}_n$. We denote the all-zeros vector of size $n$ as ${\bf 0}_n$. The operation $\otimes$ denotes the Kronecker product. For a vector $\bf x$, $\|{\bf x}\|$ denotes its $\ell_2$-norm. Given a matrix ${\bf X}$, $\|{\bf X}\|_{\rm F}$ deotes its Frobenius norm, ${\rm tr}({\bf X})$ represents its trace, and ${\rm vec}({\bf X})$ yields a vector by concatenating the columns of ${\bf X}$. If ${\bf X}$ has real eigenvalues, then $\lambda_{\rm min}({\bf X})$ and $\lambda_{\rm max}({\bf X})$ denote the minimum and maximum eigenvalues of ${\bf X}$, respectively. For a quantity ${\bf x}$ corrupted by noise, we use ${\bf x}^o$ to denote its noise-free counterpart.

\subsection{Preliminaries in probability and statistics} \label{preliminary_statistic}
\textbf{Convergence in probability.} We use $\hat {\bm \theta}_m \ \xrightarrow{p} \ {\bm \theta}^o$ to denote that $\hat {\bm \theta}_m$ converges to ${\bm \theta}^o$ in probability, i.e., for any $\varepsilon>0$, 
\begin{equation*}
    \lim_{m \rightarrow \infty} P(\|\hat {\bm \theta}_m-{\bm \theta}^o\| \geq \varepsilon)=0.
\end{equation*}
In addition, the notation $\Delta \hat {\bm \theta}_m=o_p(a_m)$ means that the sequence $\Delta \hat {\bm \theta}_m/a_m$ converges to $0$ in probability. 

\textbf{Stochastic boundedness.} The notation $\Delta \hat {\bm \theta}_m=O_p(a_m)$ means that the sequence $\Delta \hat {\bm \theta}_m/a_m$ is stochastically bounded. That is, for any $\varepsilon >0$, there exists a finite $M$ and a finite $N$ such that for any $m>M$,
\begin{equation*}
    P(\|\Delta \hat {\bm \theta}_m/a_m\|>N )<\varepsilon.
\end{equation*}

\textbf{$\sqrt{m}$-consistent estimator.} If $\hat {\bm \theta}_m$ is a $\sqrt{m}$-consistent estimator of ${\bm \theta}^o$, then
\begin{equation*}
    \hat {\bm \theta}_m-{\bm \theta}^o=O_p(1/\sqrt{m}).
\end{equation*}
This notion includes two implications: The estimator $\hat {\bm \theta}_m$ is consistent --- it converges to ${\bm \theta}^o$ in probability; The convergence rate is $1/\sqrt{m}$. 

\textbf{(Asymptotically) unbiased estimator.} The bias of an estimator $ \hat {\bm \theta}_m$ is equal to its expectation minus the true value, i.e., ${\rm Bias}(\hat {\bm \theta}_m)=\mathbb E[\hat {\bm \theta}_m]-{\bm \theta}^o$. If ${\rm Bias}(\hat {\bm \theta}_m)=0$, we call $\hat {\bm \theta}_m$ an unbiased estimator of ${\bm \theta}^o$. In particular, if $\lim_{m \rightarrow \infty}{\rm Bias}(\hat {\bm \theta}_m)=0$, we call $\hat {\bm \theta}_m$ an asymptotically unbiased estimator. It is noteworthy to see that an asymptotically unbiased estimator $\hat {\bm \theta}_m$ may not necessarily be unbiased when $m$ is finite. 

\textbf{(Asymptotically) efficient estimator.} An unbiased estimator $\hat {\bm \theta}_m$ is said to be efficient if the trace of its covariance is equal to the theoretical lower bound --- CRB, i.e., ${\rm tr}({\rm cov}(\hat {\bm \theta}_m))={\rm CRB}$. In particular, $\hat {\bm \theta}_m$ is called asymptotically efficient if it is asymptotically unbiased, and $\lim_{m \rightarrow \infty} {\rm tr}({\rm cov}(\hat {\bm \theta}_m))={\rm CRB}$.

\subsection{Rigid transformation and the essential matrix} \label{rigid_transformation}
The (proper) rigid transformations, or said relative poses, include rotations and translations. The rotation can be characterized by a rotation matrix $\bf R$. Specifically, in the 3D Euclidean space, rotation matrices belong to the special orthogonal group
\begin{equation*}
{\rm SO}(3)=\{{\bf R}\in \mathbb R^{3 \times 3} \mid {\bf R}^\top {\bf R}={\bf I}_3, {\rm det}({\bf R})=1\}.
\end{equation*}
The translation is depicted by a vector ${\bf t} \in \mathbb R^3$. 
Suppose the relative pose of the second frame w.r.t. the first frame is $({\bf R},{\bf t})$, and the coordinates of a 3D point in the second frame is ${\bf x}$. Then the coordinates of the point in the first frame is ${\bf R}{\bf x}+{\bf t}$. 

Given a vector ${\bf t}=[t_1~t_2~t_3]^\top \in \mathbb R^3$, the ``hat'' function ${\bf t}^\wedge$ generates the following skew-symmetric matrix 
$${\bf t}^\wedge=\begin{bmatrix}
	0 & -t_3 & t_2 \\
	t_3 & 0 & -t_1 \\
	-t_2 & t_1 & 0 
\end{bmatrix}.
$$
In epipolar geometry, the essential matrix is given as 
\begin{equation} \label{E_definition}
    {\bf E}={\bf t}^\wedge {\bf R}.
\end{equation}
In the noise-free case, the epipolar constraint is depicted by equation~\eqref{essential_matrix_eqn} in Section~\ref{section_introduction}.

As mentioned previously, the distance of the translation $\bf t$ cannot be identified given an image pair, hence we are interested in the normalized translation $\bar {\bf t}$, which belongs to the $2$-sphere
\begin{equation*}
    S^2=\{\bar {\bf t} \in \mathbb R^3 \mid \|\bar {\bf t}\|=1\}.
\end{equation*}
This derives the set of normalized essential matrices
\begin{equation*}
    \mathcal M_E=\{{\bf E} \mid {\bf E}=\bar {\bf t}^\wedge {\bf R}, \exists~ {\bf R} \in {\rm SO}(3), \bar {\bf t} \in S^2\}.
\end{equation*}

\section{ML Problem Formulation from Original Measurement Model} \label{section_problem_formulation}
We consider the pinhole camera model and assume the intrinsic matrices of cameras are known. The two-view geometry is shown in Figure~\ref{two_view_geometry}. We use ${\bf x}_i=[x_{i1}~x_{i2}~x_{i3}]^\top \in \mathbb R^3,i \in \{1,\ldots,m\}$ to denote the coordinates of the $i$-th 3D point in the world frame. Its 2D projections on the image planes are ${\bf y}_i =[y_{i1}~y_{i2}]^\top \in \mathbb R^2$ and ${\bf z}_i=[z_{i1}~z_{i2}]^\top \in \mathbb R^2$. Given the intrinsic matrices of cameras, we can use normalized image coordinates to represent points in the image, i.e., the focal length is set to be $1$. The homogeneous normalized image coordinates of ${\bf y}_i$ and ${\bf z}_i$ are denoted as ${\bf y}_i^h=[{\bf y}_i^\top~1]^\top$ and ${\bf z}_i^h=[{\bf z}_i^\top~1]^\top$, respectively. Without loss of generality, we take the first camera frame as the world frame. Then, the projection model for the first camera is 
\begin{equation} \label{projection1}
	{\bf y}_i=\frac{{\bf W}{\bf x}_i}{{\bf e}_3^\top {\bf x}_i}= \begin{bmatrix}
		x_{i1}/x_{i3} \\
		x_{i2}/x_{i3}
	\end{bmatrix},
\end{equation}
where ${\bf e}_i$ denotes the unit vector whose $i$-th element is 1, and ${\bf W}=[{\bf e}_1~{\bf e}_2]^\top$. Let $(\bf R,\bf t)$ be the relative pose of the first camera w.r.t. the second one. Then, the coordinates of the $i$-th 3D point in the second camera frame is ${\bf R}{\bf x}_i+{\bf t}$, and the corresponding projection model is 
\begin{equation} \label{projection2}
	{\bf z}_i=\frac{{\bf W}({\bf R}{\bf x}_i+{\bf t})}{{\bf e}_3^\top ({\bf R}{\bf x}_i+{\bf t})}.
\end{equation}

\begin{figure}[!t]
	\centering
	\includegraphics[width=0.68\columnwidth]{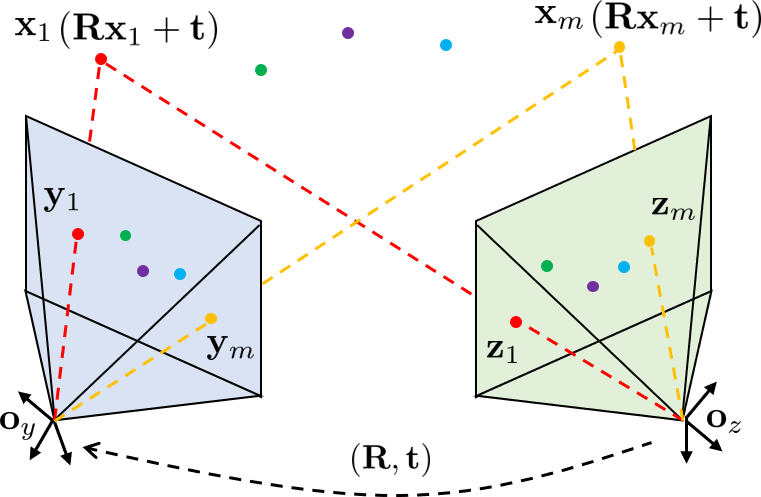}
	\caption{Two-view geometry.}
	\label{two_view_geometry}
\end{figure}

Recall that only the direction of $\bf t$ can be identified, while its scale cannot. If ${\bf t}=0$, its direction becomes arbitrary, and the identifiability condition fails. Therefore, for the two camera centers, we make the following assumption:
\begin{assumption} \label{nonzero_t}
	The centers of the two cameras do not coincide, i.e., $\|{\bf t}\| \neq 0$. 
\end{assumption}
Assumption~\ref{nonzero_t} serves as a necessary condition to estimate the translation (up to a scale). When $\|{\bf t}\| = 0$, the direction of ${\bf t} $ cannot be identified, and the CRB goes to infinity, see Figure~\ref{MSE_vs_t}. In this case, the homography matrix should be used instead of the essential matrix. 
Apart from Assumption~\ref{nonzero_t}, we also need an assumption on the spatial distribution of 3D points and the two camera centers to ensure the global identifiability of the relative pose. Before that, we will introduce the notion of ruled quadric surfaces. A quadric is a surface in the projective space $\mathbb P^3$ and is defined to be the set of points ${\bf x}^h$ such that ${\bf x}^{h \top}{\bf Q}{\bf x}^h=0$, where ${\bf Q}$ is a symmetric $4\times 4$ matrix. A ruled quadric surface is a quadric surface that contains a straight line. It includes a hyperboloid of one sheet, a cone, two (intersecting) planes, a single plane, and a single line~\cite{hartley2003multiple}.

\begin{assumption} \label{not_coplanar_assump}
	The 3D points $\{{\bf x}_i\}_{i=1}^m$ and the two camera centers 
	${\bf o}_y$ and ${\bf o}_z$ do not lie on a ruled quadric surface. 
\end{assumption}

Assumption~\ref{not_coplanar_assump} ensures that there do not exist two different configurations of $\{{\bf x}_i\}_{i=1}^m$, $\bf R$ and $\bf t$ (up to a scale) that have the same projections $\{{\bf y}_i\}_{i=1}^m$ and $\{{\bf z}_i\}_{i=1}^m$~\cite{hartley2003multiple}. In other words, $\bf R$ can be uniquely identified, and $\bf t$ can be uniquely identified up to a scale. A complete enumeration of the types of placement on a ruled quadric surface is given in~\cite[Result 22.11]{hartley2003multiple}. In general, the feature points are randomly distributed and will not concentrate on any ruled quadric surface. The most likely adverse scenario is that feature points concentrate on a man-made plane, e.g., a wall. In this case, CME algorithms may have unstable performance, which will be shown in Section~\ref{section_experiment}.   

By combining~\eqref{projection1} and~\eqref{projection2}, we obtain the following relationship between ${\bf z}_i$ and ${\bf y}_i$:
\begin{equation} \label{measurement_model}
	\begin{split}
		{\bf z}_i & =\frac{{\bf W}({\bf R}{\bf y}_i^h+{\bf t}/x_{i3})}{{\bf e}_3^\top ({\bf R}{\bf y}_i^h+{\bf t}/x_{i3})} \\
		&=\frac{{\bf W}({\bf R}{\bf y}_i^h+k_i \bar {\bf t})}{{\bf e}_3^\top ({\bf R}{\bf y}_i^h+k_i \bar {\bf t})},
	\end{split}
\end{equation}
where $\bar {\bf t}={\bf t}/\|{\bf t}\|$, and $k_i=\|{\bf t}\|/x_{i3}$. It is well-recognized that only the direction of $\bf t$, i.e., $\bar {\bf t}$ can be estimated from two-view geometry, while the length of $\bf t$, i.e., $\|\bf t\|$ cannot be recovered~\cite{hartley2003multiple}. This forms the motivation that we use $\bar {\bf t}$ in measurement model~\eqref{measurement_model}.
Note that the measurement equation~\eqref{measurement_model} models the noise-free relationship between ${\bf z}_i$ and ${\bf y}_i$. In real applications, there exist measurement noises due to, for example, a nonideal pinhole model, the inaccuracy of intrinsic matrix calibration, the spatial inconsistency of feature points, etc. Hence, the real measurement model should be
\begin{equation} \label{noisy_measurement_model}
	{\bf z}_i =\frac{{\bf W}({\bf R}{\bf y}_i^h+k_i \bar {\bf t})}{{\bf e}_3^\top ({\bf R}{\bf y}_i^h+k_i \bar {\bf t})} + {\bm \epsilon}_i,
\end{equation}
where ${\bm \epsilon}_i$ is the measurement noise. 
\begin{assumption} \label{noise_assump}
	The measurement noises ${\bm \epsilon}_i \sim \mathcal N (0,\sigma^2 {\bf I}_2),i=1,\ldots,m$ are independent and identically distributed (i.i.d.) with unknown variance $\sigma^2$.
\end{assumption}
The i.i.d. Gaussian noise assumption has been widely adopted in estimations in computer vision, e.g.,~\cite{lepetit2009epnp,hesch2011direct,urban2016mlpnp}. We remark that when ${\bf y}_i$ and ${\bf z}_i$ correspond to the same ${\bf x}_i$, model~\eqref{noisy_measurement_model} essentially characterizes the ``error in one image'' case that is frequently assumed in~\cite{hartley2003multiple}. The verification of Assumption~\ref{noise_assump} is presented in Appendix G.
Note that in~\eqref{noisy_measurement_model}, we have established the relationship between ${\bf z}_i$ and ${\bf y}_i$, and the measurement noise is in its most natural form. Hence, we call~\eqref{noisy_measurement_model} the \emph{original measurement model} and will use it to construct the ML problem. Specifically, the residual is 
\begin{equation} \label{residual}
	{\bf r}_i={\bf z}_i- \frac{{\bf W}({\bf R}{\bf y}_i^h+k_i \bar {\bf t})}{{\bf e}_3^\top ({\bf R}{\bf y}_i^h+k_i \bar {\bf t})}.
\end{equation}
With point correspondences $\{({\bf y}_i,{\bf z}_i)\}_{i=1}^{m}$, we can formulate the ML problem as
\begin{subequations}\label{LS_problem}
	\begin{align}
		\mathop{\rm minimize~}\limits_{{\bf R},\bar {\bf t},\{k_i\}} ~& \frac{1}{m} \sum_{i=1}^{m} \|{\bf r}_i\|^2  \label{LS_objective} \\
		\mathop{\rm subject~to~} ~& {\bf R} \in {\rm SO}(3) \label{LS_constraint1}\\
		& \bar {\bf t} \in S^2 \label{LS_constraint2}\\
		& k_i > 0,i=1,\ldots,m.\label{LS_constraint3}
	\end{align}
\end{subequations}

Given Assumptions~\ref{nonzero_t}-\ref{noise_assump}, the global solution to the ML problem~\eqref{LS_problem} (called the ML estimator) is consistent and asymptotically statistically-efficient. However, the ML problem~\eqref{LS_problem} is nonconvex, and finding its global solution is nontrivial. When using local iterations, e.g., the GN algorithm, an appropriate initial solution is needed, otherwise, it will converge to local minima. In the next section, we will propose a consistent estimator. The resulting consistent solution serves as a good initial value in the sense that GN iterations converge to the global minimizer of~\eqref{LS_problem} in the asymptotic case. 
\begin{remark}
    Most literature adopts the algebraic error $\left({\bf z}_i^{h \top} {\bf E} {\bf y}_i^h \right)^2$ which stems from the basic equation~\eqref{essential_matrix_eqn} of the essential matrix and formulates the following problem~\cite{zhao2020efficient,chesi2008camera,helmke2007essential}:
\begin{subequations}\label{essential_estimation}
	\begin{align}
		\mathop{\rm minimize~}\limits_{{\bf E}} ~& \frac{1}{m} \sum_{i=1}^{m} \left({\bf z}_i^{h \top} {\bf E} {\bf y}_i^h \right)^2  \label{algebraic_error} \\
		\mathop{\rm subject~to~} ~& {\bf E} \in \mathcal M_E.\label{E_constraint}
	\end{align}
\end{subequations}
Problem~\eqref{essential_estimation} is not an ML formulation, since different terms in the objective may have varied uncertainties, as will be shown in~\eqref{essential_equation}. As a result, it cannot achieve asymptotic efficiency.  
\end{remark}

\section{Consistent Estimator Design} \label{section_consistent_estimator}
In this section, we focus on the design of a $\sqrt{m}$-consistent estimator. The resulting estimators $\hat {\bf R}_m$ and $\hat{\bar {{\bf t}}}_m$ satisfy 
\begin{equation}
	\hat {\bf R}_m-{\bf R}^o=O_p(1/\sqrt{m}), ~~\hat{\bar {{\bf t}}}_m-\bar {\bf t}^o=O_p(1/\sqrt{m}),
\end{equation}
where ${\bf R}^o$ and $\bar {\bf t}^o$ are ground truth. The design consists of two procedures: First, we provide a consistent estimator of noise variance by calculating the maximum eigenvalue of a $9 \times 9$ matrix; We then execute bias elimination based on the estimate of noise variance to obtain a consistent solution. 

\subsection{Consistent noise variance estimation} \label{subsection_noise_estimation}
Regarding an estimator with finite variance, asymptotic unbiasedness is a necessary condition for consistency~\cite{lehmann2006theory}. On the one hand, in nonlinear nonconvex optimization, estimators obtained by relaxation are usually biased, even in the asymptotic case. On the other hand, the bias of an estimator is generally a function of the variance of measurement noises. Therefore, noise variance estimation is a prerequisite for bias elimination and the construction of a consistent solution. Let ${\bm \epsilon}_i^h=[{\bm \epsilon}_i^\top~0]^\top$ denote the homogeneous noise. Based on epipolar geometry constraint, we obtain the equation $ ({\bf z}_i^h-{\bm \epsilon}_i^h)^\top {\bf E} {\bf y}_i^{h}=0$, i.e., 
\begin{equation} \label{essential_equation}
	0={\bf z}_i^{h \top} {\bf E} {\bf y}_i^h-{\eta}_i,
\end{equation}
where ${\eta}_i={\bm \epsilon}_i^{h \top} {\bf E} {\bf y}_i^h$ is the new noise term. 
Let ${\bf L}_i={\bf y}_i^{h \top} \otimes {\bf I}_3 \in \mathbb R^{3 \times 9}$ and ${\bm \theta}={\rm vec}({\bf E}) \in \mathbb R^9$, we can rewrite~\eqref{essential_equation} as
\begin{equation} \label{essential_equation2}
	0={\bf z}_i^{h \top} {\bf L}_i {\bm \theta}-{\eta}_i.
\end{equation}
By stacking~\eqref{essential_equation2} for all $i\in\{1,\ldots,m\}$, we obtain the matrix form:
\begin{equation} \label{essential_matrix_form}
	{\bf 0}={\bf A}_m {\bm \theta}-{\bm \eta}_m,
\end{equation}
where 
\begin{equation*}
	{\bf A}_m=\begin{bmatrix}
		{\bf z}_1^{h \top} {\bf L}_1 \\
		\vdots \\
		{\bf z}_m^{h \top} {\bf L}_m
	\end{bmatrix},~~{\bm \eta}_m=\begin{bmatrix}
		{\eta}_1 \\
		\vdots \\
		{\eta}_m
	\end{bmatrix}.
\end{equation*}

Define 
\begin{equation} \label{bf_Q_m}
	{\bf Q}_m=\frac{{\bf A}_m^\top {\bf A}_m}{m}
\end{equation}
and
\begin{equation} \label{bf_S_m}
	{\bf S}_m={\bf Y}^{h}\otimes {\bf W}^h
\end{equation}
where ${\bf Y}^{h}=\sum_{i=1}^{m} {\bf y}_i^{h} {\bf y}_i^{h \top}/m$ and ${\bf W}^h=[{\bf W}^\top~{\bf 0}_3]^\top$. The following theorem gives a consistent estimator of the variance $\sigma^2$ of the measurement noises ${\bm \epsilon}_i$'s.
\begin{theorem} \label{theo_noise_est}
	Let $\hat \sigma_m^2=1/\lambda_{\rm max} ({\bf Q}_m^{-1} {\bf S}_m)$. Then,under Assumptions~\ref{nonzero_t}-\ref{noise_assump}, $\hat \sigma_m^2$ is a $\sqrt{m}$-consistent estimator of $\sigma^2$.
\end{theorem}
The proof is presented in Appendix A.

\subsection{Bias elimination and eigendecomposition} \label{subsection_consistent_solution}
With the $\sqrt{m}$-consistent estimator of noise variance given in Theorem~\ref{theo_noise_est}, we are ready to propose a consistent estimator of the essential matrix $\bf E$. With the consistent estimator of $\bf E$, we can further recover the consistent estimators of $\bf R$ and $\bar {\bf t}$. 
The consistent estimator is derived based on~\eqref{essential_matrix_form} and is tightly related to the classic $8$-point algorithm~\cite{hartley1997defense}. The $8$-point algorithm optimally solves the problem
\begin{subequations}\label{8_point_formulation}
	\begin{align}
		\mathop{\rm minimize~}\limits_{{\bm \theta}} ~& \|{\bf A}_m {\bm \theta}\|^2  \label{8_point_obj} \\
		\mathop{\rm subject~to~} ~& \|{\bm \theta}\|=1. \label{8_point_cons}
	\end{align}
\end{subequations}
Denote the SVD of ${\bf A}_m$ as ${\bf A}_m={\bf U}{\bf D}{\bf V}^\top$. Then, the singular vector ${\bf v}_9$ corresponding to the smallest singular value of ${\bf A}_m$, i.e., the last column of $\bf V$ is the global minimizer of problem~\eqref{8_point_formulation}~\cite{hartley2003multiple}. It can be verified that this is equivalent to finding the eigenvector corresponding to the smallest eigenvalue of ${\bf Q}_m$ (Recall that ${\bf Q}_m={\bf A}_m^\top {\bf A}_m/m$)~\cite{horn2012matrix}. 
The visual interpretation of the optimal solution to problem~\eqref{8_point_formulation} in 2D dimension is shown in Figure~\ref{interpretation_eigenvector}. The matrix ${\bf A}_m$ stretches the unit circle to an ellipse. The shortest radius of the ellipse corresponds to the smallest singular value $\sigma_2$, and the unit vector ${\bf v}_2$ is the singular vector corresponding to $\sigma_2$. 
Although the formulation~\eqref{8_point_formulation} arose from the $8$-point problem, it is essentially a constrained least-squares formulation of~\eqref{essential_matrix_form} and can be used in arbitrary point number $m$. 

To construct a consistent estimator of $\bm \theta$, first, we consider the noise-free problem
\begin{subequations}\label{noise_free_LS}
	\begin{align}
		\mathop{\rm minimize~}\limits_{{\bm \theta}} ~& \|{\bf A}_m^o {\bm \theta}\|^2  \label{noise_free_obj} \\
		\mathop{\rm subject~to~} ~& \|{\bm \theta}\|=1. \label{noise_free_cons}
	\end{align}
\end{subequations}
Since~\eqref{noise_free_LS} is noise-free, under the identifiability Assumption~\ref{not_coplanar_assump}, the optimal solutions of~\eqref{noise_free_LS} equal the true ${\bm \theta}^o={\rm vec}(\bar {\bf t}^{o \wedge} {\bf R}^o)$ up to a scale. 
As noted before, the unit eigenvectors associated with the smallest eigenvalue of ${\bf Q}_m^o$ are the optimal solutions to~\eqref{noise_free_LS}. However, as the noise-free counterpart of ${\bf Q}_m$, ${\bf Q}_m^o$ is unavailable in practice. One natural idea is to eliminate the bias between ${\bf Q}_m$ and ${\bf Q}_m^o$ in the asymptotic case and make the bias-eliminated quantity ${\bf Q}_m^{\rm BE}$ converge to ${\bf Q}_m^o$. 
Note that we can obtain a $\sqrt{m}$-consistent estimator $\hat \sigma_m^2$ of $\sigma^2$. Let
\begin{equation} \label{bias_elimination}
	{\bf Q}_m^{\rm BE}={\bf Q}_m-\hat \sigma_m^2 {\bf S}_m.
\end{equation}
From (24) in the supplementary materials, we have
\begin{equation} \label{bias_elimination_result}
	{\bf Q}_m^{\rm BE}={\bf Q}_m^o+O_p(1/\sqrt{m}),
\end{equation}
which means that ${\bf Q}_m^{\rm BE}$ is a $\sqrt{m}$-consistent estimator of ${\bf Q}_m^o$.

We then select any one of eigenvectors of ${\bf Q}_m^{\rm BE}$ as $\hat {\bm \theta}_m^{\rm BE}$. The essential matrix estimate $\hat {\bf E}_m^{\rm BE}$ can be obtained via inverse vectorization from $\hat {\bm \theta}_m^{\rm BE}$. 
Given the estimate of the essential matrix $\hat {\bf E}_m^{\rm BE}$, we can recover $\hat{\bar {{\bf t}}}_m^{\rm BE}$ and $\hat {\bf R}_m^{\rm BE}$ by the SVD. To do this, there are two points worth mentioning~\cite{hartley2003multiple}. First, the recovery result is invariant to scale, that is, for every $k \neq 0$, $k \hat {\bf E}_m^{\rm BE}$ yields the same results as $\hat {\bf E}_m^{\rm BE}$. Second, the recovery is not unique --- it produces four pairs of solutions, and the correct one needs to be selected as the final estimate. The selection criterion is that the triangulated 3D points with the solution should locate in front of two cameras. We utilize the {\bf recoverPose} function in OpenCV to complete this task. Finally, we give a theorem to end this section:
\begin{theorem} \label{the_consistent_solution}
	The solutions $\hat{\bar {{\bf t}}}_m^{\rm BE}$ and $\hat {\bf R}_m^{\rm BE}$ are $\sqrt{m}$-consistent estimators of $\bar {\bf t}^o$ and ${\bf R}^o$.
\end{theorem}
\begin{proof}
	From~\eqref{bias_elimination_result} we know that ${\bf Q}_m^{\rm BE}$ is a $\sqrt{m}$-consistent estimator of ${\bf Q}_m^o$. Recall that we obtain $\hat {\bf E}_m^{\rm BE}$ from ${\bf Q}_m^{\rm BE}$ by eigendecomposition and inverse vectorization and recover $\hat{\bar {{\bf t}}}_m^{\rm BE}$ and $\hat {\bf R}_m^{\rm BE}$ from $\hat {\bf E}_m^{\rm BE}$ via SVD. Since eigendecomposition, inverse vectorization, and SVD are all continuous functions that can preserve the property of $\sqrt{m}$-consistency, the proof is completed. 
\end{proof}
 \begin{figure}[!t]
    \centering
  \includegraphics[width=0.52\linewidth]{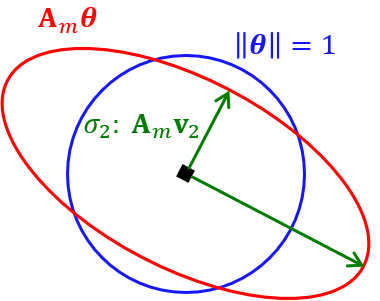}
   \caption{Illustration of the optimal solution to problem~\eqref{8_point_formulation} in 2D dimension.}
            \label{interpretation_eigenvector}
\end{figure}

\section{Gauss-Newton Iterations for Rotation and Normalized Translation} \label{section_GN}
Recall that our ultimate goal is to optimally solve the ML problem~\eqref{LS_problem} in the asymptotic case. We achieve this via a two-step scheme. As shown in Section~\ref{section_consistent_estimator}, in the first step, we have obtained consistent estimators $\hat{\bar {{\bf t}}}_m^{\rm BE}$ and $\hat {\bf R}_m^{\rm BE}$. It is noteworthy to see that although $\hat{\bar {{\bf t}}}_m^{\rm BE}$ and $\hat {\bf R}_m^{\rm BE}$ own consistency, they are not optimal in the sense of minimum variance.
Nonetheless, the consistency enables them to be a theoretically guaranteed initial value in the asymptotic case. As the number of points increases, $\hat{\bar {{\bf t}}}_m^{\rm BE}$ and $\hat {\bf R}_m^{\rm BE}$ will converge into the attraction region of the global solution to~\eqref{LS_problem}, which ensures that GN iterations optimally solve~\eqref{LS_problem}. Therefore, in the second step, we conduct GN iterations to refine the consistent estimators $\hat{\bar {{\bf t}}}_m^{\rm BE}$ and $\hat {\bf R}_m^{\rm BE}$.  

To let the refined solution converge to the ML estimator, GN iterations should be derived from the ML problem~\eqref{LS_problem}. However, this comes with two difficulties. First, in addition to $\bar {\bf t}$ and $\bf R$, there are $m$ unknown $k_i$ in~\eqref{LS_problem}. If we involve updating $k_i$'s in the GN iterations, the Jacobian matrix $\bf J$ has an expanding size as $m$ increases, and calculating the inverse of ${\bf J}^\top {\bf J}$ needs $O(m^3)$ time complexity, which is computationally inefficient. Second, the GN iterations need to meet two constraints, that is, the rotation matrix belongs to ${\rm SO}(3)$, and the normalized translation vector is on the 2-sphere. In what follows, we first eliminate $k_i$'s by resorting to the Karush-Kuhn-Tucker (KKT) conditions~\cite{boyd2004convex}. We then derive the GN iteration formulation on the ${\rm SO}(3)$ group and 2-sphere. 

\subsection{Variable elimination via KKT conditions} \label{subsection_variable_elimination}
In this subsection, we express optimal $k_i$'s with $\bar {\bf t}$ and $\bf R$ by utilizing the KKT conditions. By doing so, the Jacobian matrix $\bf J$ associated with only $\bar {\bf t}$ and $\bf R$ can be used, so the matrix ${\bf J}^\top {\bf J}$ has a fixed size. Specifically, we denote the objective function of~\eqref{LS_problem} as $f_m$. According to the KKT stationary condition, we have
\begin{equation} \label{stationary_condition}
	\frac{\partial f_m}{\partial k_i} + \lambda_i=0,
\end{equation}
where $\lambda_i$ is the Lagrange multiplier corresponding to $k_i$. In addition, the KKT complementary slackness condition requires that $\lambda_i k_i=0,i=1,\ldots,m$. 
\begin{lemma} \label{optimal_multiplier}
    In the asymptotic case, i.e., when the point number $m$ is sufficiently large, the optimal $\lambda_i^*$ satisfies that $\lambda_i^*=0,i=1,\ldots,m$.
\end{lemma}
The proof of Lemma~\ref{optimal_multiplier} is presented in Appendix B. 
By combining~\eqref{stationary_condition} and Lemma~\ref{optimal_multiplier}, we have $\partial f_m/\partial k_i=0$. Hence, we can express optimal $k_i$'s with $\bar {\bf t}$ and $\bf R$. Here we omit the tedious derivation and directly give the result:
\begin{equation} \label{expression_ki}
	k_i=\frac{{\bf y}_i^{h \top}{\bf R}^\top {\bf C}_1 ({\bf I}_3 \otimes \bar {\bf t}) {\bf R}{\bf y}_i^{h}}      {\bar {\bf t}^\top {\bf C}_2 ({\bf I}_3 \otimes \bar {\bf t}) {\bf R}{\bf y}_i^{h}},
\end{equation}
where 
\begin{align*}
	{\bf C}_1 &=\begin{bmatrix}
		-{\bf e}_3^\top & {\bf 0}_3^\top & z_{i1}{\bf e}_3^\top\\ 
		{\bf 0}_3^\top & -{\bf e}_3^\top & z_{i2}{\bf e}_3^\top\\
		{\bf e}_1^\top & {\bf e}_2^\top & -(z_{i1}{\bf e}_1^\top+z_{i2}{\bf e}_2^\top)
	\end{bmatrix}, \\
        {\bf C}_2 &=\begin{bmatrix}
		{\bf e}_3^\top & {\bf 0}_3^\top & z_{i1}{\bf e}_3^\top-{\bf e}_1^\top\\ 
		{\bf 0}_3^\top & {\bf e}_3^\top & z_{i2}{\bf e}_3^\top-{\bf e}_2^\top\\
		-z_{i1}{\bf e}_3^\top & -z_{i2}{\bf e}_3^\top & {\bf 0}_3^\top
	\end{bmatrix}.
\end{align*}
\subsection{Gauss-Newton iterations on SO(3) and 2-sphere} \label{subsection_GN}
Since the GN algorithm is an extension of Newton's method, when the initial guess is near the minimum, the rate of its convergence can approach quadratic. Thanks to the $\sqrt{m}$-consistent property of $\hat {\bf R}_m^{\rm BE}$ and $\hat{\bar {{\bf t}}}_m^{\rm BE}$, when the point number $m$ is large, they are sufficiently near the global minimum of the ML problem~\eqref{LS_problem}, and only a one-step GN iteration suffices to achieve the same asymptotic property (i.e., asymptotically statistically-efficient) as the ML estimator $\hat {\bf R}_m^{\rm ML}$ and $\hat{\bar {{\bf t}}}_m^{\rm ML}$, which will be formally stated in Theorem~\ref{asymptotic_efficiency_theorem}.
Before that, we take $\hat {\bf R}_m^{\rm BE}$ and $\hat{\bar {{\bf t}}}_m^{\rm BE}$ as the initial value and derive the one-step GN iteration. 

Note that the rotation matrix belongs to ${\rm SO}(3)$, and the normalized translation locates on the 2-sphere. Hence, we are going to derive the GN iteration formulation on ${\rm SO}(3)$ and 2-sphere. 
For the ${\rm SO}(3)$ constraint, given the initial estimate $\hat {\bf R}_m^{\rm BE}\in{\rm SO}(3)$ and any ${\bf s} \in \mathbb R^3$, the matrix $\hat {\bf R}_m^{\rm BE} \exp({\bf s}^\wedge)$ also belongs to ${\rm SO}(3)$. Hence, we can update the unconstrained vector $\bf s$ to guarantee the refined rotation matrix estimate is still in ${\rm SO}(3)$. 
For the 2-sphere constraint, let 
\begin{align*}
	\alpha_0 &=\sin^{-1} \left( \hat {\bar {t}}_{m3}^{\rm BE}\right), \\
	\beta_0 &=\begin{cases}
		\tan^{-1} \left( \hat {\bar {t}}_{m2}^{\rm BE}/\hat {\bar {t}}_{m1}^{\rm BE}\right),  ~{\rm if} ~~\hat {\bar {t}}_{m1}^{\rm BE}>0 \\
		\tan^{-1} \left( \hat {\bar {t}}_{m2}^{\rm BE}/\hat {\bar {t}}_{m1}^{\rm BE}\right)+180^{\rm o}, ~{\rm if} ~~\hat {\bar {t}}_{m1}^{\rm BE}<0,
	\end{cases}
\end{align*}
where we express $\hat{\bar {{\bf t}}}_{m}^{\rm BE}$ as $\hat{\bar {{\bf t}}}_{m}^{\rm BE}=\left[\hat {\bar {t}}_{m1}^{\rm BE} ~~\hat {\bar {t}}_{m2}^{\rm BE}~~\hat {\bar {t}}_{m3}^{\rm BE}\right]^\top$. Given any $\alpha,\beta \in \mathbb R$, the vector
\begin{equation*}
	\bar {\bf t}(\alpha,\beta)=\begin{bmatrix}
		\cos (\alpha_0+\alpha)\cos (\beta_0+\beta) \\
		\cos (\alpha_0+\alpha)\sin (\beta_0+\beta) \\
		\sin (\alpha_0+\alpha)
	\end{bmatrix}
\end{equation*}
is still on the $2$-sphere. Hence, we can update $\alpha$ and $\beta$ to refine the normalized translation estimate. Let $\hat{{\bf s}}_m^{\rm GN}$, $\hat \alpha_m^{\rm GN}$, and $\hat \beta_m^{\rm GN}$ denote the results obtained by a one-step GN iteration. Then the refined rotation matrix and normalized translation vector are given as
\begin{equation} \label{GN_R_and_t}
	\hat{{\bf R}}^{\rm GN}_m=\hat{{\bf R}}_m^{\rm BE}\exp\left({{}\hat{{\bf s}}_m^{\rm GN}}^{\wedge} \right),~~\hat{\bar {{\bf t}}}_m^{\rm GN}=\bar {\bf t}(\hat \alpha_m^{\rm GN},\hat \beta_m^{\rm GN}).
\end{equation}

For the explicit derivation of $\hat{{\bf s}}_m^{\rm GN}$, $\hat \alpha_m^{\rm GN}$, and $\hat \beta_m^{\rm GN}$, one can refer to Appendix C. 

\begin{theorem} \label{asymptotic_efficiency_theorem}
    Denote the one-step GN iteration of the $\sqrt{m}$-consistent estimators $\hat {\bf R}_m^{\rm BE}$ and $\hat{\bar {{\bf t}}}_m^{\rm BE}$ by $\hat {\bf R}_m^{\rm GN}$ and $\hat{\bar {{\bf t}}}_m^{\rm GN}$, respectively. Then, 
    \begin{equation*}
        \hat {\bf R}_m^{\rm ML}-\hat {\bf R}_m^{\rm GN}=o_p(1/\sqrt{m}),~~\hat{\bar {{\bf t}}}_m^{\rm ML}-\hat{\bar {{\bf t}}}_m^{\rm GN}=o_p(1/\sqrt{m}).
    \end{equation*}
\end{theorem}
The proof of Theorem~\ref{asymptotic_efficiency_theorem} is presented in Appendix D. Theorem~\ref{asymptotic_efficiency_theorem} implies that $\hat {\bf R}_m^{\rm GN}$ and $\hat{\bar {{\bf t}}}_m^{\rm GN}$ have the same asymptotic property as the ML estimator $\hat {\bf R}_m^{\rm ML}$ and $\hat{\bar {{\bf t}}}_m^{\rm ML}$.

By now we have introduced the whole algorithm. In summary, it mainly consists of three modules: noise variance estimation, consistent solution construction, and a one-step GN refinement. The proposed algorithm is summarized in Algorithm~\ref{pseudo_algorithm}, where we denote it as \texttt{CECME} --- {\bf C}onsistent and asymptotically statistically-{\bf E}fficient {\bf C}amera {\bf M}otion {\bf E}stimator. We remark that our algorithm has significant advantages in the asymptotic case. In terms of estimation accuracy, owing to the $\sqrt{m}$-consistency of $\hat {\bf R}_m^{\rm BE}$ and $\hat{\bar {{\bf t}}}_m^{\rm BE}$ obtained in the first step, only a single GN iteration in the second step will suffice to achieve the CRB asymptotically, which is also verified by our simulation results. We put the derivation of the CRB in Appendix E. In terms of time complexity, it can be verified that {\bf Lines}~\ref{line_1},\ref{line_9},\ref{line_11} in Algorithm~\ref{pseudo_algorithm} cost $O(m)$ time, and {\bf Lines}~\ref{line_3},\ref{line_5},\ref{line_7} cost $O(1)$ time. Therefore, the whole time complexity of \texttt{CECME} is $O(m)$, making it suitable for real-time implementation when the point number $m$ is large. The experiment results in the following section will demonstrate the superiority (in terms of MSE and CPU time) of the proposed algorithm over SOTA ones when the point number is large.  
\begin{algorithm}
	\caption{CECME}
	\label{pseudo_algorithm}
	\begin{algorithmic}[1]
		\Statex {\bf Input:} Point correspondences $\{({\bf y}_i,{\bf z}_i)\}_{i=1}^{m}$.
		\Statex {\bf Output:} The estimates of the rotation matrix $\hat{{\bf R}}^{\rm GN}_m$ and normalized translation vector $\hat{\bar {{\bf t}}}_m^{\rm GN}$. 
		\State Calculate the matrix ${\bf Q}_m$ in~\eqref{bf_Q_m} and ${\bf S}_m$ in~\eqref{bf_S_m} \label{line_1};
		\State Obtain noise variance estimate via $\hat \sigma_m^2=1/\lambda_{\rm max} ({\bf Q}_m^{-1} {\bf S}_m)$ \label{line_3} based on Thereom~\ref{theo_noise_est};
		\State Calculate ${\bf Q}_m^{\rm BE}={\bf Q}_m-\hat \sigma_m^2 {\bf S}_m$ \label{line_5} as in~\eqref{bias_elimination};
		\State Let $\hat {\bm \theta}_m^{\rm BE}$ be any eigenvector of ${\bf Q}_m^{\rm BE}$ associated with $\lambda_{\rm min}({\bf Q}_m^{\rm BE})$, as described below~\eqref{bias_elimination_result}\label{line_7};
		\State Recover $(\hat{\bar {{\bf t}}}_m^{\rm BE},\hat {\bf R}_m^{\rm BE})$ from $\hat {\bm \theta}_m^{\rm BE}$, as described below~\eqref{bias_elimination_result} \label{line_9};
		\State Execute a one-step GN iteration shown in~\eqref{GN_R_and_t} \label{line_11}.
	\end{algorithmic}
\end{algorithm}

\section{Experiment} \label{section_experiment}
In this section, we conduct experiments on both synthetic data and real images. The classical or SOTA methods compared with ours are
\begin{itemize}
	\item \texttt{Eigen}: the eigenvalue-based method proposed by Kneip and Lynen~\cite{kneip2013direct}
	\item \texttt{SDP}: the SDP-based method proposed by Zhao~\cite{zhao2020efficient}
	\item \texttt{GN-E}: the GN iterations on the manifold of normalized essential matrices proposed by Helmke \emph{et al.}~\cite{helmke2007essential}
    \item \texttt{DFE}: the learning-based fundamental matrix estimation approach proposed by Ranftl and Koltun~\cite{ranftl2018deep} 
    \item \texttt{NACNet}: a SOTA learning-based essential matrix estimation approach proposed by Moran \emph{et al.}~\cite{moran2024consensus}
    \item \texttt{LM}: the basic Levenberg-Marquardt algorithm with a robust kernel to minimize point-to-epipolar-line distances
\end{itemize}

We use open source codes for \texttt{Eigen}, \texttt{SDP}, \texttt{DFE}, and \texttt{NACNet} methods, and realize \texttt{GN-E} and \texttt{LM} methods by ourselves. Note that \texttt{Eigen}, \texttt{GN-E}, and \texttt{LM} need an initial guess of the rotation matrix or essential matrix. We take the RANSAC results with the five-point solver~\cite{nister2004efficient} as their inputs. The learning-based methods \texttt{DFE} and \texttt{NACNet} are only compared in experiments with real images.
We point out that in simulations with outliers and real image tests, in order to enhance robustness, we first implement the five-point RANSAC algorithm to clean the data for all algorithms. We also embed a truncated least-squares kernel in our GN iteration.

\subsection{Experiment with synthetic data} \label{subsection_simulation}
In the simulation, the translation is set as $[5~5~5]^\top {\rm cm}$, and the Euler angles are $[20^o~20^o~20^o]^\top$. The two cameras have the same intrinsic matrix, where the focal length is $f_x=f_y=800~ {\rm pixels}$ ($5 ~{\rm cm}$), and the size of the image plane is $640 \times 480 ~ {\rm pixels}$.  The principle point lies in the top-left corner of the image plane and the principle point offsets are $u_0=320 ~ {\rm pixels}$ and $v_0=240 ~ {\rm pixels}$. 
For the generation of 3D points that are visible in both images, we first randomly generate 2D points in the first image and then endow them each with a random depth within $[1,5]$ m. Only the 3D points whose projection in the second camera is within its image plane are selected as valid ones. 
As noted in Assumption~\ref{noise_assump}, the measurement is corrupted by a zero-mean Gaussian noise whose standard deviation is $\sigma$ pixels. 

The evaluation metric for estimation accuracy is mean squared error (MSE), which is defined as follows: 
\begin{align*}
	{\rm MSE}_{\bf R} & = \frac{1}{K} \sum_{k=1}^{K} \left\|\hat {\bf R}_k-{\bf R}^o \right\|_{\rm F}^2, \\
	{\rm MSE}_{\bar {\bf t}} & = \frac{1}{K} \sum_{k=1}^{K} \left\|\hat{\bar {{\bf t}}}_k-\bar {\bf t}^o \right\|^2,
\end{align*}
where $\hat {\bf R}_k$ and $\hat{\bar {{\bf t}}}_k$ are the estimates obtained in the $k$-th Monte Carlo test, and $K$ is the total number of Monte Carlo tests. We also present the bias of each estimator. The bias is given as
\begin{align*}
	\Delta {\bf R} & = \left|\frac{1}{K} \sum_{k=1}^{K} \hat {\bf R}_k-{\bf R}^o \right|, ~~{\rm Bias}_{\bf R}=\sum_{i=1}^{3} \sum_{j=1}^{3} \Delta {\bf R}_{ij}\\
	\Delta \bar {\bf t} & =\left| \frac{1}{K} \sum_{k=1}^{K} \hat{\bar {{\bf t}}}_k-\bar {\bf t}^o \right|,~~{\rm Bias}_{\bar {\bf t}}=\sum_{i=1}^{3} \Delta \bar {\bf t}_{i}.
\end{align*}

\begin{figure*}[!t]
	\centering
	\begin{subfigure}[b]{0.24\textwidth}
		\centering
		\includegraphics[width=\textwidth]{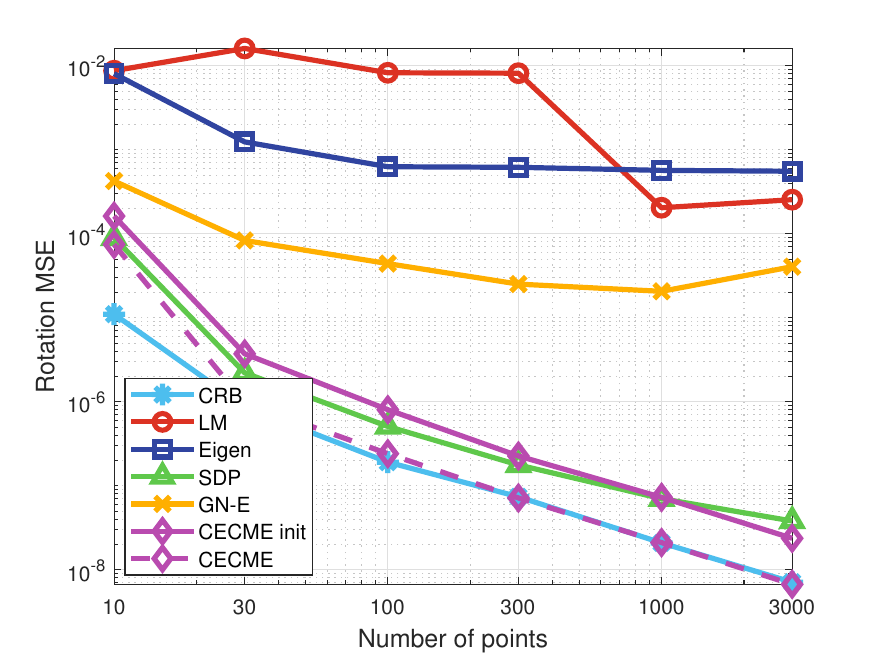}
		\caption{$\sigma=0.25$px ($\bf R$)}
		\label{MSE_R_025px}
	\end{subfigure}
	\begin{subfigure}[b]{0.24\textwidth}
		\centering
		\includegraphics[width=\textwidth]{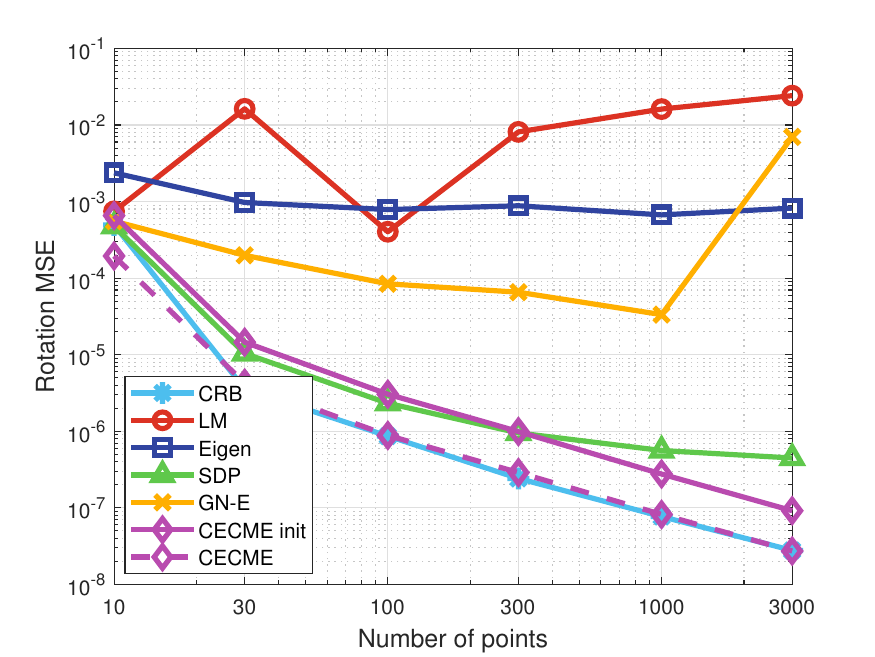}
		\caption{$\sigma=0.5$px ($\bf R$)}
		\label{MSE_R_05px}
	\end{subfigure}
	\begin{subfigure}[b]{0.24\textwidth}
		\centering
		\includegraphics[width=\textwidth]{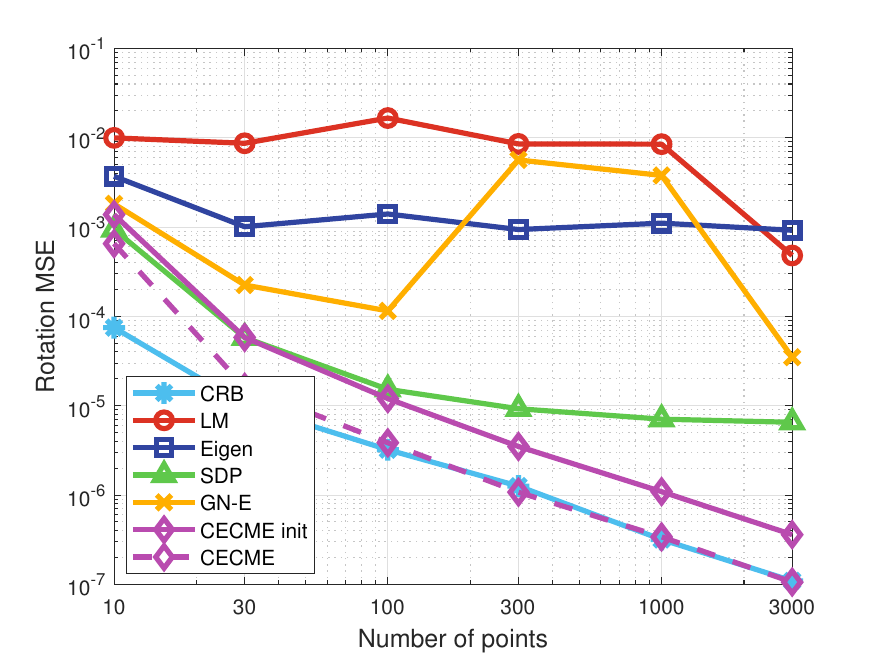}
		\caption{$\sigma=1$px ($\bf R$)}
		\label{MSE_R_1px}
	\end{subfigure}
	\begin{subfigure}[b]{0.24\textwidth}
		\centering
		\includegraphics[width=\textwidth]{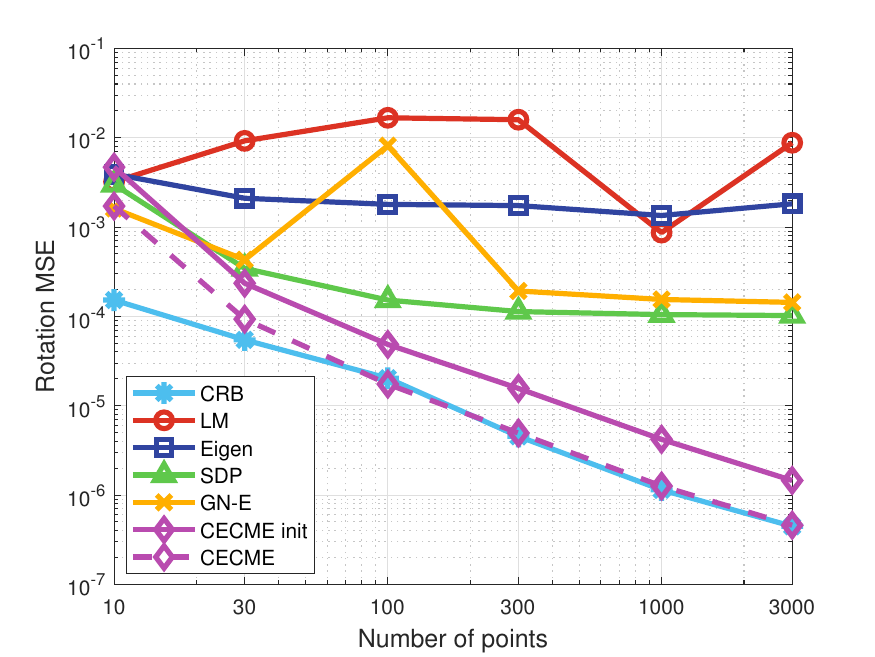}
		\caption{$\sigma=2$px ($\bf R$)}
		\label{MSE_R_2px}
	\end{subfigure}
	\begin{subfigure}[b]{0.24\textwidth}
		\centering
		\includegraphics[width=\textwidth]{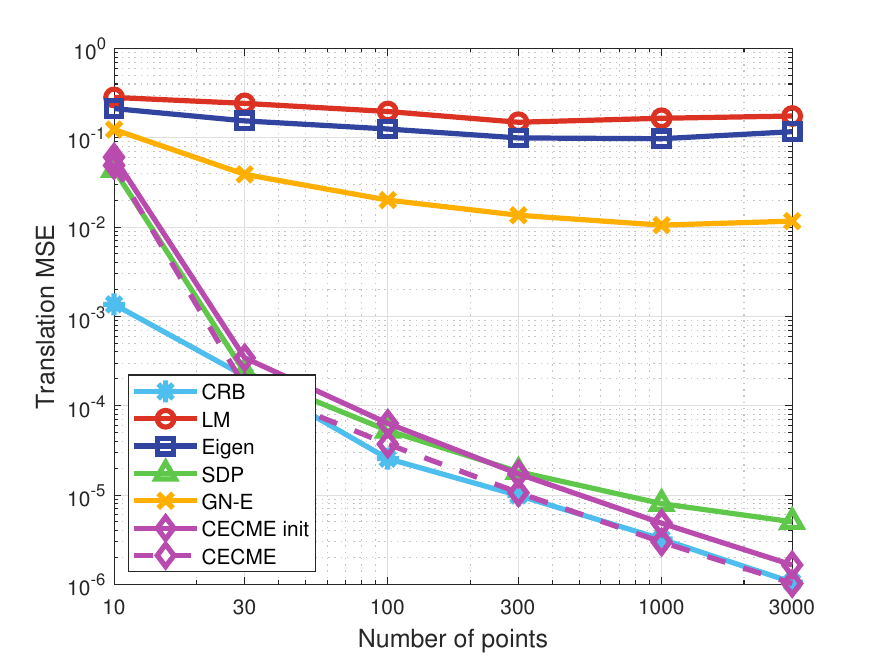}
		\caption{$\sigma=0.25$px ($\bar {\bf t}$)}
		\label{MSE_t_025px}
	\end{subfigure}
	\begin{subfigure}[b]{0.24\textwidth}
		\centering
		\includegraphics[width=\textwidth]{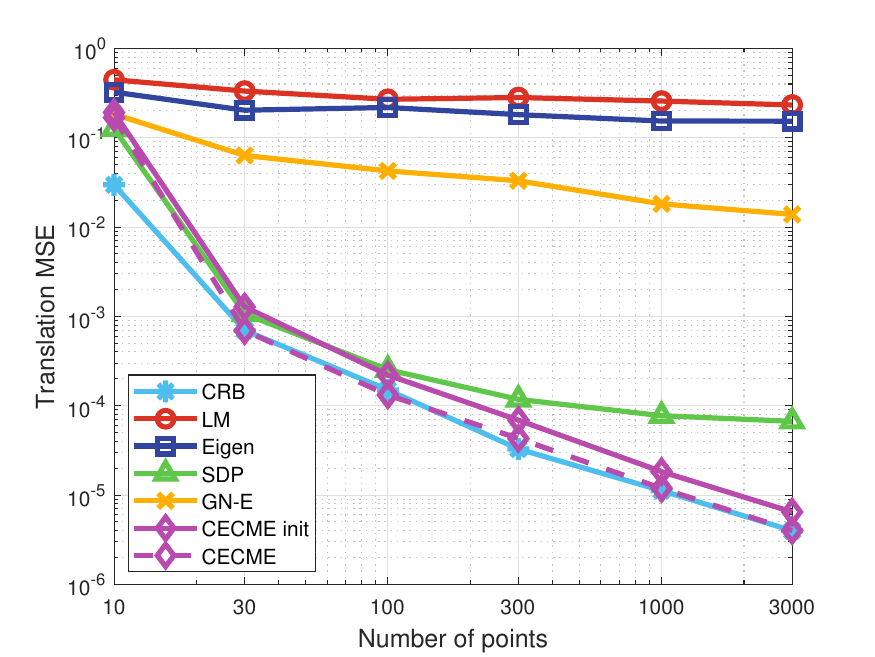}
		\caption{$\sigma=0.5$px ($\bar {\bf t}$)}
		\label{MSE_t_05px}
	\end{subfigure}
	\begin{subfigure}[b]{0.24\textwidth}
		\centering
		\includegraphics[width=\textwidth]{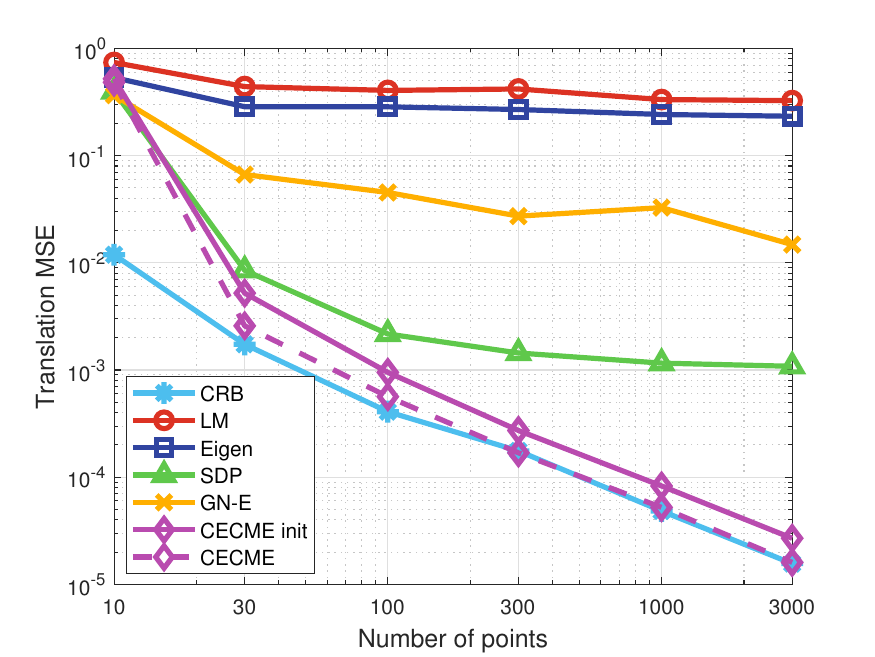}
		\caption{$\sigma=1$px ($\bar {\bf t}$)}
		\label{MSE_t_1px}
	\end{subfigure}
	\begin{subfigure}[b]{0.24\textwidth}
		\centering
		\includegraphics[width=\textwidth]{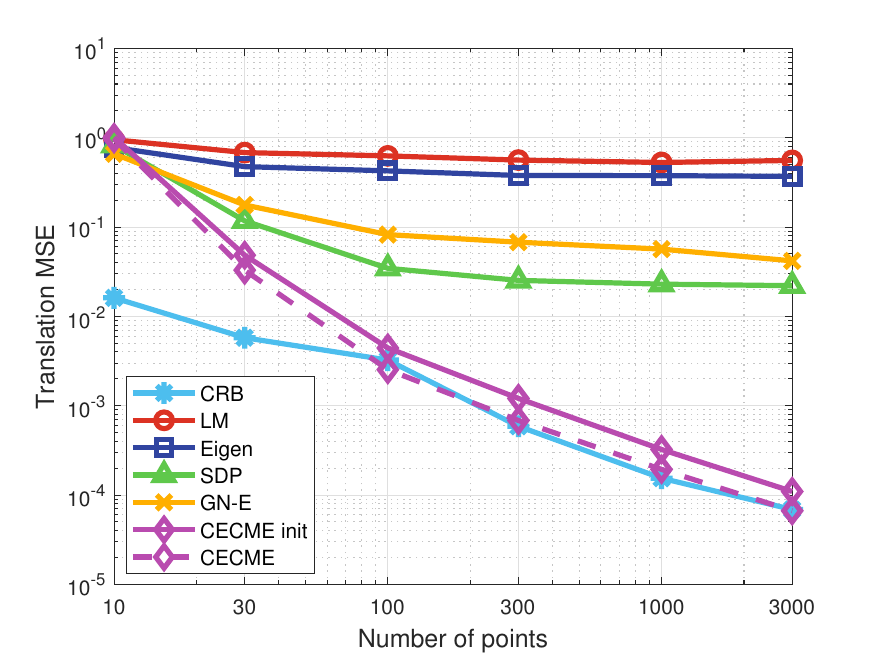}
		\caption{$\sigma=2$px ($\bar {\bf t}$)}
		\label{MSE_t_2px}
	\end{subfigure}
	\caption{MSE comparison under different noise intensities and point numbers.}
	\label{MSE_comparison}
\end{figure*}

\begin{figure*}[!t]
	\centering
	\begin{subfigure}[b]{0.24\textwidth}
		\centering
		\includegraphics[width=\textwidth]{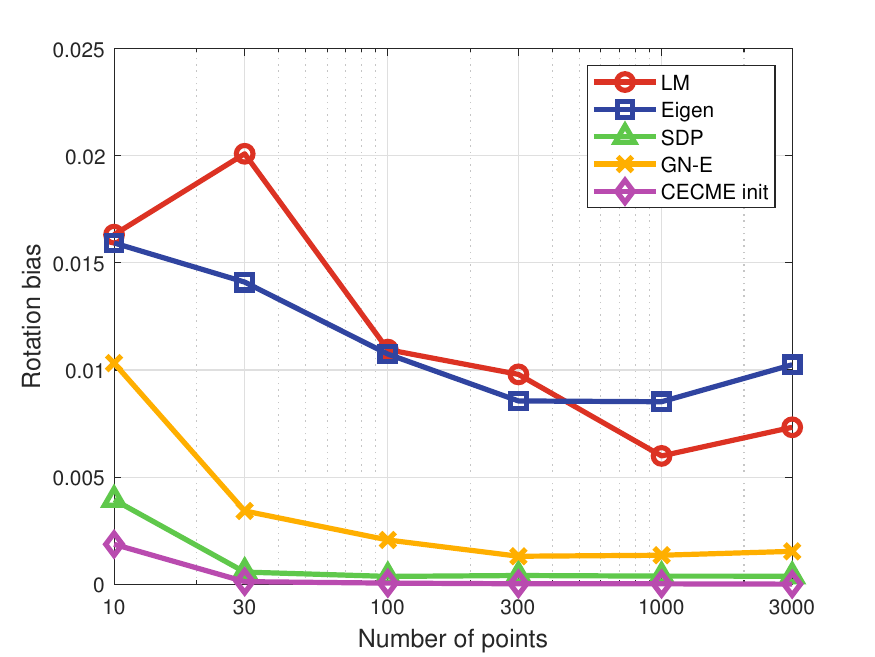}
		\caption{$\sigma=0.25$px ($\bf R$)}
		\label{bias_R_025px}
	\end{subfigure}
	\begin{subfigure}[b]{0.24\textwidth}
		\centering
		\includegraphics[width=\textwidth]{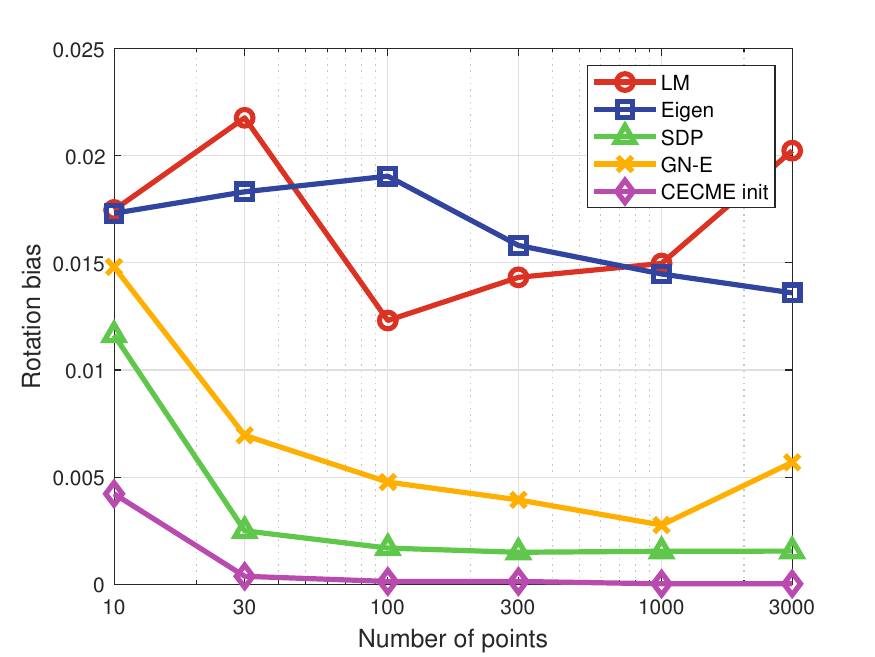}
		\caption{$\sigma=0.5$px ($\bf R$)}
		\label{bias_R_05px}
	\end{subfigure}
	\begin{subfigure}[b]{0.24\textwidth}
		\centering
		\includegraphics[width=\textwidth]{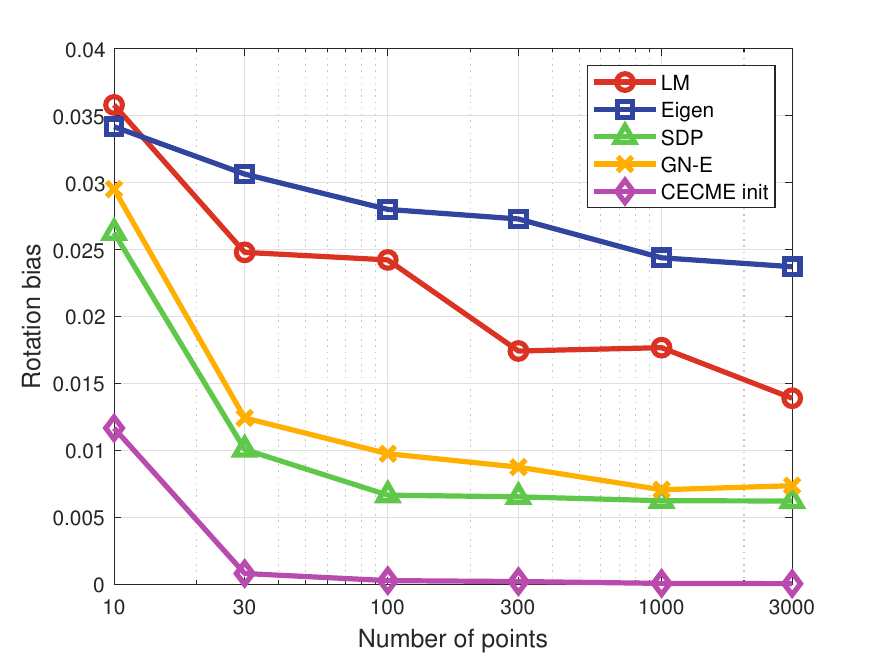}
		\caption{$\sigma=1$px ($\bf R$)}
		\label{bias_R_1px}
	\end{subfigure}
	\begin{subfigure}[b]{0.24\textwidth}
		\centering
		\includegraphics[width=\textwidth]{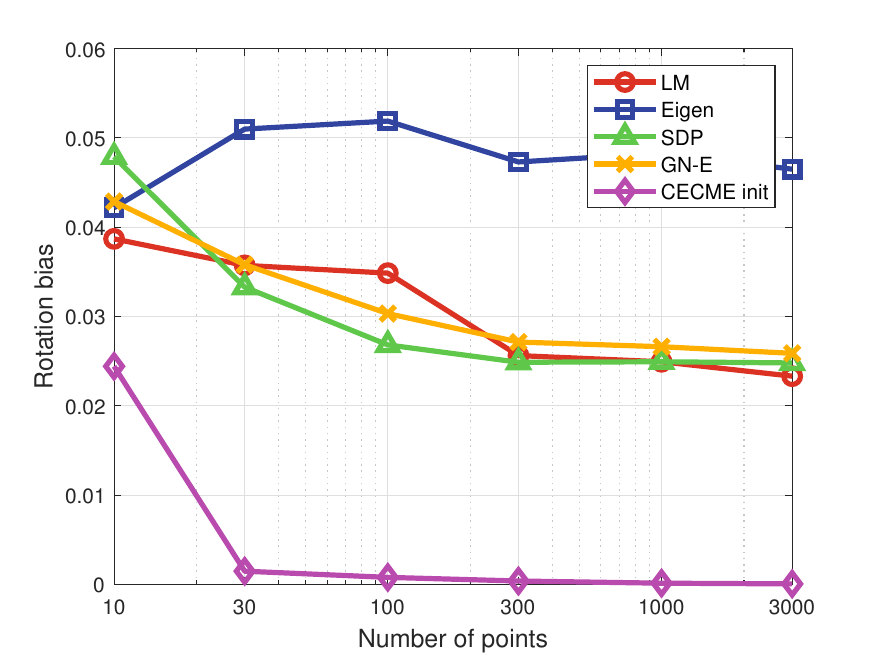}
		\caption{$\sigma=2$px ($\bf R$)}
		\label{bias_R_2px}
	\end{subfigure}
	\begin{subfigure}[b]{0.24\textwidth}
		\centering
		\includegraphics[width=\textwidth]{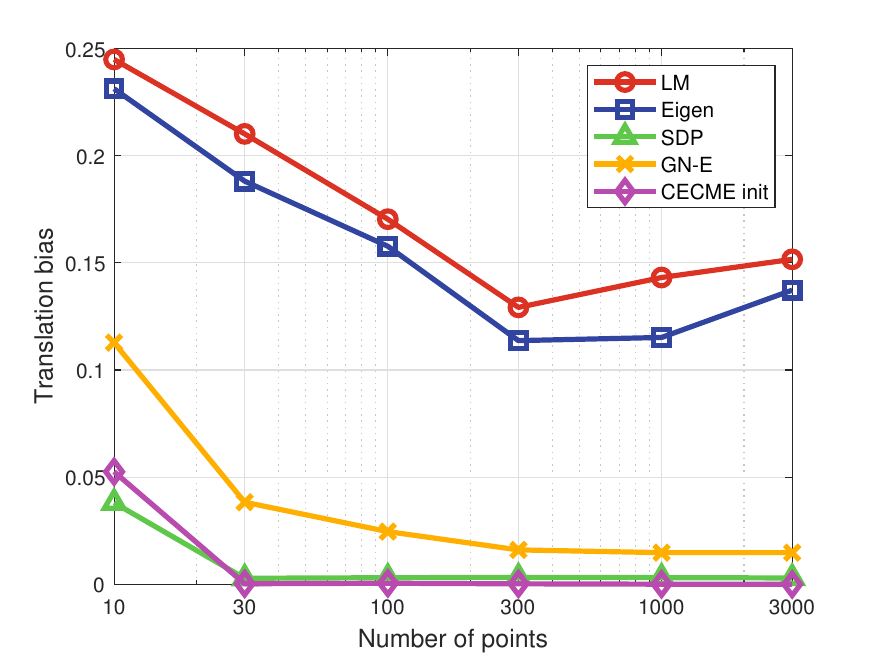}
		\caption{$\sigma=0.25$px ($\bar {\bf t}$)}
		\label{bias_t_025px}
	\end{subfigure}
	\begin{subfigure}[b]{0.24\textwidth}
		\centering
		\includegraphics[width=\textwidth]{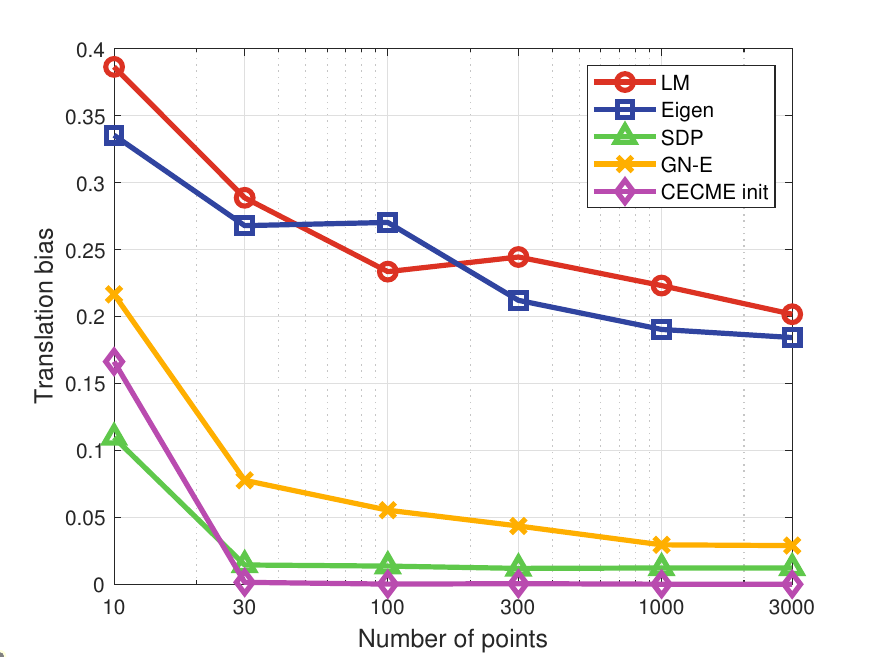}
		\caption{$\sigma=0.5$px ($\bar {\bf t}$)}
		\label{bias_t_05px}
	\end{subfigure}
	\begin{subfigure}[b]{0.24\textwidth}
		\centering
		\includegraphics[width=\textwidth]{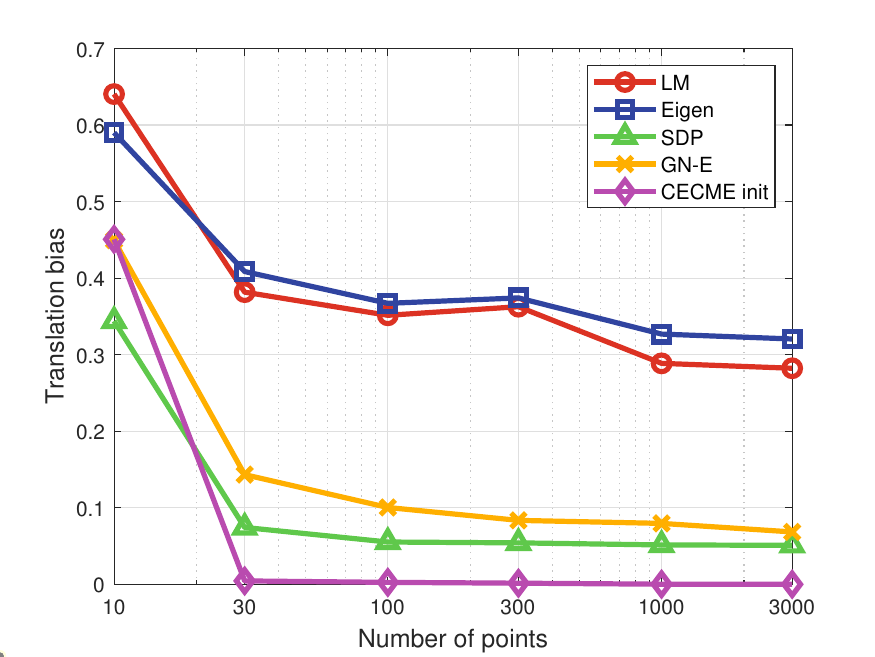}
		\caption{$\sigma=1$px ($\bar {\bf t}$)}
		\label{bias_t_1px}
	\end{subfigure}
	\begin{subfigure}[b]{0.24\textwidth}
		\centering
		\includegraphics[width=\textwidth]{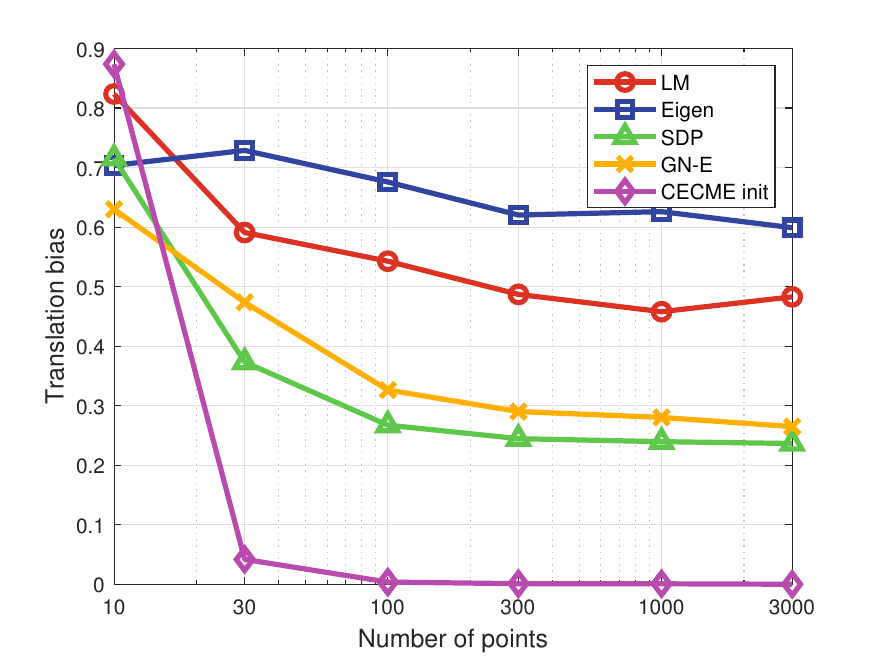}
		\caption{$\sigma=2$px ($\bar {\bf t}$)}
		\label{bias_t_2px}
	\end{subfigure}
	\caption{Bias comparison under different noise intensities and point numbers.}
	\label{bias_comparison}
\end{figure*}

\textbf{Consistency and asymptotic statistical efficiency test.} In our simulation, we run a total of $K=1000$ Monte Carlo tests to evaluate MSEs and biases. In order to verify our theoretical claim that the proposed \texttt{CECME} estimator is consistent and asymptotically statistically-efficient, we set $m=10,30,100,300,1000,3000$ and evaluate the MSEs under varied noise intensities, with $\sigma$ varying from 0.25px (0.0312\% of the image diagonal) to 2px (0.25\% of the image diagonal). The result is plotted in Figure~\ref{MSE_comparison}, where we use \texttt{CECME init} and \texttt{CECME} to denote our first-step estimator and second-step estimator, respectively. We see that the MSE of the \texttt{CECME init} estimator declines linearly w.r.t. the number of points in the log-log plot, which implies it is $\sqrt{m}$-consistent. In addition, with a one-step of GN iteration, the \texttt{CECME} estimator asymptotically reaches the CRB. Actually, when the point number exceeds one hundred, our estimator owns the statistical efficiency. It is noteworthy to see that our estimator outperforms the SOTA ones, especially when the point number and noise intensity are relatively large. 
A counter-intuitive phenomenon is that although \texttt{Eigen}, \texttt{GN-E}, and \texttt{LM} solvers use all $m$ inputs in the pose inference, their MSEs do not vary obviously w.r.t. $m$. This is because their prior pose information is provided by the five-point RANSAC solver, which has a nearly constant estimation accuracy. It shows that \texttt{Eigen}, \texttt{GN-E}, and \texttt{LM} solvers highly depend on the quality of the initial estimate. 
For the \texttt{SDP} solver, it performs well in the case of small noise intensity. As the noise intensity increases, its performance deteriorates rapidly. This coincides with the theoretical development in~\cite{zhao2020efficient}, which says that only when the noise is small enough, the SDP relaxation is tight, and the \texttt{SDP} solver gives a global solution to problem~\eqref{essential_estimation}.

We remark that the consistency of the \texttt{CECME init} estimator is owing to the proposed bias elimination~\eqref{bias_elimination} which leads to asymptotic unbiasedness. Asymptotic unbiasedness together with vanishing covariance finally yields consistency. The asymptotic unbiasedness of the \texttt{CECME init} estimator is validated in Figure~\ref{bias_comparison}, where we see that its bias converges to $0$ as the point number increases. However, the bias of the other estimators cannot converge to $0$, i.e., they are asymptotically biased. Actually, in the asymptotic case, their MSE is dominated by the asymptotic bias, and thus cannot converge to $0$, as shown in Figure~\ref{MSE_comparison}.

Recall that we choose the eigenvector corresponding to the smallest eigenvalue of the matrix ${\bf Q}_m^{\rm BE}$ in~\eqref{bias_elimination_result} as the initial guess. To justify this selection, we test how performance changes by choosing the initial guess as eigenvectors corresponding to larger eigenvalues of ${\bf Q}_m^{\rm BE}$. The result is shown in Appendix F, which demonstrates that selecting other eigenvectors results in significantly larger MSEs. 


\begin{figure*}[!t]
	\centering
\begin{minipage}{.48\linewidth}
	\centering
	\begin{subfigure}[b]{0.48\textwidth}
	\centering
	\includegraphics[width=\textwidth]{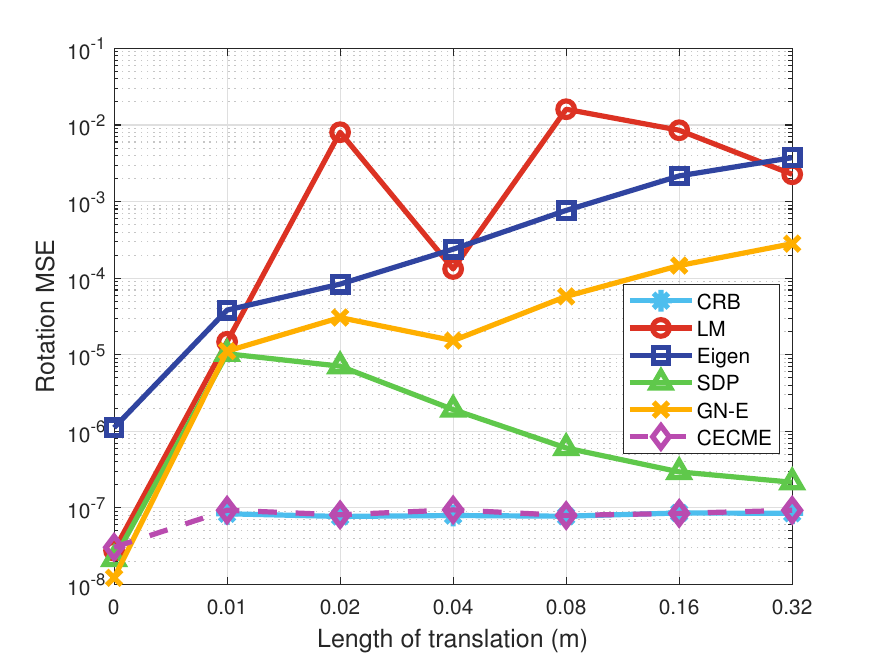}
	\caption{MSE ($\bf R$)}
	\label{MSE_R_vs_t}
\end{subfigure}
\begin{subfigure}[b]{0.48\textwidth}
	\centering
	\includegraphics[width=\textwidth]{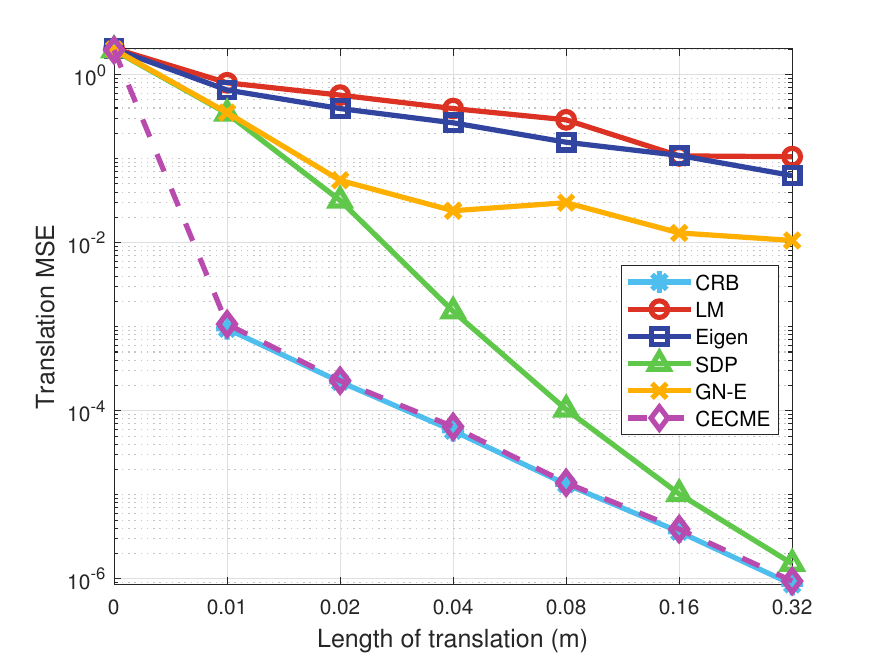}
	\caption{MSE ($\bar{\bf t}$)}
	\label{MSE_t_vs_t}
\end{subfigure}
\caption{MSE under different lengths of translation.}
\label{MSE_vs_t}
	
\end{minipage}
\begin{minipage}{.48\linewidth}
	\centering
	\begin{subfigure}[b]{0.48\textwidth}
	\centering
	\includegraphics[width=\textwidth]{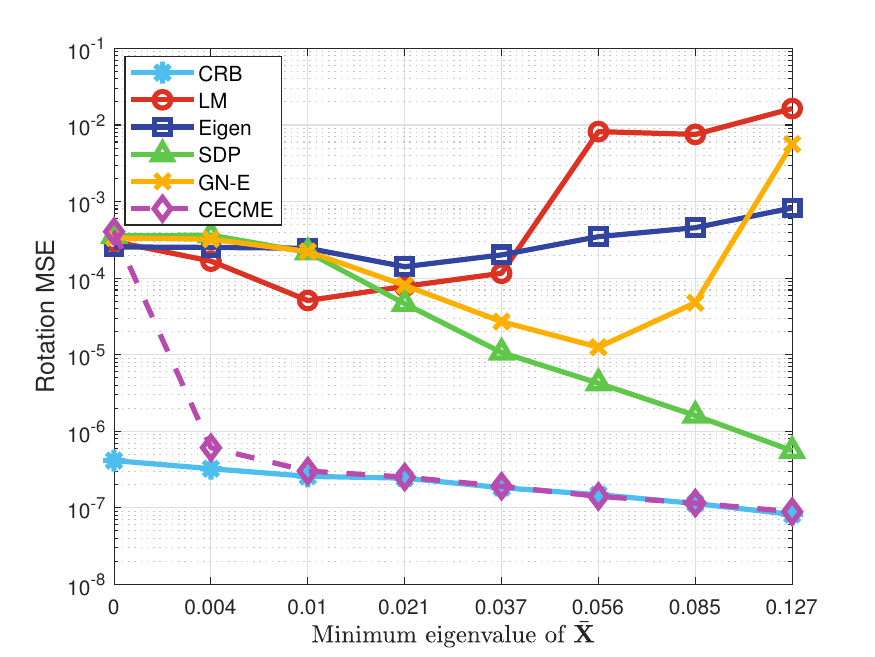}
	\caption{MSE ($\bf R$)}
	\label{MSE_R_vs_mineig}
\end{subfigure}
\begin{subfigure}[b]{0.48\textwidth}
	\centering
	\includegraphics[width=\textwidth]{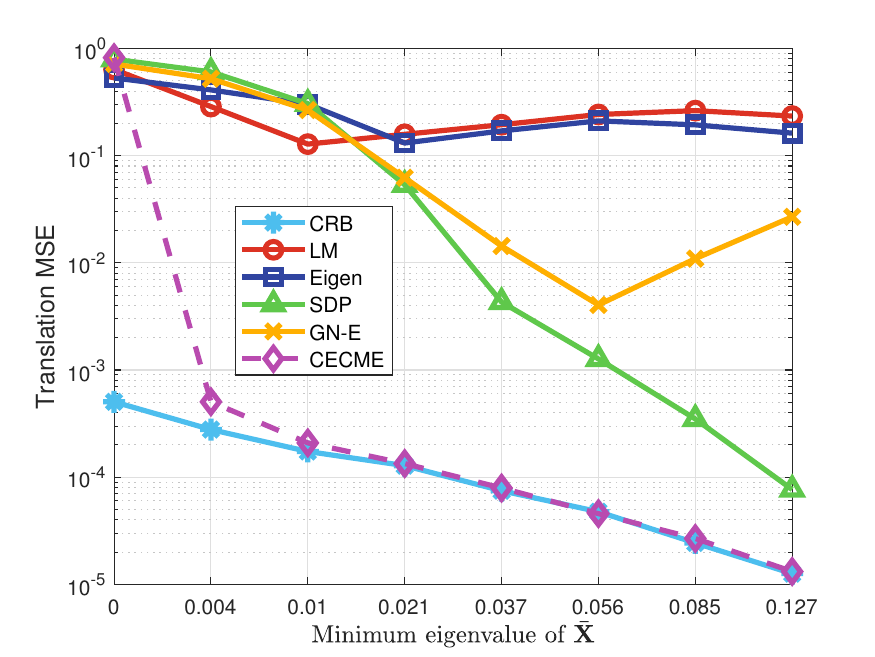}
	\caption{MSE ($\bar{\bf t}$)}
	\label{MSE_t_vs_mineig}
\end{subfigure}
\caption{MSE under different point distributions.}
\label{MSE_vs_mineig}
	
\end{minipage}
\end{figure*}

\textbf{Influence of the length of translation.} In Assumption~\ref{nonzero_t}, we assume the true translation is not equal to $0$, otherwise, the normalized translation cannot be identified, and the homography matrix should be estimated instead of the essential matrix. Nevertheless, it has been empirically shown that the length of translation $\|{\bf t}^o\|$ generally yields no impact on the estimation accuracy of the rotation matrix, even in the pure rotation cases, the rotation can be recovered from the essential matrix correctly~\cite{zhao2020efficient}. In addition, there exists a statistic that can identify the pure rotation cases~\cite{cai2019equivalent}. Specifically, the statistic is given as the average of $\{\frac{{\bf z}_i^h \times \hat{\bf R} {\bf y}_i^h}{\|{\bf z}_i^h\|\|{\bf y}_i^h\|}\}_{i=1}^{m}$, where $\hat{\bf R}$ is the estimate of the rotation matrix. This statistic can be utilized to appraise the estimation quality of the normalized translation. 

In this experiment, we fix $m=1000$, $\sigma=0.5$px (0.0625\% of the image diagonal), and change the length of the translation. The result is shown in Figure~\ref{MSE_vs_t}. Note that when $\|{\bf t}^o\|=0$, the Fisher information matrix is singular, and the CRB is not available. We see from the figure that the CRB of $\bf R$ does not change w.r.t. $\|{\bf t}^o\|$, while the CRB of $\bar {\bf t}$ increases as $\|{\bf t}^o\|$ decreases. In addition, our estimator \texttt{CECME} coincides with the CRB for both $\bf R$ and $\bar {\bf t}$. However, the other estimators have an obvious gap from the CRB, and their stability is not as good as ours. 
We also calculate the pure rotation statistics, and they are $6.9{\rm e}^{-4},2.6{\rm e}^{-3},5.6{\rm e}^{-3},1.1{\rm e}^{-2},2.2{\rm e}^{-2},4.5{\rm e}^{-2},8.5{\rm e}^{-2}$ respectively, which have the same varying trend as $\|{\bf t}^o\|$. Hence, it can serve as an indicator that depicts how confident we are with the estimation for $\bar {\bf t}$. The smaller the statistic is, the less accurate the estimate is.

\begin{table*}[!htb]
    \centering
    \captionsetup{justification=centering}
    \caption{Average estimation errors in all $25$ ETH3D scenarios. The units for rotation and translation are $10^{-3}$ rad and $10^{-5}$, respectively.} \label{average_error_comparison}
    \resizebox{1.92\columnwidth}{!}{%
        \begin{tabular}{ c c c c c c c c c c c c c c c } 
            \Xhline{4\arrayrulewidth}
            \multirow{2}{*}{Scenario} & \multicolumn{2}{c}{DFE} & \multicolumn{2}{c}{NACNet} & \multicolumn{2}{c}{LM} & \multicolumn{2}{c}{Eigen} & \multicolumn{2}{c}{SDP} & \multicolumn{2}{c}{GN-E} & \multicolumn{2}{c}{CECME}  \\
            \cline{2-15}
            & $\bf R$ & $ \bar {\bf t}$ & $\bf R$ & $ \bar {\bf t}$ & $\bf R$ & $\bar {\bf t}$ & $\bf R$ & $\bar {\bf t}$ & $\bf R$ & $\bar {\bf t}$ & $\bf R$ & $\bar {\bf t}$ & $\bf R$ & $\bar {\bf t}$ \\ 
            \Xhline{4\arrayrulewidth}
            relief & 3.39 & 5.74 & 2.65 & 4.48 & 1.23 & 1.03 & 0.929 & 0.755 & 1.23 & 8.21 & 0.893 & 0.693 & {\color{red}\bf 0.787} & {\color{red}\bf 0.485} \\ 
            facade & 3.36 & 44.1 & 3.41 & 90.4 & 1.44 & 15.4 & 1.21 & {\color{red}\bf 13.8} & 2.44 & 421 & 1.28 & 19.4 & {\color{red}\bf 1.14} & 38.3 \\ 
            courtyard & 39.9 & 858 & 37.4 & 563 & 24.5 & 412 & 27.4 & 421 & 34.8 & 1980 & 28.2 & {\color{red}\bf 392} & {\color{red}\bf 22.7} & 411 \\ 
            relief 2 & 2.80 & 10.9 & 2.30 & 4.93 & 1.28 & 2.05 & 0.850 & {\color{red}\bf 0.808} & 1.08 & 1.94 & 0.901 & 0.887 & {\color{red}\bf 0.828} & 0.842 \\
            delivery area & 5.60 & 173 & 5.37 & 75.4 & 3.49 & 596 & 2.94 & 565 & 4.32 & 405 & 2.33 & 49.9 & {\color{red}\bf 2.28} & {\color{red}\bf 41.0} \\ 
            electro & 3.54 & 23.3 & 3.26 & 20.2 & 1.30 & 3.06 & 1.19 & 3.49 & 1.64 & 4.30 & 1.13 & {\color{red}\bf 1.64} & {\color{red}\bf 1.04} & 2.63 \\ 
            terrace & 2.68 & 4.14 & 2.71 & 2.78 & 1.15 & 0.708 & 1.04 & 0.342 & 1.20 & 0.428 & 1.05 & 0.341 & {\color{red}\bf 0.979} & {\color{red}\bf 0.273} \\ 
            kicker & 7.63 & 86.0 & 6.93 & 39.3 & 2.76 & 12.6 & 2.52 & 22.3 & 3.37 & 409 & 2.22 & 24.6 & {\color{red}\bf 2.15} & {\color{red}\bf 8.76} \\
            terrains & 4.56 & 9.45 & 3.95 & 21.0 & 2.28 & 7.78 & 1.60 & {\color{red}\bf 1.49} & 2.54 & 45.0 & 1.49 & 3.20 & {\color{red}\bf 1.40} & 2.24 \\ 
            playground & 2.17 & 5.17 & 2.19 & 3.71 & 0.874 & 0.560 & 0.757 & 0.744 & 0.968 & 0.962 & 0.746 & 0.723 & {\color{red}\bf 0.680} & {\color{red}\bf 0.453} \\ 
            pipes & 3.87 & 24.4 & 3.78 & 18.5 & 2.81 & 5.31 & 2.16 & 7.29 & 2.12 & 8.40 & 2.01 & 6.76 & {\color{red}\bf 1.58} & {\color{red}\bf 1.03} \\ 
            meadow & 63.7 & 8990 & 31.6 & 4330 & 33.4 & 3730 & 34.3 & 3700 & 74.3 & 9200 & 34.3 & 3710 & {\color{red}\bf 4.51} & {\color{red}\bf 4.96} \\ 
            office & 14.6 & 4100 & 3.09 & 15.7 & 8.75 & 2620 & 8.28 & 3010 & 23.4 & 7820 & 8.27 & 2990 & {\color{red}\bf 1.37} & {\color{red}\bf 1.44} \\ 
            door & 1.86 & 4.46 & 1.22 & 0.992 & 0.649 & 0.567 & 0.435 & 0.336 & 0.587 & 0.773 & 0.377 & {\color{red}\bf 0.231} & {\color{red}\bf 0.374} & 0.232 \\ 
            observatory & 3.03 & 90.5 & 2.70 & 21.6 & {\color{red}\bf 1.32} & 4.05 & 1.35 & 5.02 & 1.58 & 14.0 & 1.37 & 4.78 & {\color{red}\bf 1.32} & {\color{red}\bf 5.90} \\ 
            boulders & 4.42 & 22.0 & 4.11 & 18.1 & 1.28 & 0.987 & 0.932 & 0.562 & 1.30 & 0.583 & {\color{red}\bf 0.905} & {\color{red}\bf 0.513} & 0.992 & 14.2 \\  
            statue & 1.87 & 0.148 & 1.32 & 0.0854 & 0.891 & 0.038 & 0.421 & 0.0088 & 0.464 & 0.012 & 0.419 & 0.0093 & {\color{red}\bf 0.407} & {\color{red}\bf 0.0087} \\ 
            bridge & 2.17 & 3.73 & 2.11 & 2.24 & 0.948 & 0.444 & 6.76 & 310 & 1.23 & 2.08 & 1.01 & {\color{red}\bf 0.410} & {\color{red}\bf 0.875} & 0.877 \\ 
            terrace 2 & 1.42 & 9.12 & 1.05 & 2.70 & 0.554 & 0.841 & 0.492 & 0.518 & 0.509 & 0.998 & 0.499 & 0.482 & {\color{red}\bf 0.428} & {\color{red}\bf 0.246} \\ 
            exhibition hall & 16.4 & 2490 & 12.1 & 1090 & 4.63 & 754 & 5.76 & 1130 & 11.1 & 4960 & 7.77 & 754 & {\color{red}\bf 3.94} & {\color{red}\bf 513} \\  
            botanical garden & 2.49 & 7.85 & 3.23 & 29.9 & 1.47 & 8.47 & {\color{red}\bf 0.911} & {\color{red}\bf 1.69} & 6.37 & 551 & 1.14 & 2.47 & 0.927 & 3.70 \\ 
            living room & 4.77 & 18.9 & 3.88 & 10.6 & 1.84 & 4.41 & 1.76 & 5.54 & 1.86 & 8.92 & 1.65 & 5.41 & {\color{red}\bf 1.44} & {\color{red}\bf 3.19} \\ 
            lecture room & 9.38 & 1140 & 6.67 & 268 & 1.11 & {\color{red}\bf 1.72} & 1.08 & 38.3 & 1.29 & 60.2 & 1.51 & 39.4 & {\color{red}\bf 1.01} & 8.10 \\ 
            lounge & 9.36 & 228 & 7.03 & 105 & 2.56 & 10.1 & 2.46 & 10.5 & 2.54 & 12.0 & 2.46 & 10.4 & {\color{red}\bf 1.58} & {\color{red}\bf 2.63} \\ 
            old computer & 4.09 & 43.7 & 3.24 & 20.4 & 1.90 & 120 & 1.53 & 28.3 & 1.69 & 23.2 & 1.50 & 31.8 & {\color{red}\bf 1.33} & {\color{red}\bf 7.06} \\ 
            \Xhline{4\arrayrulewidth}
        \end{tabular}
    }
\end{table*}

\begin{table*}[!t]
	\centering
	\captionsetup{justification=centering}
	\caption{Average CPU time (unit: $\mu s$) comparison in all $25$ scenarios} \label{average_cpu_time}
 \begin{subtable}{0.48\textwidth}
     \resizebox{0.98\columnwidth}{!}{%
		\begin{tabular}{ c c c c c c} 
			\Xhline{4\arrayrulewidth}
			Scenario & LM & Eigen & SDP & GN-E & CECME\\ 
			\Xhline{4\arrayrulewidth}
			relief & 35565 & 14484 & 8054 & 25077 & {\color{red} \bf 7611}\\ 
			door & 31011 & 11557 & 8236 & 21893 & {\color{red} \bf 5979}\\ 
			observatory & 32589 & 15824 & 6868 & 24145 & {\color{red} \bf 5469}\\ 
			facade & 27313 & 12888 & {\color{red} \bf 6333} & 20096 & 7226\\ 
			boulders & 27374 & 13033 & 6168 & 19752 & {\color{red} \bf 4406}\\ 
			courtyard & 23118 & 8008 & {\color{red} \bf 6088} & 14694 & 22096\\ 
       relief 2 & 23460 & 11421 & 5875 & 17558 & {\color{red} \bf 3917} \\
       statue & 26389 & 13443 & 5457 & 19124 & {\color{red} \bf 3380} \\
       bridge & 29064 & 20029 & 4862 & 24583 & {\color{red} \bf 2918} \\
       terrace 2 & 20312 & 9868 & 5402 & 15626 & {\color{red} \bf 4113} \\
       delivery area & 23009 & 14133 & {\color{red} \bf 4844} & 18792 & 5627 \\
       exhibition hall & 18045 & 8480 & {\color{red} \bf 4992} & 13037 & 13815 \\
			\Xhline{4\arrayrulewidth}
		\end{tabular}
	}
 \end{subtable}
  \begin{subtable}{0.48\textwidth}
     \resizebox{0.99\columnwidth}{!}{%
		\begin{tabular}{ c c c c c c} 
			\Xhline{4\arrayrulewidth}
		electro & 29772 & 21099 & 4802 & 25199 & {\color{red} \bf 4410} \\
        terrace & 26776 & 19218 & 4230 & 22991 & {\color{red} \bf 2266} \\
       kicker & 23979 & 16157 & {\color{red} \bf 4543} & 19887 & 10244 \\
       botanical garden & 28116 & 21408 & 4115 & 24692 & {\color{red} \bf 3053} \\
       terrains & 19833 & 14032 & {\color{red} \bf 4268} & 17062 & 7810 \\
       living room & 15338 & 9781 & 3860 & 12534 & {\color{red} \bf 2569}\\
       playground & 22657 & 16386 & 3878 & 19092 & {\color{red} \bf 1741} \\
       pipes & 16882 & 11850 & {\color{red} \bf 3864} & 14054 & 3960 \\
       lecture room & 14154 & 9145 & 3776 & 11490 & {\color{red} \bf 2068} \\
       lounge & 26101 & 21130 & 3481 & 23200 & {\color{red} \bf 1266} \\
       meadow & 14655 & 11354 & {\color{red} \bf 4504} & 12997 & 12585 \\
       office & 17403 & 13805 & {\color{red} \bf 3518} & 15399 & 7256 \\
       old computer & 17535 & 14637 & 3415 & 16083 & {\color{red} \bf 3176} \\
			\Xhline{4\arrayrulewidth}
		\end{tabular}
	}
 \end{subtable}
\end{table*}

 \begin{figure}[!t]
    \centering
   \begin{subfigure}{0.49\linewidth}
       \includegraphics[width=\linewidth]{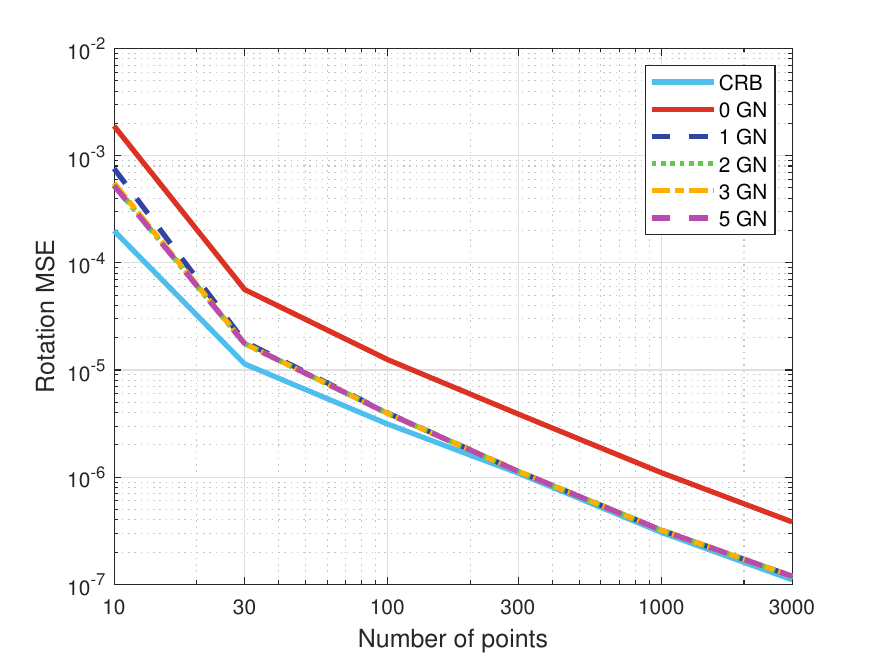}
   \end{subfigure}
    \begin{subfigure}{0.49\linewidth}
       \includegraphics[width=\linewidth]{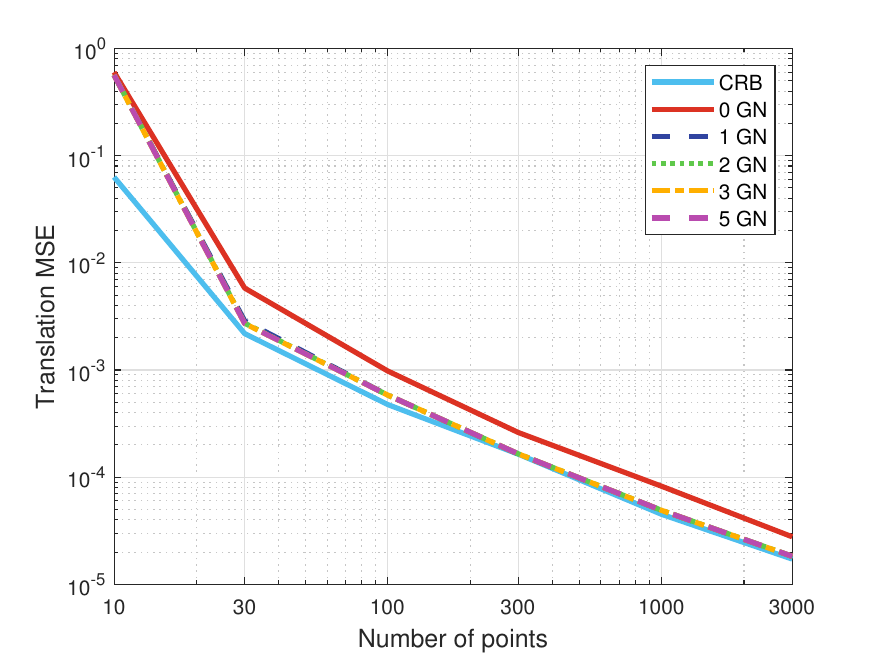}
   \end{subfigure}
   \caption{MSE comparison with varied GN numbers.}
            \label{MSE_varied_GN}
\end{figure}

\textbf{Influence of the distribution of 3D points.} In Assumption~\ref{not_coplanar_assump}, we assume the 3D points $\{{\bf x}_i\}_{i=1}^{m}$ and the two camera centers do not lie on a ruled quadric surface, which we call the degenerate configuration. The most likely degenerate configuration is that $\{{\bf x}_i\}_{i=1}^{m}$ concentrate on a man-made plane, e.g., a wall. Similar to the case of $\|{\bf t}^o\|=0$, in the coplanar case, the homography matrix should be estimated instead of the essential matrix. Let $\bar {\bf X}=[{\bf x}_1^h ~\cdots~{\bf x}_m^h][{\bf x}_1^h ~\cdots~{\bf x}_m^h]^\top/m$, where ${\bf x}_i^h$ is the homogeneous coordinates of ${\bf x}_i$. Then, if $\{{\bf x}_i\}_{i=1}^{m}$ are coplanar, it can be verified that $\lambda_{\rm min}(\bar {\bf X})=0$. Therefore, $\lambda_{\rm min}(\bar {\bf X})$ is a quantity that can identify the coplanar case. In Figure~\ref{MSE_vs_mineig}, we plot the relationship between MSE and $\lambda_{\rm min}(\bar {\bf X})$. We see that as $\lambda_{\rm min}(\bar {\bf X})$ decreases, i.e., the 3D points shrink in some dimension, the CRB and MSE increase. However, different from the case of $\|{\bf t}^o\|=0$, when $\lambda_{\rm min}(\bar {\bf X})=0$, the Fisher information matrix is nonsingular and the CRB is available, which implies the relative pose is locally identifiable. Nevertheless, in the (near) coplanar case, the MSE of all estimators deviates from the CRB, showing that the relative pose is not globally identifiable. Note that $\{{\bf x}_i\}_{i=1}^{m}$ are unavailable in practice, so we cannot obtain $\lambda_{\rm min}(\bar {\bf X})$. In order to identify the coplanar case in real applications, one can estimate a homography matrix and treat the average residual as the coplanar statistic~\cite{campos2021orb}. The smaller the statistic is, the more possible the coplanar case is.

\textbf{Effect of GN iteration number.} We claim that given a $\sqrt{m}$-consistent initial estimator, a single GN iteration is enough to reach the CRB in the large-sample regime. To substantiate this claim, we compare the MSEs for varied GN iteration numbers. We set the standard deviation of noise as $\sigma=1$px (0.125\% of the image diagonal) and execute 0-5 steps of GN iterations. As demonstrated in Figure~\ref{MSE_varied_GN}, the initial estimator's MSE (as seen from the 0 GN MSE curve) can approach, but not reach, the CRB. A single GN iteration, however, can achieve the CRB, with further iterations providing negligible improvement. The test result for the effect of GN iteration number in real datasets can be found in Appendix H.

  \begin{figure}[!t]
    \centering
   \begin{subfigure}{0.49\linewidth}
       \includegraphics[width=\linewidth]{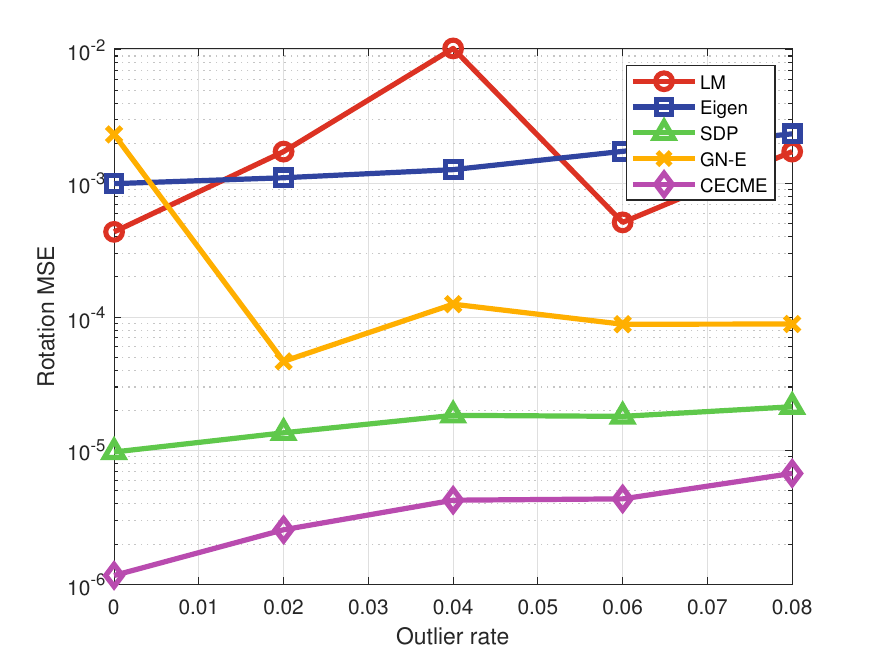}
   \end{subfigure}
    \begin{subfigure}{0.49\linewidth}
       \includegraphics[width=\linewidth]{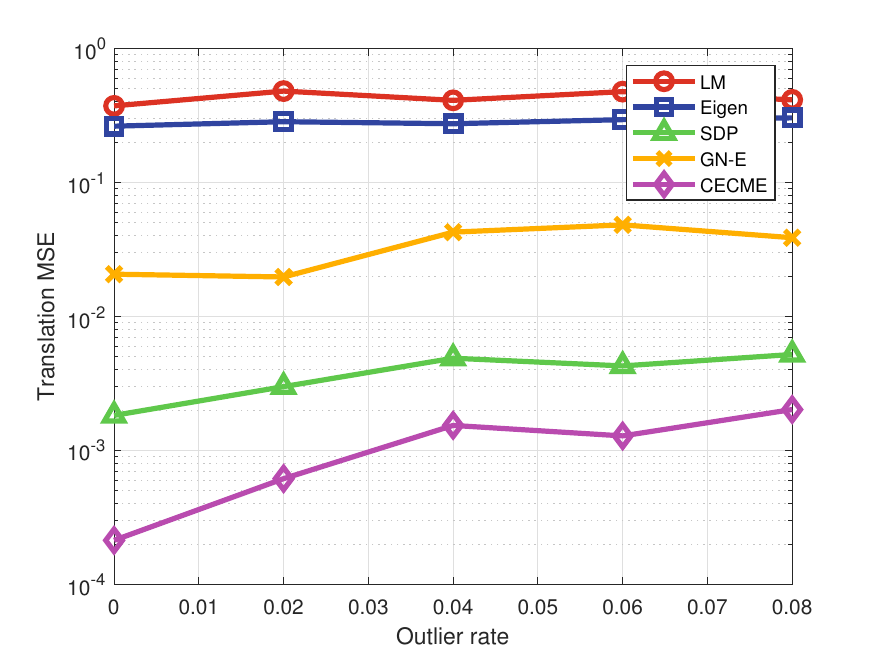}
   \end{subfigure}
   \caption{MSE comparison under varied outlier rates.}
            \label{MSE_outlier}
\end{figure}

\textbf{Robustness to outliers.} We set the point number as $m=1000$ and the standard deviation of noise as $\sigma=1$px (0.125\% of the image diagonal). The generation of outliers follows a uniform distribution in the image, and the outlier rate is set to be $0, 0.02, 0.04, 0.06, 0.08$, respectively.  
The results are presented in Figure~\ref{MSE_outlier}. As the outlier rate increases, the MSE of our estimator \texttt{CECME} rises. Nevertheless, it consistently outperforms the other algorithms, demonstrating a notable level of robustness.

\subsection{Experiment with real images}
\label{subsection_real_images}
For the experiment with real images, we use the ETH3D~\cite{schops2017multi}, KITTI odometry~\cite{geiger2012we}, and the EuRoC MAV datasets~\cite{burri2016euroc}. 
All algorithms are implemented in C++ via a PC equipped with an Intel Core i5-10400H and a 32Gb RAM. The evaluation details and codes are open source at \url{https://github.com/LIAS-CUHKSZ/epipolar_eval}. The estimation errors of the rotation matrix and normalized translation vector are given by Lie algebra angular error and cosine distance, respectively, as in~\cite{zhao2020efficient}. 

In Assumption~\ref{not_coplanar_assump}, we assume the 3D points $\{{\bf x}_i\}_{i=1}^{m}$ and the two camera centers do not lie on a ruled quadric surface, which we call the degenerate configuration. The most likely degenerate configuration is that 3D points concentrate on a plane. In real-world data tests, near-degenerate point distributions can occasionally occur. Hence, we introduce an adaptive mechanism to handle such cases. Specifically, recall that in the first step, we select the eigenvector of ${\bf Q}_m^{\rm BE}$ associated with the smallest eigenvalue as the initial estimate. Denote the corresponding epipolar cost as $c_1$. In addition, we also compute the eigenvector associated with the second-smallest eigenvalue and calculate its epipolar cost, denoted as $c_2$. If the ratio $c_2/c_1$ falls below a predefined threshold, indicating degeneracy, we use a RANSAC algorithm to obtain a robust initial estimate.

\subsubsection{ETH3D dataset} This dataset contains $25$ scenarios ranging from indoors to outdoors. The intrinsic and extrinsic parameters of the camera are given, so we can calculate the normalized image coordinates of each feature point and the true relative pose between an image pair. In addition, the 3D global map is available, and the 2D-3D point correspondences are provided, based on which the 2D-2D point correspondences can be obtained. For a fair comparison with \texttt{NACNet}, we bypass its feature detection and matching module and instead directly input the provided correspondences. To validate that the proposed estimator has advantages in the asymptotic case, we only estimate the relative pose of image pairs (the two images are not necessarily consecutive) that have more than $100$ point correspondences.

The average estimation errors in all $25$ scenarios are listed in Table~\ref{average_error_comparison}. We see that our proposed estimator \texttt{CECME} performs best in $23$ scenarios for $\bf R$ and $15$ scenarios for $\bar {\bf t}$. 
It is noteworthy to see that \texttt{LM}, \texttt{Eigen}, and \texttt{GN-E} outperform \texttt{SDP} and sometimes become the best ones, which does not coincide with the simulation result. This is because, in real image tests, the five-point RANSAC algorithms can provide a robust initial value for these three solvers. We observe that learning-based methods are inferior to classical geometry-based approaches and fail to achieve the best accuracy on any sequence. This may be attributed to the presence of large-parallax image pairs. In such cases, feature descriptors may exhibit notable variations, even for correct correspondences, leading to reduced estimation accuracy for the learning-based methods.
The average CPU time cost is listed in Table~\ref{average_cpu_time}. We see that our \texttt{CECME} estimator consumes the least time in 16 out of 25 sequences, demonstrating its high computational efficiency. The reason why it consumes more time in the remaining 9 sequences is that in these sequences, degenerate cases occurs more frequently, and the RANSAC initialization costs much more time than calculating the eigenvector of ${\bf Q}_m^{\rm BE}$. 

\begin{figure*}[!t]
	\centering
	\begin{subfigure}[b]{0.48\textwidth}
		\centering
		\includegraphics[width=\textwidth]{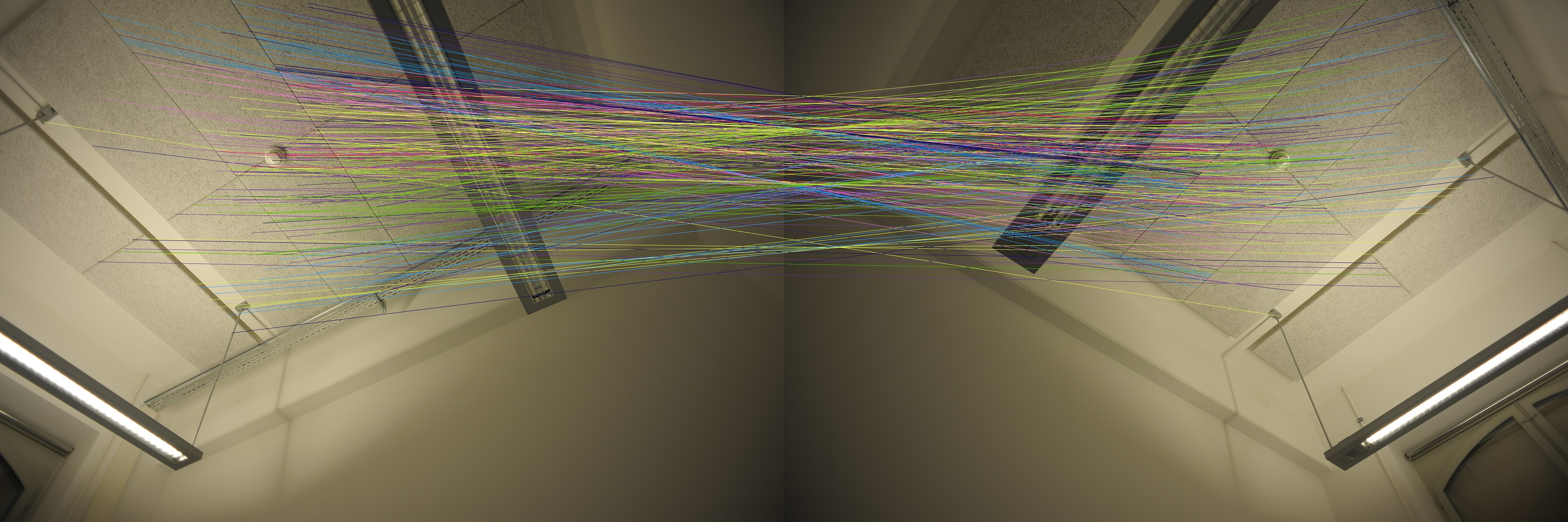}
		\caption{office 12-16}
		\label{office_12_16}
	\end{subfigure}
	\begin{subfigure}[b]{0.48\textwidth}
		\centering
		\includegraphics[width=\textwidth]{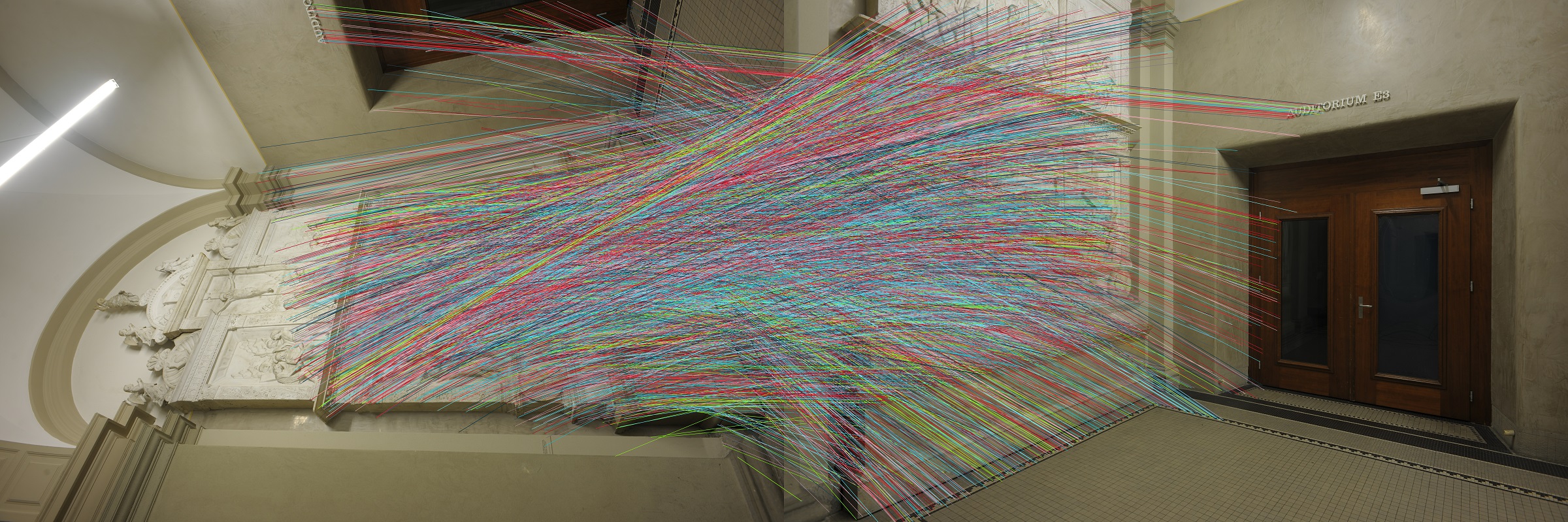}
		\caption{relief 11-18}
		\label{relief_11_18}
	\end{subfigure}
	\begin{subfigure}[b]{0.48\textwidth}
		\centering
		\includegraphics[width=\textwidth]{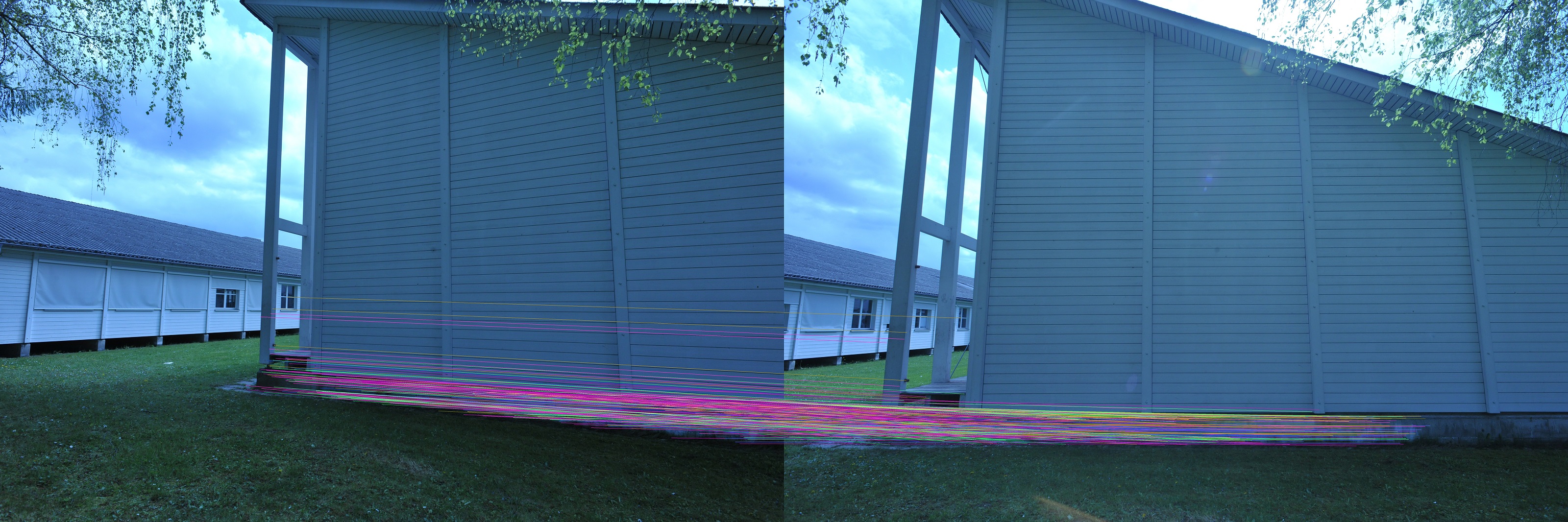}
		\caption{meadow 4-13}
		\label{meadow_4_13}
	\end{subfigure}
	\begin{subfigure}[b]{0.48\textwidth}
		\centering
		\includegraphics[width=\textwidth]{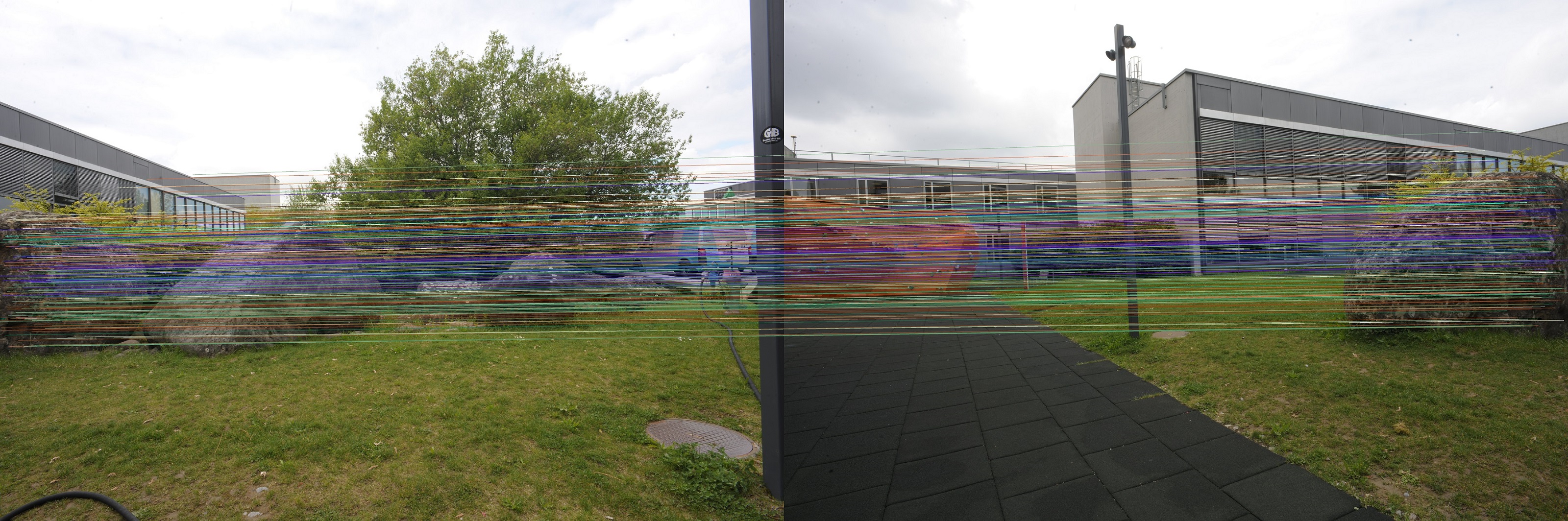}
		\caption{boulders 8-13}
		\label{boulders_8_13}
	\end{subfigure}
	\caption{Some examples of degenerate scenarios.}
	\label{four_degenerate_scenarios}
\end{figure*}

Finally, we present in Figure~\ref{four_degenerate_scenarios} some typical degenerate scenarios of epipolar pose estimation. In Figure~\ref{office_12_16}, the points concentrate on the roof of the office, which violates the noncoplanar Assumption~\ref{not_coplanar_assump}; In Figure~\ref{relief_11_18}, the translation between the image pair is too small, which can be viewed as the pure rotation case and infringes Assumption~\ref{nonzero_t}; In Figure~\ref{meadow_4_13}, the 3D points are confined within a very limited range in two dimensions and has a small $\lambda_{\rm min}(\bar {\bf X})$; In Figure~\ref{boulders_8_13}, most of the 3D points locate on leaves, whose position may change due to the disturbance of wind and in turn weaken the estimation accuracy. The above degenerate scenarios pose a challenge to the monocular-vision-based relative pose estimation. 

\subsubsection{KITTI odometry dataset}
The KITTI odometry dataset is captured in diverse outdoor driving scenarios, spanning urban, suburban, and rural environments. It offers rectified stereo images at a frame rate of 10 Hz. The dataset also includes ground-truth poses obtained from an OXTS3003 GPS/IMU unit for 11 training sequences (seq00-10), enabling pose estimation evaluation and benchmarking. 

We estimate the relative pose between each pair of consecutive frames. For front-end feature matching, we test three methods: SIFT, SURF, and optical flow tracking with Harris corner detection, and the results are presented in Tables~\ref{average_error_comparison_kitti_sift}-\ref{average_error_comparison_kitti_harris}. Notably, \texttt{NACNet} is independent of front-end feature matching, and \texttt{DFE} uses the SIFT descriptor in its open-source code. Hence, we only present the results of the two learning-based methods in Table~\ref{average_error_comparison_kitti_sift} which corresponds to the SIFT descriptor. Three tables show that regardless of the front-end method used, our \texttt{CECME} algorithm consistently delivers the best performance in most sequences. When using Harris corner detection and optical flow tracking, \texttt{CECME} achieves the highest accuracy across all sequences for both rotation and translation, with the exception of translation accuracy in sequence 01.

\begin{table*}[!t]
	\centering
	\captionsetup{justification=centering}
	\caption{Average estimation errors in all $11$ KITTI odometry sequences using SIFT features. The units for rotation and translation are $10^{-4}$ rad and $10^{-4}$, respectively.} \label{average_error_comparison_kitti_sift}
	\resizebox{1.9\columnwidth}{!}{%
		\begin{tabular}{ c c c c c c c c c c c c c c c } 
			\Xhline{4\arrayrulewidth}
			\multirow{2}{*}{Sequence} & \multicolumn{2}{c}{DFE} & \multicolumn{2}{c}{NACNet} & \multicolumn{2}{c}{LM} & \multicolumn{2}{c}{Eigen} & \multicolumn{2}{c}{SDP} & \multicolumn{2}{c}{GN-E} & \multicolumn{2}{c}{CECME}  \\
			\cline{2-15}
			& $\bf R$ & $ \bar {\bf t}$ & $\bf R$ & $ \bar {\bf t}$ & $\bf R$ & $\bar {\bf t}$ & $\bf R$ & $\bar {\bf t}$ & $\bf R$ & $\bar {\bf t}$ & $\bf R$ & $\bar {\bf t}$ & $\bf R$ & $\bar {\bf t}$ \\ 
			\Xhline{4\arrayrulewidth}
			seq00 & 23.0 & 34.1 & 10.0 & 6.91 & 12.96 & 9.51 & 9.08 & 5.49 & 9.10 & 3.84 & 9.33 & 5.69 & {\color{red}\bf 8.55} & {\color{red}\bf 3.65} \\ 
			seq01 & 23.6 & 289 & {\color{red}\bf 6.13} & 189 & 25.39 & 199.62 & 13.52 & {\color{red}\bf 59.11} & 12.33 & 258.02 & 14.70 & 282.43 & 21.57 & 211.75 \\ 
			seq02 & 19.5 & 10.9 & 7.92 & 3.13 & 10.31 & 4.07 & 6.78 & {\color{red}\bf 2.38} & 6.92 & 2.39 & 6.97 & 2.41 & {\color{red}\bf 6.58} & 2.41 \\ 
			seq03 & 18.4 & 34.6 & 6.95 & 3.23 & 10.77 & 8.64 & 6.49 & 2.95 & 6.75 & 3.21 & 6.65 & 2.96 & {\color{red}\bf 6.32} & {\color{red}\bf 2.27} \\ 
			seq04 & 17.5 & 8.19 & 5.40 & 1.01 & 11.15 & 5.42 & 4.52 & 0.97 & 4.60 & 1.01 & 4.79 & 1.05 & {\color{red}\bf 4.15} & {\color{red}\bf 0.78} \\ 
			seq05 & 18.3 & 26.4 & 6.87 & 1.55 & 10.73 & 5.38 & 6.66 & 1.12 & 6.78 & 1.18 & 6.91 & 1.17 & {\color{red}\bf 6.26} & {\color{red}\bf 1.11} \\ 
			seq06 & 18.3 & 32.9 & 5.98 & 1.72 & 9.09 & 8.15 & 5.55 & 4.10 & 5.38 & 1.64 & 5.63 & 1.69 & {\color{red}\bf 4.85} & {\color{red}\bf 1.16} \\
			seq07 & 21.2 & 82.4 & 7.15 & 12.71 & 10.93 & 5.12 & 6.04 & 1.49 & 6.14 & 1.61 & 6.27 & 1.56 & {\color{red}\bf 5.82} & {\color{red}\bf 1.36} \\
			seq08 & 18.3 & 28.0 & 7.24 & 14.9 & 10.82 & 10.32 & 6.83 & 7.32 & 6.85 & 6.89 & 6.93 & 6.92 & {\color{red}\bf 6.54} & {\color{red}\bf 6.82} \\
			seq09 & 19.7 & 14.1 & 7.59 & 2.14 & 11.32 & 13.09 & 7.41 & 5.97 & 7.19 & 1.92 & 7.49 & 3.88 & {\color{red}\bf 6.74} & {\color{red}\bf 1.61} \\
			seq10 & 19.0 & 7.94 & 8.62 & 2.92 & 11.03 & 3.65 & 8.39 & {\color{red}\bf 2.21} & 8.45 & 2.23 & 8.50 & 2.23 & {\color{red}\bf 8.15} & 2.23 \\
			\Xhline{4\arrayrulewidth}
		\end{tabular}
	}
\end{table*}

\begin{table*}[!htb]
    \centering
    \captionsetup{justification=centering}
    \caption{Average estimation errors in all $11$ KITTI odometry sequences using SURF features. The units for rotation and translation are $10^{-4}$ rad and $10^{-4}$, respectively.} \label{average_error_comparison_kitti_surf}
    \resizebox{1.4\columnwidth}{!}{%
        \begin{tabular}{ c c c c c c c c c c c } 
            \Xhline{4\arrayrulewidth}
            \multirow{2}{*}{Sequence} & \multicolumn{2}{c}{LM} & \multicolumn{2}{c}{Eigen} & \multicolumn{2}{c}{SDP} & \multicolumn{2}{c}{GN-E} & \multicolumn{2}{c}{CECME}  \\
            \cline{2-11}
            & $\bf R$ & $\bar {\bf t}$ & $\bf R$ & $\bar {\bf t}$ & $\bf R$ & $\bar {\bf t}$ & $\bf R$ & $\bar {\bf t}$ & $\bf R$ & $\bar {\bf t}$ \\ 
            \Xhline{4\arrayrulewidth}
            seq00 & 19.07 & 18.90 & 11.64 & 7.91 & 11.72 & 5.86 & 12.20 & 8.76 & {\color{red}\bf 10.44} & {\color{red}\bf 4.83} \\ 
            seq01 & 36.98 & 430.72 & 19.97 & {\color{red}\bf 285.65} & {\color{red}\bf 17.88} & 422.78 & 23.43 & 491.96 & 27.01 & 546.87 \\ 
            seq02 & 16.46 & 8.41 & 9.63 & {\color{red}\bf 3.36} & 10.05 & 3.61 & 10.11 & 3.55 & {\color{red}\bf 9.04} & 3.84 \\ 
            seq03 & 16.17 & 13.74 & 7.76 & 3.13 & 8.13 & 3.53 & 8.14 & 3.33 & {\color{red}\bf 7.50} & {\color{red}\bf 2.70} \\ 
            seq04 & 17.87 & 9.43 & 6.68 & 1.74 & 6.85 & 1.87 & 7.25 & 1.88 & {\color{red}\bf 5.88} & {\color{red}\bf 1.38} \\ 
            seq05 & 15.99 & 13.92 & 8.62 & 2.03 & 8.87 & 2.24 & 9.11 & 2.19 & {\color{red}\bf 7.78} & {\color{red}\bf 2.01} \\ 
            seq06 & 15.36 & 28.77 & 8.32 & 6.32 & 8.02 & 2.68 & 8.89 & 6.14 & {\color{red}\bf 6.65} & {\color{red}\bf 2.35} \\
            seq07 & 16.48 & 15.49 & 7.97 & 2.49 & 8.28 & 2.73 & 8.53 & 3.97 & {\color{red}\bf 7.32} & {\color{red}\bf 2.13} \\
            seq08 & 16.70 & 19.23 & 9.11 & 9.45 & 9.33 & 8.85 & 9.53 & 10.20 & {\color{red}\bf 8.40} & {\color{red}\bf 7.81} \\
            seq09 & 17.52 & 23.66 & 10.85 & 8.80 & 10.66 & 5.37 & 11.39 & 8.65 & {\color{red}\bf 9.46} & {\color{red}\bf 3.80} \\
            seq10 & 17.23 & 8.80 & 10.72 & {\color{red}\bf 3.03} & 11.01 & 3.05 & 11.20 & 3.11 & {\color{red}\bf 10.11} & 3.07 \\
            \Xhline{4\arrayrulewidth}
        \end{tabular}
    }
\end{table*}

\begin{table*}[!htb]
    \centering
    \captionsetup{justification=centering}
    \caption{Average estimation errors in all $11$ KITTI odometry sequences using Harris corner detection and optical flow tracking. The units for rotation and translation are $10^{-4}$ rad and $10^{-4}$, respectively.} \label{average_error_comparison_kitti_harris}
    \resizebox{1.4\columnwidth}{!}{%
        \begin{tabular}{ c c c c c c c c c c c } 
            \Xhline{4\arrayrulewidth}
            \multirow{2}{*}{Sequence} & \multicolumn{2}{c}{LM} & \multicolumn{2}{c}{Eigen} & \multicolumn{2}{c}{SDP} & \multicolumn{2}{c}{GN-E} & \multicolumn{2}{c}{CECME}  \\
            \cline{2-11}
            & $\bf R$ & $\bar {\bf t}$ & $\bf R$ & $\bar {\bf t}$ & $\bf R$ & $\bar {\bf t}$ & $\bf R$ & $\bar {\bf t}$ & $\bf R$ & $\bar {\bf t}$ \\ 
            \Xhline{4\arrayrulewidth}
            seq00 & 14.77 & 20.82 & 11.54 & 15.20 & 12.37 & 20.36 & 12.59 & 21.28 & {\color{red}\bf 9.50} & {\color{red}\bf 4.34} \\ 
            seq01 & 47.93 & 966.34 & 34.63 & {\color{red}\bf 497.24} & 30.81 & 1050.56 & 31.67 & 1365.37 & {\color{red}\bf 27.66} & 1332.90 \\ 
            seq02 & 14.24 & 17.16 & 10.76 & 9.51 & 11.66 & 13.27 & 11.36 & 7.61 & {\color{red}\bf 8.45} & {\color{red}\bf 2.86} \\ 
            seq03 & 10.61 & 5.31 & 7.03 & 2.45 & 7.59 & 2.74 & 7.40 & 2.55 & {\color{red}\bf 6.72} & {\color{red}\bf 2.21} \\ 
            seq04 & 11.28 & 4.95 & 5.38 & 0.94 & 5.73 & 0.97 & 5.97 & 1.10 & {\color{red}\bf 4.66} & {\color{red}\bf 0.82} \\ 
            seq05 & 12.76 & 8.32 & 10.31 & 3.29 & 10.74 & 3.52 & 11.23 & 3.69 & {\color{red}\bf 8.35} & {\color{red}\bf 2.42} \\ 
            seq06 & 13.36 & 33.80 & 9.31 & 30.95 & 10.19 & 31.03 & 10.38 & 33.58 & {\color{red}\bf 5.93} & {\color{red}\bf 1.45} \\
            seq07 & 11.65 & 5.23 & 8.09 & 3.80 & 8.66 & 3.45 & 8.75 & 3.21 & {\color{red}\bf 6.79} & {\color{red}\bf 2.22} \\
            seq08 & 13.13 & 19.39 & 10.21 & 17.87 & 10.62 & 14.82 & 10.54 & 15.04 & {\color{red}\bf 7.90} & {\color{red}\bf 7.75} \\
            seq09 & 14.23 & 23.77 & 13.23 & 13.10 & 12.42 & 10.03 & 13.14 & 37.18 & {\color{red}\bf 9.35} & {\color{red}\bf 4.34} \\
            seq10 & 22.35 & 12.38 & 16.99 & 8.62 & 17.98 & 10.90 & 18.16 & 12.64 & {\color{red}\bf 14.47} & {\color{red}\bf 5.61} \\
            \Xhline{4\arrayrulewidth}
        \end{tabular}
    }
\end{table*}

\subsubsection{EuRoC MAV dataset}
Most monocular visual odometry or visual-inertial odometry algorithms adopt a 3D-2D framework for robot motion estimation after initilization, which depends on PnP pose estimation rather than epipolar relative pose estimation. The epipolar relative pose estimation will be only used in the initilization stage. Therefore, we replace the epipolar estimation module in the initialization of VINS~\cite{qin2018vins} (a SOTA and widely-used monocular visual-inertial odometry method) with our proposed \texttt{CECME} estimator and evaluate the overall odometry performance in the EuRoC MAV dataset. EuRoC MAV is a visual-inertial dataset collected on-board a micro aerial vehicle (MAV). The dataset contains stereo images, synchronized IMU measurements, and accurate motion and structure ground-truth.
The original VINS method adopts the seven-point RANSAC algorithm for epipolar pose estimation.

The test results on the EuRoC MAV dataset are summarized in Table~\ref{average_error_comparison_euroc}. Integrating our \texttt{CECME} algorithm into the initialization stage of the VINS method improves odometry performance on 8 out of 11 sequences, demonstrating the effectiveness of \texttt{CECME}. Notably, the improvements on the V102 and V103 sequences are significant, with RMSE reductions of 36.7\% and 37.4\%, respectively. 
\texttt{CECME} exhibits slightly lower accuracy on the MH04, V201, and V202 sequences.
This reduced performance can be attributed to the fact that most initial points in these sequences lie on a plane, which is a degenerate case for essential matrix estimation. In contrast, the seven-point RANSAC algorithm, which samples only seven points at a time, has a probability of selecting points that are not coplanar and therefore achieves better accuracy in these cases.

\begin{table}[!htb]
    \centering
    \captionsetup{justification=centering}
    \caption{Absolute pose errors in all $11$ EuRoC MAV dataset sequences. VINS (CECME) represents replacing the epipolar estimation module in the initialization stage of the VINS algorithm with our proposed CECME estimator.} \label{average_error_comparison_euroc}
    \resizebox{0.98\columnwidth}{!}{
        \begin{tabular}{ c c c c c c c } 
             \Xhline{4\arrayrulewidth}
            \multirow{2}{*}{Sequence} & \multicolumn{3}{c}{VINS} & \multicolumn{3}{c}{VINS (CECME)} \\
            \cline{2-7}
            & RMSE & Mean & Median & RMSE & Mean & Median \\ 
             \Xhline{4\arrayrulewidth}
            MH01 & 0.424 & 0.371 & 0.432 & {\color{red}\bf 0.408} & {\color{red}\bf 0.352} & {\color{red}\bf 0.344} \\ 
            MH02 & 0.375 & 0.324 & 0.297 & {\color{red}\bf 0.362} & {\color{red}\bf 0.295} & {\color{red}\bf 0.198} \\ 
            MH03 & 0.260 & 0.225 & {\color{red}\bf 0.171} & {\color{red}\bf 0.220} & {\color{red}\bf 0.198} & 0.179 \\ 
            MH04 & {\color{red}\bf 0.324} & {\color{red}\bf 0.272} & {\color{red}\bf 0.309} & 0.335 & 0.293 & 0.347 \\ 
            MH05 & 0.372 & 0.322 & {\color{red}\bf 0.250} & {\color{red}\bf 0.369} & {\color{red}\bf 0.318} & 0.262 \\ 
            V101 & 0.146 & 0.140 & 0.146 & {\color{red}\bf 0.136} & {\color{red}\bf 0.129} & {\color{red}\bf 0.131} \\ 
            V102 & 0.289 & 0.224 & 0.186 & {\color{red}\bf 0.183} & {\color{red}\bf 0.153} & {\color{red}\bf 0.132} \\
            V103 & 0.187 & 0.170 & 0.144 & {\color{red}\bf 0.117} & {\color{red}\bf 0.106} & {\color{red}\bf 0.0875} \\
            V201 & {\color{red}\bf 0.0822} & {\color{red}\bf 0.0765} & {\color{red}\bf 0.0796} & 0.0863 & 0.0805 & 0.0837 \\
            V202 & {\color{red}\bf 0.0932} & {\color{red}\bf 0.0804} & {\color{red}\bf 0.0719} & 0.0972 & 0.0814 & 0.0731 \\
            V203 & 0.218 & 0.204 & 0.199 & {\color{red}\bf 0.176} & {\color{red}\bf 0.169} & {\color{red}\bf 0.168} \\
             \Xhline{4\arrayrulewidth}
        \end{tabular}
    }
\end{table}

\section{Conclusion} \label{section_conclusion}
In this paper, we have revisited the CME problem that plays an important role in many computer vision applications. We derived the original measurement model associated with the rotation matrix and normalized translation, based on which the ML problem was formulated. To optimally solve the ML problem in the asymptotic case, we first estimated the noise variance by calculating the maximum eigenvalue of a $9 \times 9$ matrix. Based on the $\sqrt{m}$-consistent noise variance estimator, we then proposed a two-step estimator that has the same asymptotic statistical property as the ML estimator, i.e., consistency and asymptotic efficiency. We showed that our algorithm has $O(m)$ time complexity. 
Experiments on both synthetic data and real images demonstrated that when the point number reaches the order of hundreds, the proposed algorithm outperforms the SOTA ones in terms of MSE and CPU time. It is noteworthy to see that the value of $\|\bf t\|$ hardly affects the estimation of the rotation matrix, but has a significant effect on the translation. In addition, there exists a pure rotation statistic that can efficiently evaluate the estimation quality of the translation. Hence, the assumption that $\|{\bf t}\| \neq 0$ can be relaxed in real applications.

\section*{Appendix A: Proof of Theorem~\ref{theo_noise_est}} \label{proof_consistent_noise}
\setcounter{equation}{23}
The proof is mainly based on the following lemma: 
\begin{lemma}[{\cite[Lemma 6]{zeng2023consistent}}] \label{lemma_largest_eig}
	Let ${\bf R}$ and ${\bf S}$ be two real symmetric matrices and ${\bf Q}={\bf R}+{\bf S}$. If ${\bf Q}$ is positive-definite and ${\bf R}$ is positive-semidefinite with $0$ eigenvalues, then $\lambda_{\rm max}( {\bf Q}^{-1} {\bf S})=1$.
\end{lemma}
First, we show the positive definiteness of ${\bf Q}_m$ when $m \geq 9$. Note that 
\begin{equation*}
	{\rm rank}({\bf A}_m)={\rm rank}\left( \begin{bmatrix}
		{\bf z}_1^{h \top} {\bf L}_1 \\
		\vdots \\
		{\bf z}_m^{h \top} {\bf L}_m
	\end{bmatrix} \right)={\rm rank}\left( \begin{bmatrix}
		{\bf z}_1^{h \top} \otimes {\bf y}_1^{h \top}\\
		\vdots \\
		{\bf z}_m^{h \top} \otimes {\bf y}_m^{h \top}
	\end{bmatrix} \right).
\end{equation*}
Given Assumption~\ref{not_coplanar_assump}, ${\bf y}_i,i=1,\ldots,m$ are not collinear. Hence, the column rank of the matrix $[{\bf y}_1^{h}~\cdots~{\bf y}_m^{h}]^\top$ is $3$. By further combining the fact that ${\bf z}_i, i \in \{1,\ldots,m\}$ are independent random variables, it holds that the matrix ${\bf A}_m$ has full column rank with probability one. 
Since ${\bf Q}_m$ has the same rank as ${\bf A}_m$, we have that ${\bf Q}_m$ is positive-definite with probability one. 
Next, we construct a positive-semidefinite matrix with $0$ eigenvalues. Let ${\bf A}_m^o$ be the noise-free counterpart of ${\bf A}_m$, i.e., ${\bf A}_m^o=[{\bf L}_1^\top {\bf z}_1^{ho}~\cdots~{\bf L}_m^\top {\bf z}_m^{ho}]^\top$, where ${\bf z}_i^{ho}$ is the noise-free counterpart of ${\bf z}_i^{h}$. Since ${\bf z}_i^{ho \top} {\bf E} {\bf y}_i^{h}=0$ (epipolar geometry), we have ${\bf A}_m^o {\bm \theta}={\bf 0}$. Given Assumption~\ref{nonzero_t}, ${\bm \theta}={\rm vec}({\bf E}) \neq {\bf 0}$, which implies that the matrix ${\bf A}_m^o$ is not full column rank, and ${\bm \theta}$ is an eigenvector of ${\bf A}_m^o$ associated with the $0$ eigenvalue. 
Let ${\bf Q}_m^o={\bf A}_m^{o \top} {\bf A}_m^o/m$. Then the matrix ${\bf Q}_m^o$ is positive-semidefinite with at least one $0$ eigenvalue. The following lemma plays an important role in identifying the relationship among ${\bf Q}_m$, ${\bf Q}_m^o$, and ${\bf S}_m$:
\begin{lemma}[{\cite[Lemma 4]{zeng2022global}}] \label{lemma_noise_aver}
	Let $\{X_i\}$ be a sequence of independent random variables with $\mathbb E[X_i]=0$ and $\mathbb E\left[X_i^2 \right]  \leq \varphi <\infty$ for all $i$. Then, there holds $\sum_{i=1}^{m}X_i/m=O_p(1/\sqrt{m})$.
\end{lemma}
Based on Lemma~\ref{lemma_noise_aver}, it can be verified that 
\begin{equation} \label{psd_matrix_convergence}
	{\bf Q}_m={\bf Q}_m^o+\sigma^2 {\bf S}_m+O_p(1/\sqrt{m}).
\end{equation}
Further, note that ${\bf Q}_{m}$ is positive-definite, ${\bf Q}_{m}^o$ is positive-semidefinite with $0$ eigenvalues, and $O_p(1/\sqrt{m})$ is a quantity that converges to $0$ at a rate of $1/\sqrt{m}$. According to Lemma~\ref{lemma_largest_eig}, it holds that $\lambda_{\rm max}( {\bf Q}_{m}^{-1} \sigma^2 {\bf S}_{m})$ converges to $1$ at a rate of $1/\sqrt{m}$. In other words,  $\hat \sigma_m^2=1/\lambda_{\rm max} ({\bf Q}_m^{-1} {\bf S}_m)$ converges to $\sigma^2$ at a rate of $1/\sqrt{m}$, which completes the proof. 

\section*{Appendix B: Proof of Lemma~\ref{optimal_multiplier}}
We are going to show that $\lambda_i^*=0,i=1,\ldots,m$ in the asymptotic case. Let ${\bf r}_i^o$ denote the residual in the noise-free case. We can decompose $f_m$ as
\begin{align*}
	f_m & =\frac{1}{m} \sum_{i=1}^{m} \left\| {\bf r}_i\right\|^2 \\
	& = \frac{1}{m} \sum_{i=1}^{m} \left\| {\bf r}_i^o\right\|^2  + \underbrace{\frac{2}{m} \sum_{i=1}^{m} {\bm \epsilon}_i^\top {\bf r}_i^o}_{\rightarrow \ 0} + \underbrace{\frac{1}{m} \sum_{i=1}^{m} \|{\bm \epsilon}_i\|^2}_{\rightarrow \ 2\sigma^2} \\
	& \rightarrow \ \underbrace{\frac{1}{m} \sum_{i=1}^{m} \left\| {\bf r}_i^o\right\|^2}_{:=f_m^o} + 2 \sigma^2,
\end{align*}
where the third line is based on Lemma~\ref{lemma_noise_aver}. Note that $f_m^o=0$ when $\bf R$, $\bar {\bf t}$, and $k_i$'s take true values. Moreover, according to Theorem 22.9~\cite{hartley2003multiple}, given Assumption~\ref{not_coplanar_assump}, there does not exist a conjugate configuration of true values such that $f_m^o=0$. In other words, $f_m^o=0$ only if $k_i=\|{\bf t}^o\|/x_{i3}^o>0$. Then, from the KKT complementary slackness condition $\lambda_i^* k_i=0$, we obtain $\lambda_i^*=0,i=1,\ldots,m$. 

\section*{Appendix C: Derivation of GN iterations on ${\rm SO}(3)$ and 2-sphere} \label{derivation_GN_iteration}
The measurement equation~\eqref{noisy_measurement_model} can be rephrased as
\begin{equation} \label{measurement_model2}
	{\bf z}_i =\frac{{\bf W}({\bf L}_i {\rm vec}({\bf R})+k_i \bar {\bf t})}{{\bf e}_3^\top ({\bf L}_i {\rm vec}({\bf R})+k_i \bar {\bf t})}+{\bm \epsilon}_i.
\end{equation}
Define 
\begin{equation*}
	\begin{split}
		{\bf g}_{i}({\bf s},\alpha,\beta)&={\bf W}({\bf L}_i {\rm vec}(\hat {\bf R}_m^{\rm BE}\exp({\bf s}^{\wedge}))+k_i({\bf s},\alpha,\beta) \bar {\bf t}(\alpha,\beta)), \\
		h_{i}({\bf s},\alpha,\beta)&={\bf e}_3^\top({\bf L}_i {\rm vec}(\hat {\bf R}_m^{\rm BE}\exp({\bf s}^{\wedge}))+k_i({\bf s},\alpha,\beta) \bar {\bf t}(\alpha,\beta)), \\
		{\bf u}_{i}({\bf s},\alpha,\beta)&=\frac{{\bf W}({\bf L}_i {\rm vec}(\hat {\bf R}_m^{\rm BE}\exp({\bf s}^{\wedge}))+k_i({\bf s},\alpha,\beta) \bar {\bf t}(\alpha,\beta))}{{\bf e}_3^\top ({\bf L}_i {\rm vec}(\hat {\bf R}_m^{\rm BE}\exp({\bf s}^{\wedge}))+k_i({\bf s},\alpha,\beta) \bar {\bf t}(\alpha,\beta))}, 
	\end{split}
\end{equation*}
where $k_i({\bf s},\alpha,\beta)$ is defined by substituting $\bf R$ with $\hat {\bf R}_m^{\rm BE}\exp({\bf s}^{\wedge})$ and $\bar {\bf t}$ with $\bar {\bf t}(\alpha,\beta)$ in~\eqref{GN_R_and_t}. 
Then we have 
\begin{align*}
	\frac{\partial {\bf u}_{i}}{\partial {\bf s}^\top} =& \frac{(h_i({\bf 0}){\bf W}-{\bf g}_i({\bf 0}){\bf e}_3^\top)({\bm \Psi}_i+\bar {\bf t}({\bf 0}) \partial k_i/\partial {\bf s}^\top)}{h_i({\bf 0})^2}, \\
	\frac{\partial {\bf u}_{i}}{\partial \alpha} =& \frac{(h_i({\bf 0}){\bf W}-{\bf g}_i({\bf 0}){\bf e}_3^\top)(k_i {\bm \Phi}+\bar {\bf t}({\bf 0}) \partial k_i/\partial \alpha)}{h_i({\bf 0})^2}, \\
	\frac{\partial {\bf u}_{i}}{\partial \beta} =& \frac{(h_i({\bf 0}){\bf W}-{\bf g}_i({\bf 0}){\bf e}_3^\top)(k_i {\bm \Theta}+\bar {\bf t}({\bf 0}) \partial k_i/\partial \beta)}{h_i({\bf 0})^2},
\end{align*}
where all partial derivatives are evaluated at ${\bf s},\alpha,\beta=0$, and
\begin{align*}
	{\bm \Psi}_i & = {\bf y}_i^{h \top} \otimes \hat{\bf R}_m^{\rm BE} \frac{\partial {\rm vec}(\exp(\bf s^{\wedge}))}{
		\partial {\bf s}^\top}, \\
	{\bm \Phi} & = \frac{\partial \bar {\bf t}(\alpha,\beta)}{\partial \alpha} =\begin{bmatrix}
		-\sin \alpha_0 \cos \beta_0 \\
		-\sin \alpha_0 \sin \beta_0 \\
		\cos \alpha_0
	\end{bmatrix},\\
    {\bm \Theta} & = \frac{\partial \bar {\bf t}(\alpha,\beta)}{\partial \beta} =\begin{bmatrix}
		-\cos \alpha_0 \sin \beta_0 \\
		\cos \alpha_0 \cos \beta_0 \\
		0
	\end{bmatrix},
\end{align*}
and $\partial k_i/\partial {\bf s}^\top$, $\partial k_i/\partial \alpha$, $\partial k_i/\partial \beta$ are given in~\eqref{long_eqn1},
\begin{figure*}[b]
    \begin{equation} \label{long_eqn1}
		\begin{split}
			\frac{\partial k_i}{\partial {\bf s}^\top} & = \frac{{\rm den}\left( {\bf y}_i^{h \top}  {{}\hat{\bf R}_m^{\rm BE}}^\top ({\bf C}_1 {\bf I}_3 \otimes \bar {\bf t} + {\bf I}_3 \otimes \bar {\bf t}^\top {\bf C}_1^\top) {\bm \Psi}_i\right) - {\rm num} \left(\bar {\bf t}^\top {\bf C}_2 {\bf I}_3 \otimes \bar {\bf t} {\bm \Psi}_i \right)}{{\rm den}^2}, \\
			\frac{\partial k_i}{\partial \alpha} & = \frac{{\rm den}\left( {\bf y}_i^{h \top}  {{}\hat{\bf R}_m^{\rm BE}}^\top {\bf C}_1 (\hat{\bf R}_m^{\rm BE}{\bf y}_i^{h}) \otimes {\bf I}_3 {\bm \Phi}\right) - {\rm num} \left( \bar {\bf t}^\top {\bf C}_2(\hat{\bf R}_m^{\rm BE}{\bf y}_i^{h}) \otimes {\bf I}_3+(\hat{\bf R}_m^{\rm BE}{\bf y}_i^{h})^\top \otimes \bar {\bf t}^\top {\bf C}_2^\top\right) {\bm \Phi}}{{\rm den}^2}, \\
			\frac{\partial k_i}{\partial \beta} & = \frac{{\rm den}\left( {\bf y}_i^{h \top}  {{}\hat{\bf R}_m^{\rm BE}}^\top {\bf C}_1 (\hat{\bf R}_m^{\rm BE}{\bf y}_i^{h}) \otimes {\bf I}_3 {\bm \Theta}\right) - {\rm num} \left( \bar {\bf t}^\top {\bf C}_2(\hat{\bf R}_m^{\rm BE}{\bf y}_i^{h}) \otimes {\bf I}_3+(\hat{\bf R}_m^{\rm BE}{\bf y}_i^{h})^\top \otimes \bar {\bf t}^\top {\bf C}_2^\top\right) {\bm \Theta}}{{\rm den}^2},
		\end{split}
	\end{equation}
\end{figure*}
where $\rm den$ and $\rm num$ represent the denominator and numerator of $k_i$ in~\eqref{expression_ki}, respectively.

Then we can obtain the Jacobian matrix
\begin{equation*}
	{\bf J}=\begin{bmatrix}
		\vdots ~~~~~~\vdots~~~~~\vdots\\
		\frac{\partial {\bf u}_{i}}{\partial {\bf s}^\top} ~~\frac{\partial {\bf u}_{i}}{\partial \alpha} ~~\frac{\partial {\bf u}_{i}}{\partial \beta} \\
		\vdots ~~~~~~\vdots~~~~~\vdots
	\end{bmatrix} \in \mathbb R^{2m \times 5} .
\end{equation*}
The GN iteration is 
\begin{equation}  \label{GN_iteration}
	\left[ \hat {\bf s}_m^{\rm GN} ~
		\hat \alpha_m^{\rm GN} ~
		\hat \beta_m^{\rm GN}\right]^\top=
	[{\bf 0} ~
		\alpha_0 ~
		\beta_0]^\top+ 
	({\bf J}^\top {\bf J}) ^{-1}
	{\bf J}^\top {\bf r} ,
\end{equation}
where ${\bf r}=[{\bf r}_1^\top~\cdots~{\bf r}_m^\top]^\top$.

\section*{Appendix D: Proof of Theorem~\ref{asymptotic_efficiency_theorem}} \label{asymptotic_efficiency_proof}
Let $f_m({\bf s},\alpha,\beta)$ denote the objective function of~\eqref{LS_problem}, where ${\bf R}=\hat{{\bf R}}_m^{\rm BE}\exp({\bf s}^{\wedge})$ and $\bar {\bf t}=\bar {\bf t}(\alpha,\beta)$, and denote the optimal $\bf s$, $\alpha$, and $\beta$ as $ \hat {\bf s}_m^{\rm ML}$, $\hat \alpha_m^{\rm ML}$, and $\hat \beta_m^{\rm ML}$. Since $\hat {\bf R}_m^{\rm BE}$ and $\hat{\bar {{\bf t}}}_m^{\rm BE}$ are $\sqrt{m}$-consistent, it holds that 
\begin{equation*}
   \left[ \hat {\bf s}_m^{\rm ML} ~
		\hat \alpha_m^{\rm ML} ~
		\hat \beta_m^{\rm ML}\right]^\top-
	[{\bf 0} ~
		\alpha_0 ~
		\beta_0]^\top=O_p(1/\sqrt{m}).
\end{equation*}
Based on the optimality condition $\nabla f_m(\hat {\bf s}_m^{\rm ML},\hat \alpha_m^{\rm ML},\hat \beta_m^{\rm ML})={\bf 0}$ and the Taylor expansion, we have
\begin{equation*}
    {\bf 0}=\nabla f_{m,{\bf 0}}+\nabla^2 f_{m,{\bf 0}} \begin{bmatrix}
        \hat {\bf s}_m^{\rm ML}-{\bf 0} \\
        \hat \alpha_m^{\rm ML}-\alpha_0 \\
        \hat \beta_m^{\rm ML}-\beta_0
    \end{bmatrix} + o_p(\frac{1}{\sqrt{m}}),
\end{equation*}
where we use $f_{m,{\bf 0}}$ for the abbreviation of $f_m({\bf 0},\alpha_0,\beta_0)$.
Then,
\begin{equation*}
    \begin{bmatrix}
        \hat {\bf s}_m^{\rm ML}-{\bf 0} \\
        \hat \alpha_m^{\rm ML}-\alpha_0 \\
        \hat \beta_m^{\rm ML}-\beta_0
    \end{bmatrix}=-\nabla^2 f_{m,{\bf 0}}^{-1}\nabla f_{m,{\bf 0}}  + o_p(\frac{1}{\sqrt{m}}).
\end{equation*}
By combining the GN iteration~\eqref{GN_iteration}, we finally obtain
\begin{align*}
    & \begin{bmatrix}
        \hat {\bf s}_m^{\rm ML}-\hat {\bf s}_m^{\rm GN} \\
        \hat \alpha_m^{\rm ML}-\hat \alpha_m^{\rm GN} \\
        \hat \beta_m^{\rm ML}-\hat \beta_m^{\rm GN}
    \end{bmatrix} \\
    & =-\nabla^2 f_{m,{\bf 0}}^{-1}\nabla f_{m,{\bf 0}}-({\bf J}^\top {\bf J}) ^{-1}
	{\bf J}^\top {\bf r}+o_p(\frac{1}{\sqrt{m}}) \\
 & = \left( \frac{2 \nabla^2 f_{m,{\bf 0}}^{-1}}{m}-\frac{({\bf J}^\top {\bf J}) ^{-1}}{m}\right) \frac{{\bf J}^\top {\bf r}}{m} +o_p(\frac{1}{\sqrt{m}}) \\
 & = O_p(\frac{1}{\sqrt{m}}) O_p(\frac{1}{\sqrt{m}})+o_p(\frac{1}{\sqrt{m}}) \\
 & =o_p(\frac{1}{\sqrt{m}}),
\end{align*}
where the third ``$=$'' is based on Lemma~\ref{lemma_noise_aver} in Appendix A. Since the $\exp, \sin, \cos$ are all continuous functions, we have $ \hat {\bf R}_m^{\rm ML}-\hat {\bf R}_m^{\rm GN}=o_p(1/\sqrt{m})$ and $\hat{\bar {{\bf t}}}_m^{\rm ML}-\hat{\bar {{\bf t}}}_m^{\rm GN}=o_p(1/\sqrt{m})$, which completes the proof.

\section*{Appendix E: The Cramer-Rao Bound} \label{derivation_CRB}
Define 
\begin{equation*}
	\begin{split}
		{\bf g}_{i}'({\bf R},\bar {\bf t})&={\bf W}({\bf R}{\bf y}_i^h+k_i \bar {\bf t}), \\
		h_{i}'({\bf R},\bar {\bf t})&={\bf e}_3^\top ({\bf R}{\bf y}_i^h+k_i \bar {\bf t}), \\
		{\bf u}_{i}'({\bf R},\bar {\bf t})&=\frac{{\bf W}({\bf R}{\bf y}_i^h+k_i \bar {\bf t})}{{\bf e}_3^\top ({\bf R}{\bf y}_i^h+k_i \bar {\bf t})}.
	\end{split}
\end{equation*}
Let ${\bm \Sigma}=\sigma^2 {\bf I}_2$, then given Assumption~\ref{noise_assump}, the likelihood function is 
\begin{equation*}
	\mathcal L({\bf R},\bar {\bf t};{\bf z}) = \prod_{i=1}^{m} \frac{1}{2 \pi \sigma^2} ~{\rm exp} \left( -\frac{1}{2} \left\|( {\bf z}_i-{\bf u}_{i}'({\bf R},\bar {\bf t}))\right\|_{{\bm \Sigma}}^2\right),
\end{equation*}
which further yields the log-likelihood function
\begin{equation} \label{log_likelihood}
	\ell({\bf R},\bar {\bf t};{\bf z})= m~ {\rm ln} \frac{1}{2 \pi \sigma^2}-\sum_{i=1}^{m} \frac{1}{2} \left\|( {\bf z}_i-{\bf u}_{i}'({\bf R},\bar {\bf t}))\right\|_{{\bm \Sigma}}^2.
\end{equation}
Let ${\bm \xi}_{\bf R}={\rm vec}({\bf R})$ and ${\bm \xi}=\left[{\bm \xi}_{\bf R}^\top ~~ \bar {\bf t}^\top \right]^\top$. The derivative of $ \ell({\bf R},\bar {\bf t};{\bf z})$ is 
\begin{equation*}
	\frac{\partial  \ell}{\partial {\bm \xi}^\top} = \sum_{i=1}^{m} \left( {\bf z}_i-{\bf u}_{i}'({\bf R},\bar {\bf t})\right)^\top {\bm \Sigma}^{-1} \frac{\partial  {\bf u}_{i}'}{\partial {\bm \xi}^\top}.
\end{equation*}
To obtain $ \partial  {\bf u}_{i}'/\partial {\bm \xi}^\top$, we need to calculate $ \partial  {\bf u}_{i}'/\partial {\bm \xi}_{\bf R}^\top$ and $ \partial  {\bf u}_{i}'/\partial \bar {\bf t}^\top$, respectively. The result is, 
	\begin{align*}
		\frac{\partial  {\bf u}_{i}'}{\partial {\bm \xi}_{\bf R}^\top} &= \frac{h_{i}' \left( {\bf W}{\bf L}_i + {\bf W} \bar {\bf t} \frac{\partial k_i}{\partial {\bm \xi}_{\bf R}^\top} \right)-{\bf g}_{i}'\left( {\bf e}_3^\top {\bf L}_i + {\bf e}_3^\top \bar {\bf t} \partial \frac{\partial k_i}{\partial {\bm \xi}_{\bf R}^\top} \right)}{h_{i}'^2},\\
		\frac{\partial  {\bf u}_{i}'}{\partial \bar {\bf t}^\top} &= \frac{h_{i}' \left( {\bf W}k_i + {\bf W} \bar {\bf t} \frac{\partial k_i}{\partial \bar {\bf t}^\top} \right)-{\bf g}_{i}'\left( {\bf e}_3^\top k_i + {\bf e}_3^\top \bar {\bf t} \partial \frac{\partial k_i}{\partial \bar {\bf t}^\top} \right)}{h_{i}'^2},
	\end{align*}
where $ \partial  k_i/\partial {\bm \xi}_{\bf R}^\top$ and $ \partial  k_i/\partial \bar {\bf t}^\top$ are given in~\eqref{long_eqn2},
	\begin{figure*}[b]
	    \begin{equation} \label{long_eqn2}
		\begin{split}
			\frac{\partial k_i}{\partial {\bm \xi}_{\bf R}^\top} &= \frac{{\rm den}\left( {\bf y}_i^{h \top} {{}\hat{\bf R}_m^{\rm BE}}^\top ({\bf C}_1 {\bf I}_3 \otimes \bar {\bf t} + {\bf I}_3 \otimes \bar {\bf t}^\top {\bf C}_1^\top) {\bf L}_i\right) - {\rm num} \left(\bar {\bf t}^\top {\bf C}_2 {\bf I}_3 \otimes \bar {\bf t} {\bf L}_i \right)}{{\rm den}^2}, \\
			\frac{\partial k_i}{\partial \bar {\bf t}^\top} &= \frac{{\rm den}\left( {\bf y}_i^{h \top}  {{}\hat{\bf R}_m^{\rm BE}}^\top {\bf C}_1 (\hat{\bf R}_m^{\rm BE}{\bf y}_i^{h}) \otimes {\bf I}_3 \right) - {\rm num} \left( \bar {\bf t}^\top {\bf C}_2(\hat{\bf R}_m^{\rm BE}{\bf y}_i^{h}) \otimes {\bf I}_3+(\hat{\bf R}_m^{\rm BE}{\bf y}_i^{h})^\top \otimes \bar {\bf t}^\top {\bf C}_2^\top\right)}{{\rm den}^2},
		\end{split}
	\end{equation}
	\end{figure*}
where $\rm den$ and $\rm num$ represent the denominator and numerator of $k_i$ in~\eqref{expression_ki}, respectively.

Then we have $\partial  {\bf u}_{i}'/\partial {\bm \xi}^\top=\left[ \partial  {\bf u}_{i}'/\partial {\bm \xi}_{\bf R}^\top~~\partial  {\bf u}_{i}'/\partial \bar {\bf t}^\top \right]$. Note that ${\bf z}_i-{\bf u}_{i}'({\bf R}^o,\bar {\bf t}^o)={\bm \epsilon}_i$. Hence, the Fisher information matrix can be calculated as 
\begin{align*}
	{\bf F} & = \mathbb E\left[ \frac{\partial  \ell}{\partial {\bm \xi}} \frac{\partial  \ell}{\partial {\bm \xi}^\top}\right] \\
	& = \mathbb E \left[ \left( \sum_{i=1}^{m} \frac{\partial  {\bf u}_{i}'}{\partial {\bm \xi}}{\bm \Sigma}^{-1}{\bm \epsilon}_i \right) \left( \sum_{i=1}^{m} {\bm \epsilon}_i^\top {\bm \Sigma}^{-1} \frac{\partial  {\bf u}_{i}'}{\partial {\bm \xi}^\top} \right) \right] \\
	& = \sum_{i=1}^{m} \frac{\partial  {\bf u}_{i}'}{\partial {\bm \xi}} {\bm \Sigma}^{-1} \frac{\partial  {\bf u}_{i}'}{\partial {\bm \xi}^\top},
\end{align*}
where the third line is derived based on the independence among measurement noises ${\bm \epsilon}_i$'s.
Note that what we have now derived is unconstrained Fisher information. Since there are constraints on the rotation matrix $\bf R$ and the normalized translation $\bar {\bf t}$, we need to calculate a constrained counterpart ${\bf F}_c$. We can use the following $7$ equations to characterize the constraints
\begin{equation*}
	{\bf h}({\bm \xi}) = \begin{bmatrix}
		{\bm \xi}_{1:3}^{\top}{\bm \xi}_{1:3}-1 \\
		{\bm \xi}_{4:6}^{\top}{\bm \xi}_{1:3} \\
		{\bm \xi}_{7:9}^{\top}{\bm \xi}_{1:3} \\
		{\bm \xi}_{4:6}^{\top}{\bm \xi}_{4:6}-1 \\
		{\bm \xi}_{7:9}^{\top}{\bm \xi}_{4:6} \\
		{\bm \xi}_{7:9}^{\top}{\bm \xi}_{7:9}-1 \\
		{\bm \xi}_{10:12}^{\top}{\bm \xi}_{10:12}-1
	\end{bmatrix}={\bf 0},
\end{equation*}
where the first $6$ equations are associated with the rotation matrix~\cite{lynch2017modern} and the last one is for the normalized translation. Let 
\begin{equation*}
	{\bf H}({\bm \xi}) = \frac{\partial {\bf h}({\bm \xi})}{\partial {\bm \xi}^\top} \in \mathbb R^{7 \times 12}.
\end{equation*}
The gradient matrix ${\bf H}({\bm \xi})$ have full row rank since the constraints are nonredundant, and hence there exists a matrix ${\bf U}\in\mathbb{R}^{12\times 5}$ whose columns form an orthonormal basis for the nullspace of ${\bf H}({\bm \xi})$, that is, 
$${\bf H}({\bm \xi}){\bf U}={\bf 0}$$
where ${\bf U}^{\top}{\bf U}={\bf I}$.
Finally, the constrained Fisher information is given as~\cite{stoica1998cramer}
\begin{equation*}
	{\bf F}_c = {\bf U} ({\bf U}^\top {\bf F} {\bf U})^{-1}{\bf U}^\top,
\end{equation*}
and the theoretical lower bound is ${\rm CRB}={\rm tr}({\bf F}_c)$.

\section*{Appendix F: Performance comparison of choosing varied initial guesses} \label{diff_eigenvectors}

Here we show how performance changes according to the policy of choosing the initial guess with eigenvectors corresponding to larger eigenvalues of ${\bf Q}_m^{\rm BE}$. 
We adopt the same simulation setting as Figure 4(c), and the result is shown in Figure~\ref{MSE_diff_eigenvector}. In our paper, since $\bm \theta \in \mathbb R^9$, the curve ``v9'' denotes choosing the eigenvector corresponding to the smallest eigenvalue of ${\bf Q}_m^{\rm BE}$, ``v8'' represents choosing the eigenvector corresponding to the second smallest eigenvalue of ${\bf Q}_m^{\rm BE}$, and by analogy for curves ``v7'' and ``v6''. We see that by choosing the eigenvector corresponding to the smallest eigenvalue of ${\bf Q}_m^{\rm BE}$ as the initial value, our algorithm can asymptotically achieve the theoretical lower bound, CRB. In contrast, selecting eigenvectors corresponding to other eigenvalues results in significantly larger MSEs that do not decrease as the point number increases. 

 \begin{figure}[!htbp]
    \centering
   \begin{subfigure}{0.49\linewidth}
       \includegraphics[width=\linewidth]{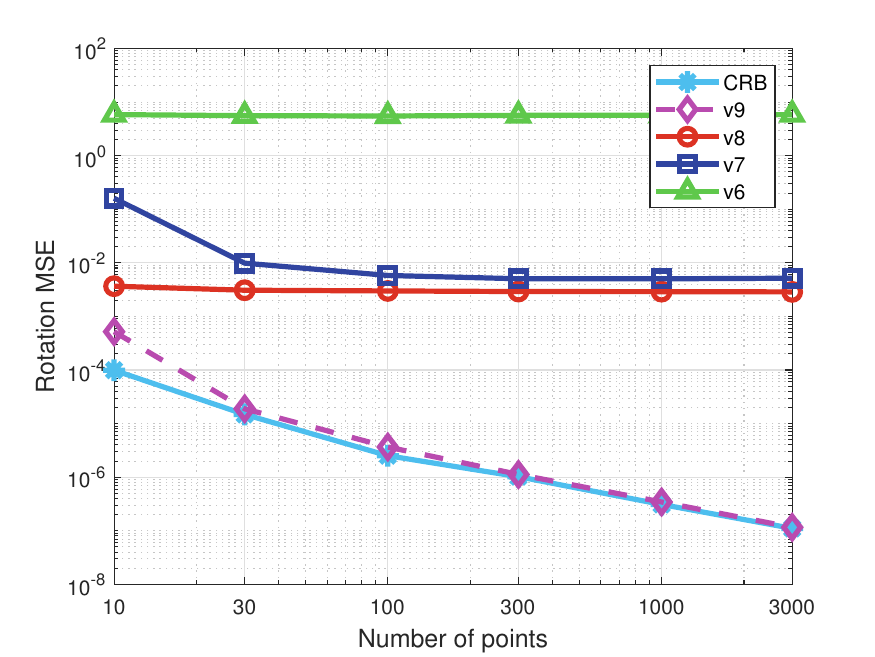}
   \end{subfigure}
    \begin{subfigure}{0.49\linewidth}
       \includegraphics[width=\linewidth]{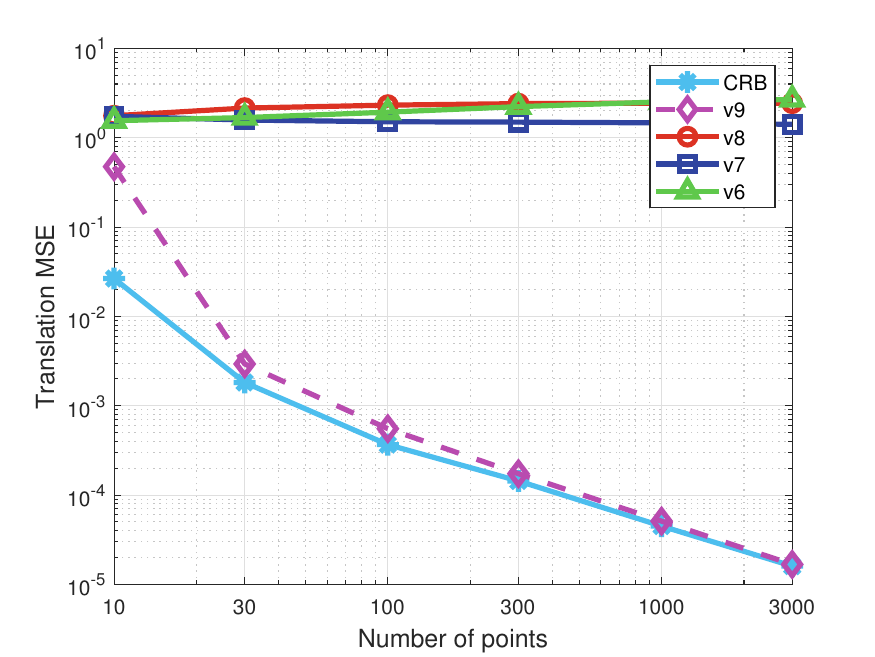}
   \end{subfigure}
   \caption{MSE of selecting varied eigenvectors as the initial guess.}
            \label{MSE_diff_eigenvector}
\end{figure}

\section*{Appendix G: Discussion of Gaussian assumption}  \label{Gaussian_discussion}
We observe that the i.i.d. Gaussian noise assumption is typically well satisfied. For instance, we select images 11 and 12 from the exhibition hall sequence in the ETH3D dataset, as well as images 7 and 8 from sequence 4 in the KITTI odometry dataset. Using the ground truth relative pose, we project feature points from one image to the other, represented as epipolar lines, and plot the empirical distribution of epipolar point-to-line distances (in pixels). As illustrated in Figure~\ref{empirical_distributions}, the epipolar point-to-line errors align closely with a Gaussian distribution.

If the feature noises exhibit varying variances, the proposed algorithm will still maintain consistency; however, it will lose statistical efficiency. We include a simulation that violates the i.i.d. assumption. Specifically, a uniformly random noise level within the range of $[0.5, 1.5]$ pixels ($[0.0625\%, 0.1875\%]$ of the image diagonal) is assigned to each point correspondence. As shown in Figure~\ref{MSE_non_iid}, under the non-i.i.d. noise condition, while our estimator \texttt{CECME} still demonstrates consistency, it fails to achieve the theoretical lower bound, CRB. This is because, in our constructed ML problem, equal weights are assigned to each point correspondence based on the i.i.d. noise assumption, which reduces the problem to a LS formulation rather than a true ML problem and results in a suboptimal estimation covariance.

 \begin{figure}[!htbp]
    \centering
   \begin{subfigure}{0.48\linewidth}
       \includegraphics[width=\linewidth]{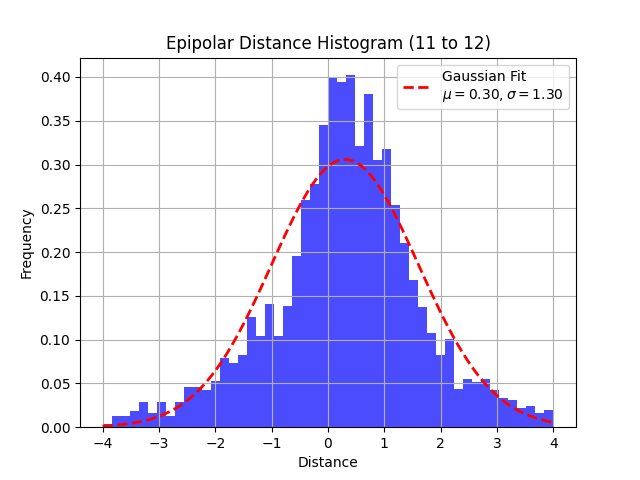}
       \caption{ETH3D exhib. hall 11-12}
   \end{subfigure}
    \begin{subfigure}{0.48\linewidth}
       \includegraphics[width=\linewidth]{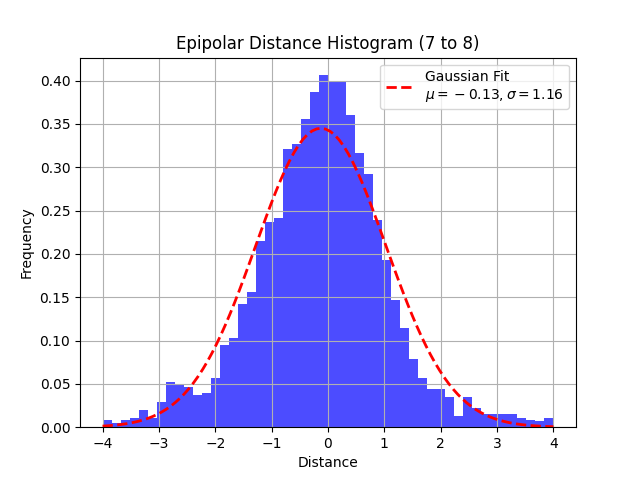}
       \caption{KITTI seq04 7-8}
   \end{subfigure}
   \caption{Empirical distribution of feature-matching noises.}
            \label{empirical_distributions}
\end{figure}

 \begin{figure}[!htbp]
    \centering
   \begin{subfigure}{0.49\linewidth}
       \includegraphics[width=\linewidth]{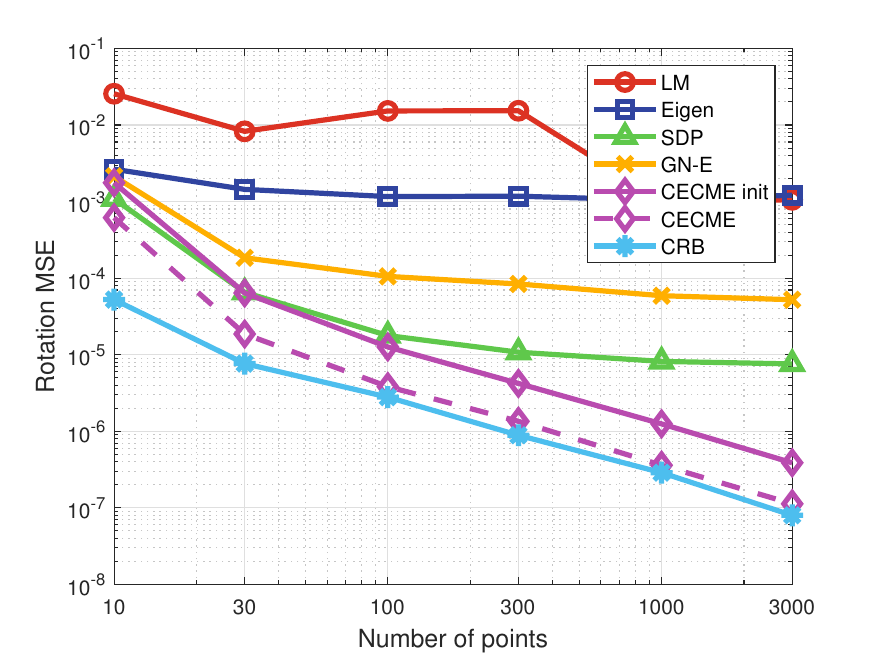}
   \end{subfigure}
    \begin{subfigure}{0.49\linewidth}
       \includegraphics[width=\linewidth]{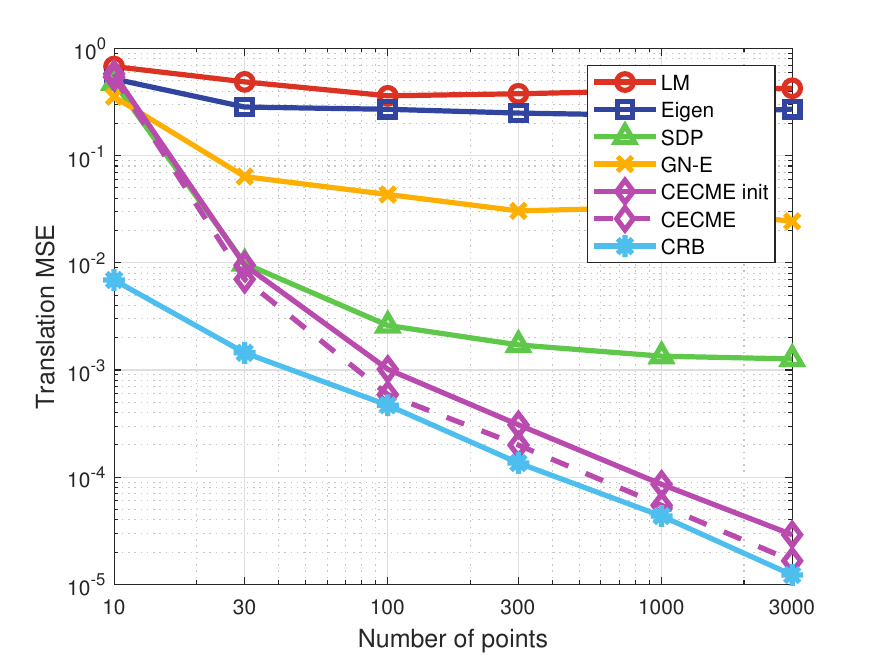}
   \end{subfigure}
   \caption{MSE comparison in the non-i.i.d. noise case.}
            \label{MSE_non_iid}
\end{figure}

\section*{Appendix H: Effect of GN iteration number in the ETH3D dataset}  \label{effect_GN_num_ETH3D}
We test the estimation errors in the ETH3D dataset with varied GN numbers. The results are listed in Table~\ref{ETH3D_varied_GN}, from which we see that a single GN iteration largely improves the initial estimate, while 5 GN iterations result in errors comparable to those achieved with just one iteration.

\begin{table}[!htb]
    \centering
    \captionsetup{justification=centering}
    \caption{Average estimation errors in the ETH3D dataset with varied GN numbers. The units for rotation and translation are $10^{-3}$ rad and $10^{-5}$, respectively. The percentages next to 1 GN and 5 GN represent the changes of 1 GN relative to 0 GN and the changes of 5 GN relative to 1 GN, respectively.} \label{ETH3D_varied_GN}
    \resizebox{0.98\columnwidth}{!}{%
        \begin{tabular}{ c c c c c c c } 
            \hline
            \multirow{2}{*}{Scenario} & \multicolumn{3}{c}{$\bf R$} & \multicolumn{3}{c}{$\bar {\bf t}$} \\
            \cline{2-7}
            & 0 GN & 1 GN & 5 GN & 0 GN & 1 GN & 5 GN \\ 
            \hline
            door & 0.523 & 0.374 (${\color{red} \downarrow 28\%}$) & 0.374 (0.0\%) & 1.56 & 0.232 (${\color{red} \downarrow 85\%}$) & 0.232 (0.0\%) \\ 
            relief & 2.20 & 0.787 (${\color{red} \downarrow 64\%}$) & 0.797 (${\color{blue} \uparrow 1.3\%}$) & 3.91 & 0.485 (${\color{red} \downarrow 88\%}$) & 0.492 (${\color{blue} \uparrow 1.4\%}$) \\ 
            observatory & 1.65 & 1.32 (${\color{red} \downarrow 20\%}$) & 1.31 (${\color{red} \downarrow 0.76\%}$) & 12.5 & 5.90 (${\color{red} \downarrow 53\%}$) & 5.18 (${\color{red} \downarrow 12\%}$) \\ 
            facade & 2.10 & 1.14 (${\color{red} \downarrow 46\%}$) & 1.12 (${\color{red} \downarrow 1.8\%}$) & 55.5 & 38.3 (${\color{red} \downarrow 31\%}$) & 38.3 (0.0\%) \\ 
            boulders & 3.40 & 0.992 (${\color{red} \downarrow 71\%}$) & 0.832 (${\color{red} \downarrow 16\%}$) & 17.4 & 14.2 (${\color{red} \downarrow 18\%}$) & 0.490 (${\color{red} \downarrow 97\%}$) \\ 
            courtyard & 27.2 & 22.7 (${\color{red} \downarrow 17\%}$) & 21.2 (${\color{red} \downarrow 6.6\%}$) & 464 & 411 (${\color{red} \downarrow 11\%}$) & 402 (${\color{red} \downarrow 2.2\%}$) \\ 
            relief 2 & 1.58 & 0.828 (${\color{red} \downarrow 48\%}$) & 0.844 (${\color{blue} \uparrow 1.9\%}$) & 2.89 & 0.842 (${\color{red} \downarrow 71\%}$) & 0.850 (${\color{blue} \uparrow 0.95\%}$) \\ 
            terrace 2 & 0.652 & 0.428 (${\color{red} \downarrow 34\%}$) & 0.428 (0.0\%) & 1.32 & 0.246 (${\color{red} \downarrow 81\%}$) & 0.246 (0.0\%) \\ 
            statue & 0.671 & 0.407 (${\color{red} \downarrow 39\%}$) & 0.407 (0.0\%) & 0.120 & 0.0087 (${\color{red} \downarrow 93\%}$) & 0.0087 (0.0\%) \\ 
            delivery area & 3.93 & 2.28 (${\color{red} \downarrow 42\%}$) & 2.05 (${\color{red} \downarrow 10\%}$) & 72.5 & 41.0 (${\color{red} \downarrow 43\%}$) & 34.4 (${\color{red} \downarrow 16\%}$) \\ 
            bridge & 1.56 & 0.875 (${\color{red} \downarrow 44\%}$) & 0.838 (${\color{red} \downarrow 4.2\%}$) & 2.79 & 0.877 (${\color{red} \downarrow 69\%}$) & 0.470 (${\color{red} \downarrow 46\%}$) \\ 
            exhibition hall & 6.05 & 3.94 (${\color{red} \downarrow 35\%}$) & 3.66 (${\color{red} \downarrow 7.1\%}$) & 562 & 513 (${\color{red} \downarrow 8.7\%}$) & 474 (${\color{red} \downarrow 7.6\%}$) \\ 
            electro & 2.21 & 1.04 (${\color{red} \downarrow 53\%}$) & 1.04 (0.0\%) & 11.5 & 2.63 (${\color{red} \downarrow 77\%}$) & 2.04 (${\color{red} \downarrow 22\%}$) \\ 
            terrace & 1.80 & 0.979 (${\color{red} \downarrow 46\%}$) & 0.982 (${\color{blue} \uparrow 0.31\%}$) & 2.47 & 0.273 (${\color{red} \downarrow 89\%}$) & 0.272 (${\color{red} \downarrow 0.37\%}$) \\ 
            kicker & 4.71 & 2.15 (${\color{red} \downarrow 54\%}$) & 2.18 (${\color{blue} \uparrow 1.4\%}$) & 28.4 & 8.76 (${\color{red} \downarrow 69\%}$) & 6.91 (${\color{red} \downarrow 21\%}$) \\ 
            botanical garden & 2.94 & 0.927 (${\color{red} \downarrow 68\%}$) & 0.899 (${\color{red} \downarrow 3.0\%}$) & 28.0 & 3.70 (${\color{red} \downarrow 87\%}$) & 3.29 (${\color{red} \downarrow 11\%}$) \\ 
            terrains & 2.80 & 1.40 (${\color{red} \downarrow 50\%}$) & 1.51 (${\color{blue} \uparrow 7.9\%}$) & 7.01 & 2.24 (${\color{red} \downarrow 68\%}$) & 6.79 (${\color{blue} \uparrow 203\%}$) \\ 
            playground & 1.44 & 0.680 (${\color{red} \downarrow 53\%}$) & 0.681 (${\color{blue} \uparrow 0.15\%}$) & 0.940 & 0.453 (${\color{red} \downarrow 52\%}$) & 0.450 (${\color{red} \downarrow 0.66\%}$) \\ 
            living room & 2.28 & 1.44 (${\color{red} \downarrow 37\%}$) & 1.46 (${\color{blue} \uparrow 1.4\%}$) & 5.78 & 3.19 (${\color{red} \downarrow 45\%}$) & 3.20 (${\color{blue} \uparrow 0.31\%}$) \\ 
            lecture room & 1.90 & 1.01 (${\color{red} \downarrow 47\%}$) & 0.979 (${\color{red} \downarrow 3.1\%}$) & 40.0 & 8.10 (${\color{red} \downarrow 80\%}$) & 1.31 (${\color{red} \downarrow 84\%}$) \\ 
            pipes & 2.83 & 1.58 (${\color{red} \downarrow 44\%}$) & 1.77 (${\color{blue} \uparrow 12\%}$) & 10.4 & 1.03 (${\color{red} \downarrow 90\%}$) & 3.18 (${\color{blue} \uparrow 209\%}$) \\ 
            lounge & 2.28 & 1.58 (${\color{red} \downarrow 31\%}$) & 1.79 (${\color{blue} \uparrow 13\%}$) & 4.23 & 2.63 (${\color{red} \downarrow 38\%}$) & 4.30 (${\color{blue} \uparrow 63\%}$) \\ 
            office & 2.55 & 1.37 (${\color{red} \downarrow 46\%}$) & 1.36 (${\color{red} \downarrow 0.73\%}$) & 8.38 & 1.44 (${\color{red} \downarrow 83\%}$) & 1.49 (${\color{blue} \uparrow 3.5\%}$) \\ 
            meadow & 10.0 & 4.51 (${\color{red} \downarrow 55\%}$) & 4.38 (${\color{red} \downarrow 2.9\%}$) & 45.5 & 4.96 (${\color{red} \downarrow 89\%}$) & 4.66 (${\color{red} \downarrow 6.0\%}$) \\ 
            old computer & 2.29 & 1.33 (${\color{red} \downarrow 42\%}$) & 1.28 (${\color{red} \downarrow 3.8\%}$) & 20.4 & 7.06 (${\color{red} \downarrow 65\%}$) & 4.50 (${\color{red} \downarrow 36\%}$) \\ 
            \hline
        \end{tabular}
    }
\end{table}

\bibliographystyle{IEEEtran}
\bibliography{sj_reference}

\end{document}